\documentclass[11pt,letterpaper]{article}

\usepackage{amsmath,amssymb,amsthm}
\usepackage{deepthink}

\usepackage{mathtools} % Loads amsmath automatically
\usepackage{algorithm}
\usepackage{algpseudocode}
\usepackage{enumitem}
\usepackage{tcolorbox}
\usepackage{subcaption}
\usepackage{comment}
\usepackage{oubraces}
\usepackage{wrapfig}
\usepackage{cancel}

\usepackage[nameinlink]{cleveref}
\usepackage[
  backend=biber,
  style=alphabetic,
  maxbibnames=100,
  minbibnames=100,
  maxcitenames=2,
  mincitenames=2
]{biblatex}
\newcommand{\tr}{\mathrm{Tr}}

\def\mbf{\mathbf}
\def\mbb{\mathbb}
\def\mc{\mathcal}

\newtheorem{theorem}{Theorem}
\newtheorem{lemma}[theorem]{Lemma} 
\newtheorem{proposition}[theorem]{Proposition} 

\newtheorem{corollary}[theorem]{Corollary}

\title{\LARGE{Out-of-Distribution Generalization of In-Context Learning: A Low-Dimensional Subspace Perspective}}

\newcommand{\jointfirst}{\textsuperscript{\dag}}

\authorblock{
  \href{https://soominkwon.github.io/}{\textbf{Soo Min Kwon}}\jointfirst,
  \href{https://alecxu00.github.io/}{\textbf{Alec S. Xu}}\jointfirst,
  \href{https://canyaras.com/}{\textbf{Can Yaras}},
  \href{https://web.eecs.umich.edu/~girasole/}{\textbf{Laura Balzano}},
  \href{https://qingqu.engin.umich.edu/}{\textbf{Qing Qu}}
}

\affiliation{
  Department of Electrical and Computer Engineering, University of Michigan
}

\authornote{
  \quad
  \jointfirst\ Equal contribution
}

\abstracttext{
\noindent 
The transformer's remarkable ability to perform in-context learning (ICL) has sparked a wide range of studies designed to understand its strengths and limitations. However, a theoretical understanding of when ICL can and cannot generalize beyond its pre-training data still remains unclear. This paper puts forth a minimal mathematical model that provably identifies when ICL can generalize out-of-distribution (OOD). By studying linear regression tasks parameterized with low-rank covariance matrices, we model distribution shifts as varying angles between subspaces and derive conditions under which a single-layer linear attention model interpolates across all angles. 
We show that if pre-training task vectors are drawn from a union of subspaces, transformers can generalize to all angle shifts—enabling ICL even in regions with zero probability mass in the training distribution. On the other hand, if the pre-training tasks are drawn from a single Gaussian, the test risk shows a non-negligible dependence on the angle, implying that ICL cannot generalize OOD.
  We empirically show that our results also hold for models such as GPT-2, and present experiments on how our results extend to nonlinear function classes.
}


\date{\today}
\correspondence{\href{kwonsm@umich.edu}{kwonsm@umich.edu}}
\resources{\href{https://github.com/soominkwon/ood-icl-generalize}{https://github.com/soominkwon/ood-icl-generalize}}

\headerlogo{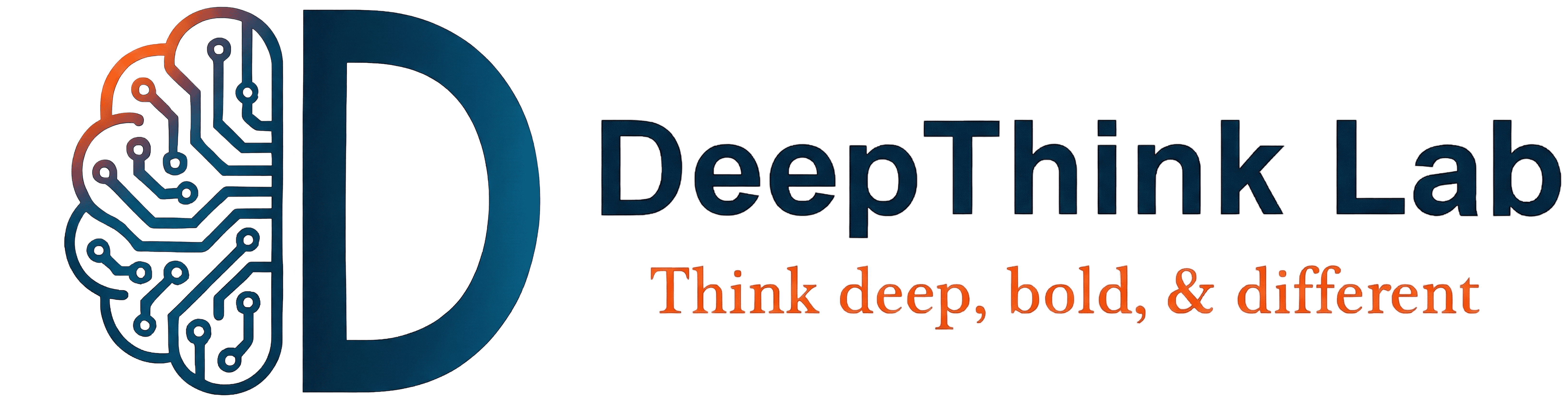}{https://deepthink-umich.github.io}

\begin{document}

\makeDeepthinkHeader

% ------------------------------------------------------------
% Optional: teaser figure
% ------------------------------------------------------------
\vspace{-0.2 in}
\begin{figure}[h!]
  \centering
  \IfFileExists{Deepthink_landscape_do_not_delete_compressed.png}{%
    \includegraphics[width=0.9\linewidth]{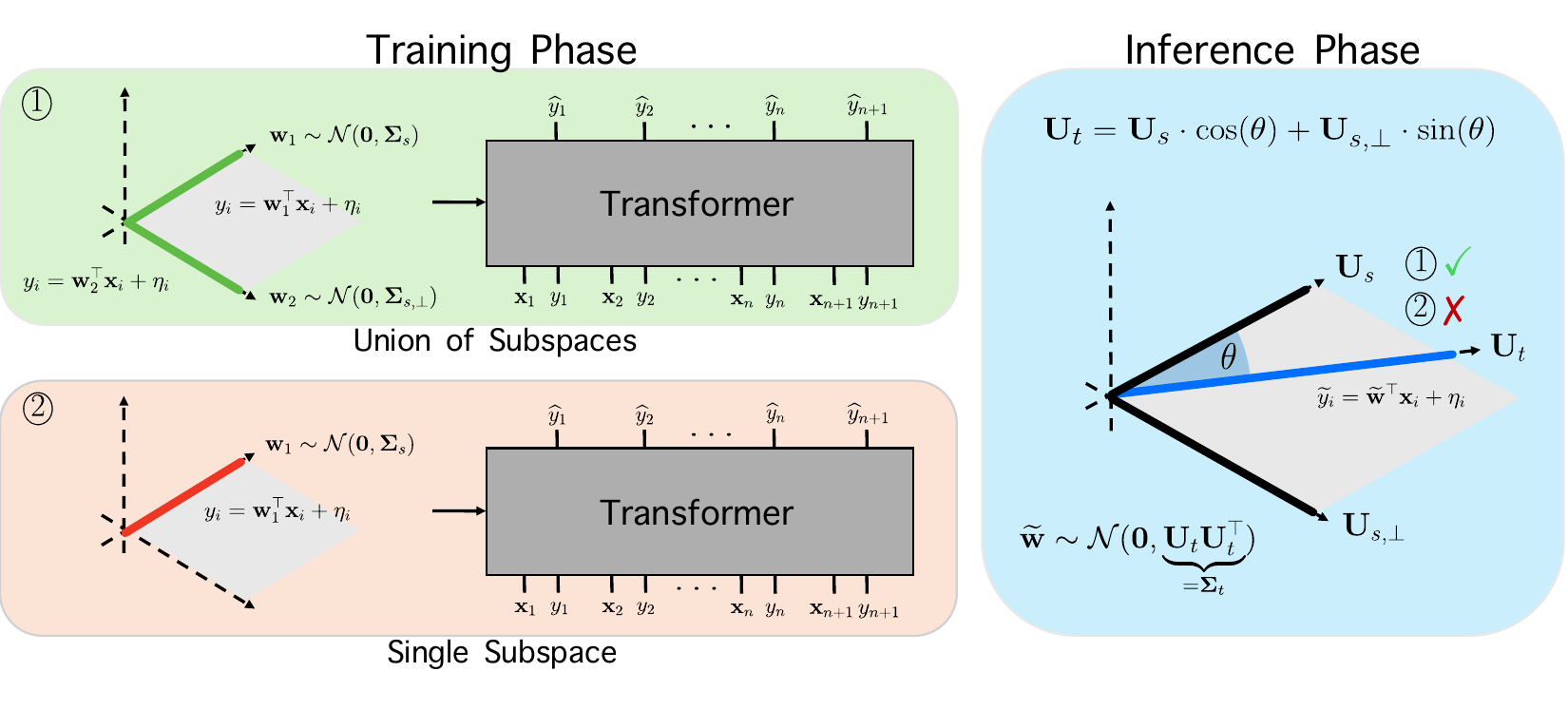}
  }{%
    \fbox{\parbox[c][3.2cm][c]{0.9\linewidth}{\centering Teaser image placeholder}}
  }
  \caption{\textbf{An illustrative overview of the setup of this paper.} We consider two models trained on different task vectors and analyze the conditions under which a model can generalize OOD. When task vectors are drawn from a union of subspaces, a transformer via ICL can generalize to subspaces with zero density in the training data, whereas a model trained on a single subspace cannot.}
  \label{fig:main}
\end{figure}

\newpage
\tableofcontents

\newpage

\section{Introduction}

Transformer-based large language models (LLMs) \cite{vaswani2017attention} have revolutionized natural language processing and driven significant progress across a wide range of domains, including logical reasoning~\cite{wei2022chain}, sentiment classification~\cite{chen2024retrieval, wang2024chatgpt, xu2024improving}, machine translation~\cite{vilar2022prompting, agrawal2023context}, and code generation~\cite{li2023large, patel2024evaluating}.
Their success is largely attributed to scaling up model size, which has been shown to improve both performance and sample efficiency~\cite{kaplan2020scalinglawsneurallanguage}. Interestingly, large-scale transformers also exhibit emergent capabilities—abilities that arise only beyond a certain scale~\cite{wei2022emergent}. One striking example is in-context learning (ICL), where a model can perform a task simply by being prompted with a few input–output examples, without any gradient-based updates. This has sparked a variety of research  aimed at understanding the underlying mechanism behind ICL, along with its strengths and limitations.

While there is an abundance of work on ICL, including studies on test-time training~\cite{gozeten2025testtime} and how specific prompts affect performance~\cite{task_specific_prompts, fang2025rethinking}, the literature can be broadly divided into two categories: (i) empirical investigations into the extent to which ICL can solve novel tasks~\cite{garg2022what, raventos2023pretraining, yadlowsky2023pretraining, wang2025can, ahuja2023closer, zhang2024trained, li2024nonlinear, pan2023context, kossen2024incontext}, and (ii) theoretical analyses of the mechanisms underlying ICL~\cite{zhang2024trained, huang2024context, li2024nonlinear, akyurek2023learning, von2023transformers, ahn2023transformers, li2024fine}. Interestingly, both lines of work share a common observation: ICL often generalizes out of distribution (OOD) beyond its training data. For ICL, there are a few types of distribution shift: covariate shifts (i.e., shifts in the inputs), task-function shifts (i.e., shifts in the function that generates outputs), and query shifts (i.e., the query input differs from the inputs in the test prompts). 
\cite{zhang2024trained} showed that while a single-layer linear attention model cannot tolerate input shifts, it can handle shifts in the regression weights (task vectors) well.
\cite{garg2022what} reached a similar conclusion, empirically showing that ICL is relatively robust to distribution shifts in several settings, as transformer performance closely matched that of the least-squares estimator on linear regression tasks.
However, \cite{wang2025can} recently challenged these views on language data, empirically demonstrating that ICL can generally solve only in-distribution tasks. To resolve these contrasting perspectives, the concurrent work of \cite{goddard2025when} investigates task-function shifts and attributes ICL OOD capabilities to pre-training task diversity. Yet, the complex nature of their definition of the task vector made theoretical analysis difficult, restricting their study to empirical results. Taken together, these works highlight the lack of a theoretical framework that clearly explains when ICL can and cannot generalize OOD.

In this work, we present a mathematical model to demystify and quantify the OOD generalization capabilities of ICL. We primarily focus on task distribution shifts by studying ICL in a single-layer linear attention model performing linear regression, where the weight (or task) vectors are sampled from low-dimensional subspaces. This allows us to quantify the distribution shift in the task vector via the principal angles between subspaces, and characterize the OOD test risk as a function of these angles (see \Cref{fig:main} for an illustration). Unlike \cite{goddard2025when}, we then provide theoretical guarantees by proving conditions on the pre-training task vectors under which OOD generalization is possible.
Furthermore, we empirically show that our findings extend beyond linear settings, in that (i) our results hold for nonlinear transformers such as GPT-2; and (ii) our results apply to nonlinear function classes.  Overall, our contributions can be summarized as follows:  
\begin{itemize}[left=0.0em]
    \item \textbf{OOD Generalization When Trained on a Union of Subspaces.} We prove that when the pre-training task vectors are sampled from a union of subspaces, a linear attention model can generalize across all principal angles between the training subspaces. This result implies that ICL can generalize to any subspace within the span of the training subspaces, including regions with zero probability density under the training distribution. We hypothesize that this explains the apparent OOD capabilities of ICL observed in the literature: the testing task vector lies within the span of the training task vectors.
    \item \textbf{No OOD Generalization When Trained on a Single Subspace.} On the other hand, we prove that when the training task vectors are drawn from a single $r$-dimensional subspace, ICL applied to a task vector whose subspace is shifted away from the training subspace by an angle $\theta$ incurs a test risk that depends on $\theta$. This implies that ICL cannot generalize beyond the span of the training subspace.
\end{itemize}
Consequently, our results complement those of \cite{goddard2025when} and advocate for task diversity in the pre-training data, since training on a union of subspaces can be viewed as a form of task diversity, with each subspace spanning different directions. Given that large-scale transformers are typically trained on vast corpora and thus exposed to diverse pre-training data, we attribute their OOD generalization in ICL to this diversity.

\section{Problem Setup}
\label{sec:setup}

In this section, we present the basic setup for ICL, together with a motivating example that illustrates how OOD generalization arises depending on the pre-training data. We then describe the analysis setup, which involves both linear attention and linear regression task vectors.

\subsection{Preliminaries}

\paragraph{Notation.}
We denote scalars with unbolded letters (e.g., $m, M$), vectors with bold lower-case letters (e.g., $\mbf{x}$)
and matrices with bold upper-case letters (e.g., $\mbf{X}$).
%Given any $n \in \N$,
We use $\mbf{I}_n$ to denote an identity matrix of size $n \in \mbb{N}$. We use $\mc{R}(\mbf{X})$ to denote the range or the column space of the matrix $\mbf{X}$. Lastly, given any $n \in \mbb{N}$, we use $[n]$ to denote the index set $\{1, \ldots, n\}$. 

\paragraph{ICL Setup.}
Given a sequence of $n$ input-output example pairs $\{\mbf{x}_i, y_i\}_{i=1}^n \subset \mbb{R}^d \times \mbb{R}$, the objective of ICL is to predict the output $y_{n+1} \in \mbb{R}$ corresponding to an unseen query $\mbf{x}_{n+1} \in \mbb{R}^d$. Following prior works  \cite{garg2022what}, we assume each output is generated via $y_i = f(\mbf{x}_i)$ for some function $f(\cdot)$, where $f \in \mc{F}$ is sampled from a distribution over a function class $\mc{F}$. By convention, a transformer takes in these  $n + 1$ pairs as an input prompt $\mbf{Z} \in \mbb{R}^{(n+1) \times (d+1)}$ constructed in the following form: 
\begin{align*}
    \mbf{Z} &= \begin{bmatrix}
    \mbf{z}_1 & \ldots & \mbf{z}_n & \mbf{z}_{n+1}
    \end{bmatrix}^\top = \begin{bmatrix}
    \mbf{x}_1 & \ldots & \mbf{x}_n & \mbf{x}_{n+1} \\
    y_1 & \ldots & y_n & 0
    \end{bmatrix}^\top,
\end{align*}
where $\mbf{z}_i := \begin{bmatrix}
    \mbf x_i^\top & y_i
\end{bmatrix}^\top$ and $\mbf{z}_{n + 1} := \begin{bmatrix}
    \mbf x_{n + 1}^\top & 0
\end{bmatrix}^\top$. 

Then, a transformer $g_{\texttt{ATT}}$, parameterized by weights $\mc{W}$, takes these prompts as input and is trained by minimizing the following expected squared loss with respect to $\mc{W}$:
\begin{align}
\label{eqn:expected_lin_att_objective}
\underset{\mc{W}}{\min} \,\, \mc{L}_{\texttt{ATT}}(\mc{W}) := %\underset{\mc{W}}{\min} \,\, 
\mbb{E}\left[\left(y_{n+1} - g_\texttt{ATT}(\mbf{Z})  \right)^2 \right].
\end{align}
During inference time, we test the trained model, denoted as $g^\star_{\texttt{ATT}}$, using $m + 1$ paired examples $\{\mbf{x}_j, \tilde{y}_j\}_{j=1}^{m+1}$. The input prompts are constructed in the same manner:
\begin{align*}
&\widetilde{\mbf{Z}} = \begin{bmatrix}
    \mbf{x}_1 & \ldots & \mbf{x}_n & \mbf{x}_{n+1} \\
    \widetilde{y}_1 & \ldots & \widetilde{y}_n & 0
    \end{bmatrix}^\top 
    \,\, \text{and} \quad \widetilde{\mbf{z}}_{m+1} = \begin{bmatrix}
         \mbf{x}_{m+1} \\ 0
    \end{bmatrix}.
\end{align*}
In this work, we are interested in ICL's OOD generalization abilities under distribution shifts in the underlying task function, i.e., %Since we are interested in the OOD capabilities of transformers, 
the labels are generated via $\widetilde{y}_j = \widetilde{f}(\mbf{x}_j)$, where $\widetilde{f} \in \widetilde{\mathcal{F}} \neq \mathcal{F}$. %for some function $\widetilde{f} \neq f$. %, i.e., the task function between training and testing are different.

\begin{figure}[t!]
    \centering
    \includegraphics[width=0.5\linewidth]{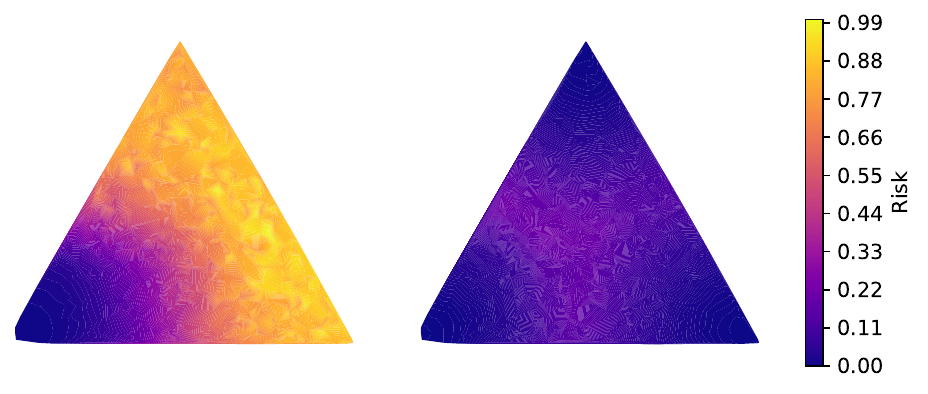}
    \caption{Visualization of the generalization behavior of transformers for learning cosine functions in-context. Each corner of a triangle represents a one-dimensional subspace spanned by $\psi^C_1$ (bottom left), $\psi^C_2$ (bottom right), or $\psi^C_3$ (top), with all possible convex combinations given by the interior. Left: Test risk when trained on $\mbox{span}(\{\psi^C_1\})$. Right: Test risk when trained on $\mbox{span}(\{\psi^C_1\}) \cup \mbox{span}(\{\psi^C_2\}) \cup \mbox{span}(\{\psi^C_3\})$.}
    \label{fig:cosines}
\end{figure}

\subsection{Case Study: When Can In-Context Learning Generalize Out-of-Distribution?}
\label{sec:case_study}

 %As discussed, there has been recent interest in OOD generalization in ICL \cite{garg2022what,zhang2024trained,wang2025can,goddard2025when}.
 Now, we present a case study on how the pre-training distribution affects OOD generalization of ICL. Specifically, %we consider a the function the space $L^2([0, 1])$ under the uniform measure, which models rich sets of signals commonly observed in real-world data. 
we train a GPT-2  model\footnote{We defer the experimental details to \Cref{sec:experiments}.
} by constructing input prompts with inputs $x \sim \mc{U}([0, 1])$ (and hence $d=1$) and generate labels via cosine bases, which model rich sets of signals in real-world data:
\begin{equation*}
    \psi_k^C(x) = \frac{1}{\sqrt{2}}\cos(k\pi x) \quad \mbox{for} \; k \in \mathbb{N}.
\end{equation*}
In \Cref{fig:cosines} (left), we show the test risk when training on prompts drawn from $\mbox{span}(\{\psi^C_1\})$ and testing on prompts drawn from $\mbox{span}(\{\psi^C_1, \psi^C_2, \psi^C_3\})$. As the test prompts shift away from the pre-training distribution $\mbox{span}(\{\psi^C_1\})$, the test risk increases, indicating that transformers are not robust to distribution shifts in this case. This raises the question: under what conditions on the pre-training data does generalization to such regions become possible?

In \Cref{fig:cosines} (right), we present the test risk when training on the union of rank-one subspaces $\mbox{span}(\{\psi^C_1\}) \cup \mbox{span}(\{\psi^C_2\}) \cup \mbox{span}(\{\psi^C_3\})$, i.e., we sample $k \in \{1, 2, 3\}$ uniformly and draw the prompt from $\mbox{span}(\{\psi^C_k\})$. In this case, the test risk is negligible in all regions, even in those with zero probability density under the pre-training distribution. This implies transformers generalizes to OOD tasks in this setting. We hypothesize that this may explain why ICL achieves OOD generalization: the test data actually lies within the span of the training data.
These observations motivate us to  formalize this notion with a tractable analytical setting. To facilitate such analysis, we study a single-layer linear attention model and linear regression data.

\subsection{Single-Layer Linear Attention}

We now introduce the linear attention architecture. \cite{ahn2024linear} empirically demonstrated that many phenomena observed in vanilla transformers can also be replicated in transformers with linear attention. These results motivated subsequent works~\cite{ahn2023transformers, zhang2024trained, li2024fine} to adopt linear attention as a testbed for studying ICL, which we also adopt for analysis.

For linear attention, we employ a causal mask following existing work~\cite{ahn2023transformers, mahankali2024one}. Given prompt $\mbf{Z}$, we construct the following:
\begin{align*}
\mbf{Z}_{\mc{M}} = 
    \begin{bmatrix}
    \mbf{z}_1 & \ldots & \mbf{z}_n & \mbf{0}
    \end{bmatrix}^\top  \,\, \text{and} \quad \mbf{z}_{n+1} = \begin{bmatrix}
         \mbf{x}_{n+1} \\ 0
    \end{bmatrix}.
\end{align*}
Then, a single-layer linear attention model sets $g_{\texttt{ATT}}$ as follows to make the prediction $\widehat{y}_{n+1}$: %is parameterized such that it takes the following sequence-to-sequence form:
\begin{align}
\label{eqn:linear_att} 
\widehat{y}_{n+1} &= g_\texttt{ATT}(\mbf{Z}) = \frac{1}{n}\left(\mbf{z}_{n+1}^\top\mbf{W}_{Q} \mbf{W}_K^\top \mbf{Z}_\mc{M}^\top \right) \mbf{Z}_\mc{M} \mbf{W}_V \mbf{p},
\end{align}
where $\mbf{p} =[ \mbf{0}_d \,\,\, 1]^\top$
and $\mbf{W}_K, \mbf{W}_{Q}, \mbf{W}_V \in \mbb{R}^{(d+1) \times (d+1)}$ are the key, query, and value weight matrices, respectively. This sets $\mc{W} = \{\mbf{W}_K, \mbf{W}_{Q}, \mbf{W}_V\}$ as the collection of trainable weights corresponding to the linear attention model. 
Then, let $\mc{W}^\star = \{\mbf{W}_K^\star, \mbf{W}_Q^\star, \mbf{W}_V^\star\}$ be the optimal weights obtained by minimizing the loss in \Cref{eqn:expected_lin_att_objective}. During inference time, we test the optimal linear attention model $g^\star_{\texttt{ATT}}$ using the $m + 1$ paired examples:
\begin{align*}
   \widehat{y}_{m+1} &=  g^\star_\texttt{ATT}\left( \widetilde{\mbf{Z}} \right) = \frac{1}{m}\left(\widetilde{\mbf{z}}_{m+1}^\top\mbf{W}_Q^\star \mbf{W}_K^{\star\top} \widetilde{\mbf{Z}}_\mc{M}^\top \right) \widetilde{\mbf{Z}}_\mc{M} \mbf{W}_V^\star \mbf{p}.
\end{align*}
Notice at inference time, we normalize by a factor of $m$ instead of $n$.

\subsection{ICL with Linear Regression}

The most commonly studied ICL task is linear regression, which specifies  $f(\mbf{x}) = \mbf{w}^\top \mbf{x}$ for some input $\mbf{x} \in \mbb{R}^d$ and weight (or task) vector $\mbf{w} \in \mbb{R}^d$. In this work, we also adopt the linear regression setting, but assign particular forms to the task vectors to mimic the setup of \Cref{sec:case_study} using low-dimensional subspaces. Below, we specify the corresponding training and testing distributions.

\paragraph{Training Distribution.}

Recall that in \Cref{fig:cosines} (left), we trained on prompts drawn from a single cosine basis, while in \Cref{fig:cosines} (right), we trained on prompts drawn from a union of cosine bases. %For simplicity, we focus on the $K=2$ bases case, but we also show that our results extend to $K > 2$.
To this end, suppose that $d \geq 2r$ and let $\mbf{U}_{s} \in \mbb{R}^{d\times r}$ and $\mbf{U}_{s, \perp} \in \mbb{R}^{d\times r}$  be two $r$-dimensional orthonormal bases in $\mbb{R}^d$. Consider the following two covariance matrices:
\begin{align*}
    &\mbf{\Sigma}_s = \mbf{U}_s \mbf{U}_s^\top + \epsilon \cdot \mbf{I}_d \quad \text{and} \quad \mbf{\Sigma}_{s, \perp} = \mbf{U}_{s, \perp} \mbf{U}_{s, \perp}^\top + \epsilon \cdot \mbf{I}_d,
\end{align*}
where $\epsilon > 0$ is a small constant included to ensure a non-degenerate distribution.
For training, we consider two separate models trained on two different task vector distributions. Specifically, each feature and label pair $(\mbf{x}_i, y_i)$ is generated as follows. For all $i \in [n+1]$, let $\mbf{x}_i \sim \mc{N}(\mbf{0}, \mbf{I}_d)$ and
 \begin{align}
 \label{eqn:vanilla_setup}
y_i = \mbf{w}^\top \mbf{x}_i + \eta_i,
\end{align}
where $\eta_i \sim \mathcal{N}(0, \sigma^2)$ is iid noise with variance $\sigma^2$. Finally, we consider the following distributions for the task vector $\mbf{w}$: 
\begin{align}
\label{eqn:single_subspace_setup}
\mbf{w}\sim \mathcal{N}(\mbf{0}, \mbf{\Sigma}_s) \tag{Single Subspace}, 
\end{align}
or 
\begin{align}
\label{eqn:mix_subspaces_setup}
\mbf{w} \sim
    \begin{cases}
    \mathcal{N}(\mbf{0}, \mbf{\Sigma}_{s}) &\text{w.p.} \quad  \gamma,  \\   
    \mathcal{N}(\mbf{0}, \mbf{\Sigma}_{s, \perp})  &\text{w.p.} \quad  1-\gamma,
    \end{cases}\tag{Union of Subspaces} 
\end{align}
where $0 < \gamma < 1$ denotes the mixture probability. Under these two distributions, we show that the trained models exhibit different OOD generalization behaviors with respect to the testing distribution defined in the following section. Finally, in the union of subspaces distribution, while we focus on $K = 2$ bases for ease of exposition, we also generalize our results to consider a union of $K > 2$ bases.

\paragraph{Testing Distribution.} Recall that in \Cref{sec:case_study}, we tested the transformer with all possible convex combinations of the cosine bases. Similarly, we aim to define a testing subspace $\mbf{U}_t \in \mbb{R}^{d\times r}$ such that it represents a region between the training subspaces. To this end, we parameterize $\mbf{U}_t$ as such \cite[Section~3.8]{absil2004riemannian}: 
\begin{align}
\label{eq:task_subspace}
\mbf{U}_t = \mbf{U}_{s}\cdot \mathrm{cos}\big(\mbf{\Theta}\big) + \mbf{U}_{s,\perp} \cdot \mathrm{sin}\big( \mbf{\Theta} \big),
\end{align}
where 
$\mbf{\Theta} \in \mbb{R}^{r \times r}$ is a diagonal matrix with elements 
$\theta_i \in \left[ 0, \frac{\pi}{2} \right]$ for all $i \in [r]$, %\qq{I feel it is not neccessary to include $\tau_i$ here, having $\theta_i$ is enough for the presenation and be concise}
while $\cos(\cdot)$ and $\sin(\cdot)$ are only applied to the diagonal elements of $\mbf{\Theta}$. Here, $\theta_i$ represents the $i$-th principal angle between $\mbf{U}_s$ and $\mbf{U}_t$. For simplicity, we will assume all principal angles are equal, i.e., for all $i \in [r]$, $\theta_i = \theta$ for some $\theta \in \left[ 0, \frac{\pi}{2} \right]$ so that
$\mbf{\Theta} = \theta \cdot \mbf{I}_r$. Notice when $\theta = 0$, $\mbf{U}_t = \mbf{U}_{s}$, and when $\theta = \frac{\pi}{2}$, $\mbf{U}_t = \mbf{U}_{s,\perp}$. Hence, for varying values of $\theta$, $\mbf{U}_t$ can be viewed as a ``slice'' between $\mbf{U}_s$ and $\mbf{U}_{s, \perp}$.
Finally, we parameterize $\mbf{\Sigma}_t$ as
\begin{align}
\label{eqn:covariance_t}
    \mbf{\Sigma}_t &= \mbf{U}_t \mbf{U}_t^\top + \epsilon \cdot \mbf{I}_d, 
\end{align}
and generate each testing pair $(\mbf{x}_j, \widetilde{y}_j)$  independent of the training data in a similar fashion: for all $j \in [m + 1]$, let $\mbf{x}_j \sim \mc{N}(\mbf{0}, \mbf{I}_d)$ and 
\begin{align}
\label{eq:vanilla_setup_test}
\widetilde{y}_j = \widetilde{\mbf{w}}^\top \mbf{x}_j + \eta_j, \quad \text{where} \quad \widetilde{\mbf{w}}\sim \mathcal{N}(\mbf{0}, \mbf{\Sigma}_t)
\end{align}
and $\eta_j \sim \mathcal{N}(0, \sigma^2)$ is again iid noise.

\section{Main Results}
\label{sec:main_results}

\subsection{Transformers Can Generalize to the Span When Trained on a Union of Subspaces}

In this section, we derive the test risk of the optimal linear attention model trained with prompts whose task vectors are drawn from a union of two subspaces, and tested on a task vector from subspace $\mbf{U}_t$ at any angle $\theta \in [0, \tfrac{\pi}{2}]$. The following result shows that for sufficiently large prompt lengths, the test risk can be arbitrarily close to the optimal test risk (the label noise variance) independent of $\theta$. This implies that linear attention is robust to such subspace shifts in this setting.

\begin{figure*}[t!]
    \centering
    \includegraphics[width=0.85\linewidth]{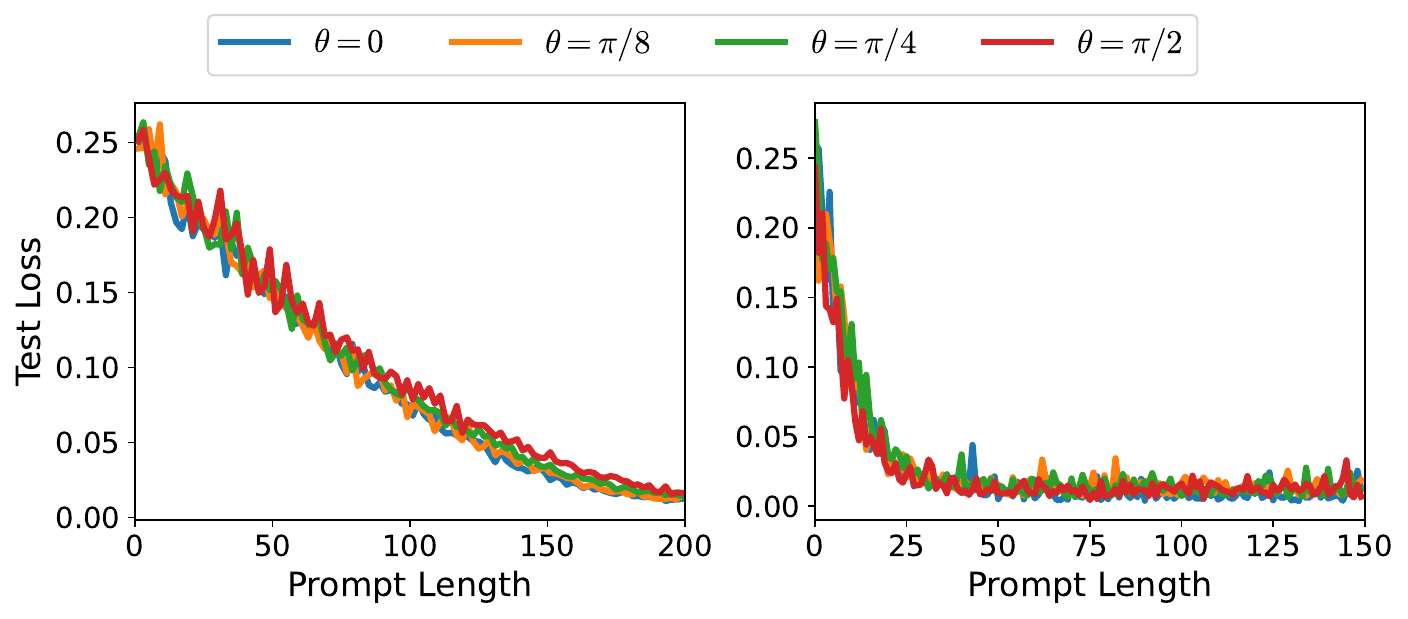}
    \caption{Plot of the test risk as a function of the prompt length for a linear transformer (left) and a GPT-2 model (right). When the prompt length at test time is large enough, the test risk goes nearly to zero for all $\theta \in \left[0, \frac{\pi}{2} \right]$, corroborating Theorem~\ref{thm:mix_two_subspaces} in that both linear and nonlinear transformers can generalize to the span of the training task vectors at test-time. }
    \label{fig:pos_result}
\end{figure*}

\begin{tcolorbox}
\begin{theorem}[Test Risk Under a Union of Subspaces]
\label{thm:mix_two_subspaces}
    Let $g_{\texttt{ATT}}^\star$ denote the optimal linear attention model corresponding to the independent data setting in \Cref{eqn:vanilla_setup}, where the task vector is drawn from \Cref{eqn:mix_subspaces_setup} with $\gamma = 0.5$. For all $j \in [m+1]$, suppose that the prompts at test time are constructed with features $\mbf{x}_j \sim \mc{N}(\mbf{0}, \mbf{I}_d)$ and labels 
    whose task vectors are drawn from \Cref{eqn:covariance_t}. For any  $\theta \in \left[ 0, \frac{\pi}{2} \right]$ and $\delta \in (0, r)$, if
    \begin{equation*}
        m \geq n > \frac{\left(2 (r + \sigma^2) + 1\right) r}{\delta} - \left(2 (r + \sigma^2) + 1\right),
    \end{equation*}
    then we have
    \begin{align*}
    \lim\limits_{\epsilon \to 0}\, \mbb{E}\left[ \left(\widetilde{y}_{m+1} - g^\star_{\texttt{ATT}}\left( \widetilde{\mbf{Z}} \right) \right)^2 \right] < \sigma^2 + \delta.
    \end{align*}
\end{theorem}
\end{tcolorbox}
 In \Cref{thm:mix_two_subspaces}, we take $\epsilon \to 0$ for two reasons: (i) to eliminate any dependence on $\epsilon$ and isolate its effect on test risk as it is assumed to be a small constant, and (ii) analyze the test risk when the covariance matrices are exactly low-rank.
In this limit, task vectors sampled from the test subspace $\mbf{U}_t$ have zero probability density in the training task distribution. Hence, \Cref{thm:mix_two_subspaces} implies ICL generalizes well to this OOD task. We hypothesize this can explain why ICL achieves OOD generalization: the test data actually lies within the span of the training data. In \Cref{fig:pos_result}, we corroborate this theory on both linear attention and a GPT-2 model, showing that for sufficiently large prompt lengths, the test risk indeed tends to zero. Experimental details are deferred to \Cref{sec:experiments}.

%The analysis involves deriving the test risk under an arbitrary distribution shift, assuming the linear attention model is parameterized by the optimal weights according to loss \Cref{eqn:expected_lin_att_objective}. 

We provide some intuition behind the proof of \Cref{thm:mix_two_subspaces}, deferring the full proof to \Cref{sec:proof_of_k_subspaces}. At the optimal linear attention weights under the loss in \Cref{eqn:expected_lin_att_objective}, the model reduces to a single step of projected gradient descent (PGD)~\cite{li2024fine, ahn2023transformers, von2023transformers, mahankali2024one}.
Denoting $\mbf{A} \in \mbb{R}^{d \times d}$ as the PGD projection matrix that arises from the optimal weights, 
we sketch how the dependence on $\theta$ is mitigated (assuming $\sigma = 0$ and $\delta \to 0$ for simplicity):
\begin{align*}
    \widehat{y}_{m+1} = g^\star_{\texttt{ATT}}\left( \widetilde{\mbf{Z}} \right)   = \frac{1}{m} \mbf{x}_{m+1}^\top \mbf{A} \mbf{X}^\top \underbrace{\mbf{X} \widetilde{\mbf{w}}}_{=\mbf{y}} &\to \mbf{x}_{m+1}^\top \mbf{U}_{2r} \mbf{U}_{2r}^\top \widetilde{\mbf{w}}
    \tag{$m, n \to \infty$ and $\epsilon \to 0$} \\
    &= \mbf{x}_{m+1}^\top \mbf{U}_{2r} \mbf{U}_{2r}^\top \mbf{U}_t \mbf{g}, \tag{$\widetilde{\mbf{w}} = \mbf{U}_t \mbf{g}$}
\end{align*}
where $\mbf{g} \sim \mc{N}(\mbf{0}, \mbf{I}_r)$, and we have defined the following:
\begin{align*}
    \mbf{X} &\coloneqq \begin{bmatrix}
    \mbf{x}_1 & \dots & \mbf{x}_m
\end{bmatrix}^\top \in \mbb{R}^{m \times d}, \quad 
\mbf{y} \coloneqq \begin{bmatrix}
    \widetilde{y}_1 & \dots & \widetilde{y}_m
\end{bmatrix}^\top \in \mbb{R}^m, \quad 
\mbf{U}_{2r} \coloneqq \begin{bmatrix}
    \mbf{U}_{s} & \mbf{U}_{s, \perp}
\end{bmatrix} \in \mbb{R}^{d\times 2r}.
\end{align*}
Since 
$\mc{R}(\mbf{U}_t) \subset \mc{R}(\mbf{U}_{2r})$ for all $\theta \in \left[0, \frac{\pi}{2}\right]$, the trained model perfectly recovers $\widetilde{y}_{m+1}$. Lastly, we set $\gamma = 0.5$ to simplify constants in the test risk, though our results hold for any $\gamma \in (0,1)$.

\paragraph{Beyond a Mixture of Two Gaussians.}
 We now generalize the above result to consider a mixture of $K > 2$ Gaussians.  This model has been extensively studied in machine learning \cite{vidal2011subspace} and has recently gained attention in deep learning \cite{wang2024diffusion,xu2025understanding,wang2024a}. Specifically, assume $d \geq Kr$, and let $\mbf{U} = \begin{bmatrix}
    \mbf{u}_1 & \dots & \mbf{u}_d
\end{bmatrix} \in \mbb{R}^{d \times d}$ be an orthonormal basis for $\mbb{R}^d$. For all $k \in [K]$, we define  $\mbf{U}_{s, k} = \begin{bmatrix}
    \mbf{u}_{(k-1)\cdot r + 1} & \dots & \mbf{u}_{kr}
\end{bmatrix} \in \mbb{R}^{d \times r}$. Note that $\mbf{U}_{s, k}^\top \mbf{U}_{s, l} = \mbf{0}_{r \times r}$ for all $k \neq l$. Then, we assume the training task $\mbf{w} \in \mbb{R}^d$ is sampled as such:
\begin{align}
    \label{eqn:task_multiple_mog}
    &\mbf{w} \sim \sum\limits_{k=1}^K \gamma_k \cdot \mathcal{N}(\mbf{0}, \mbf{\Sigma}_{s, k}), %\; \; \\
    %&\text{where} \; \; \mbf{\Sigma}_{s, k} = \mbf{U}_{s, k}\mbf{U}_{s, k}^\top + \epsilon \cdot \mbf{I}_d \; \; \text{and} \; \; \sum\limits_{k=1}^K \gamma_k = 1.
\end{align}
where $\mbf{\Sigma}_{s, k} = \mbf{U}_{s, k}\mbf{U}_{s, k}^\top + \epsilon \cdot \mbf{I}_d$ and $\sum\limits_{k=1}^K \gamma_k = 1.$ At inference time, we define an orthonormal basis $\overline{\mbf{U}}_t \in \mbb{R}^{d \times r}$ that lies within the span of $\{\mbf{U}_{s, k}\}_{k=1}^K$:
\begin{equation} \label{eq:U_t_bar}
    \overline{\mbf{U}}_t = \sum\limits_{k=1}^K \alpha_k \mbf{U}_{s, k}, \; \; \text{for} \; \; \{\alpha_k\}_{k=1}^K \; \text{s.t.} \; \sum\limits_{k=1}^K \alpha_k^2 = 1,
\end{equation}
where the constraint on $\{\alpha_k\}_{k=1}^K$ ensures $\overline{\mbf{U}}_t \in \mbb{R}^{d\times r}$ is an orthonormal basis. Then, similar to \Cref{thm:mix_two_subspaces}, we consider testing on task vectors $\widetilde{\mbf{w}} \sim \mathcal{N}\left(\mbf{0}, \overline{\mbf{\Sigma}}_t \right)$ with $\overline{\mbf{\Sigma}}_t = \overline{\mbf{U}}_t\overline{\mbf{U}}_t^\top + \epsilon \cdot \mbf{I}_d$. We again emphasize $\overline{\mbf{U}}_t$ is unseen during training, but lies within the span of the training subspaces.

\begin{tcolorbox}[breakable]
\begin{theorem}
\label{thm:mixture_k_subspaces}
Let $g_{\texttt{ATT}}^\star$ denote the optimal linear attention model corresponding to the independent data setting in Equation~(\ref{eqn:vanilla_setup}), where the task vector is drawn from \Cref{eqn:task_multiple_mog} with $\gamma_k = \frac{1}{K}$ for all $k \in [K]$. For all $j \in [m+1]$, suppose that the test prompts are constructed with features $\mbf{x}_j \sim \mc{N}(\mbf{0}, \mbf{I}_d)$ and labels whose task vectors are constructed with subspaces defined in \Cref{eq:U_t_bar}. For any $\{\alpha_k\}_{k=1}^K$ s.t. $\sum\limits_{k=1}^K \alpha_k^2 = 1$ and $\delta \in (0, r)$, if
    \begin{equation*}
        m \geq n > \frac{(K(r + \sigma^2) + 1)r}{\delta} - (K(r + \sigma^2) + 1),
    \end{equation*}
    then we have
    \begin{align*}
    \lim\limits_{\epsilon \to 0} \mbb{E}\left[ \left(\widetilde{y} - g^\star_{\texttt{ATT}}\left( \widetilde{\mbf{Z}} \right) \right)^2 \right] < \sigma^2 + \delta.
    \end{align*}
    
    %Then, for any $\delta > 0$, if
    % \begin{equation}
    %   \min\{ m, n \} \geq \frac{Kr(Kr + \sigma^2 + 1)}{\delta} + \sqrt{\frac{(2K-1)\left( K(r + \sigma^2) + 1 \right)r}{\delta}},
    % \end{equation}
    %then we have $\lim\limits_{\epsilon \rightarrow 0} \left( \mbb{E}\left[ \left(\widetilde{y}_{m+1} - g^\star_{\texttt{ATT}}(\widetilde{\mbf{z}}_{m+1}, \widetilde{\mbf{Z}}_{\mc{M}})\right)^2 \right] - \mathcal{L}_s^\star \right) < \delta$, where $\mathcal{L}_s^\star$ is the optimal risk. %\ax{Maybe we can just say $\mathcal{O}\left(\frac{K^2 r^2}{\delta}\right)$ to clean it up? Or just make it a limiting result to make it consistent w/ the others; I don't know if this sample complexity is interesting}
\end{theorem}
\end{tcolorbox}
The proof is deferred to \Cref{sec:proof_of_k_subspaces}.
Similar to \Cref{thm:mix_two_subspaces}, if the linear attention model is trained on task vectors that lie in a union of $K$ subspaces, it can generalize well to any region within the span of the $K$ subspaces, even if those regions have zero probability density during training. 
Lastly, note that by setting $K = 2$, $\alpha_1 = \cos(\theta)$, and $\alpha_2 = \sin(\theta)$, we exactly recover \Cref{thm:mix_two_subspaces}.

\subsection{Transformers Cannot Generalize When Trained on a Single Subspace}

Previously, we saw that task vectors drawn from a union of subspaces—which can be viewed as a form of task diversity—can enable OOD generalization. Similar to \Cref{sec:case_study}, we now aim to identify a setting in which generalization beyond the pre-training data is not possible. The following result shows that if the task vectors are drawn from a Gaussian whose covariance matrix spans only a single subspace $\mbf{U}_s \in \mbb{R}^{d \times r}$, then testing a linear attention model on a task vector shifted away from the training subspace by an angle $\theta$ yields a test risk with a non-negligible dependence on $\theta$, even as the prompt length tends to infinity.

\begin{figure*}[t!]
    \centering
    \includegraphics[width=0.95\linewidth]{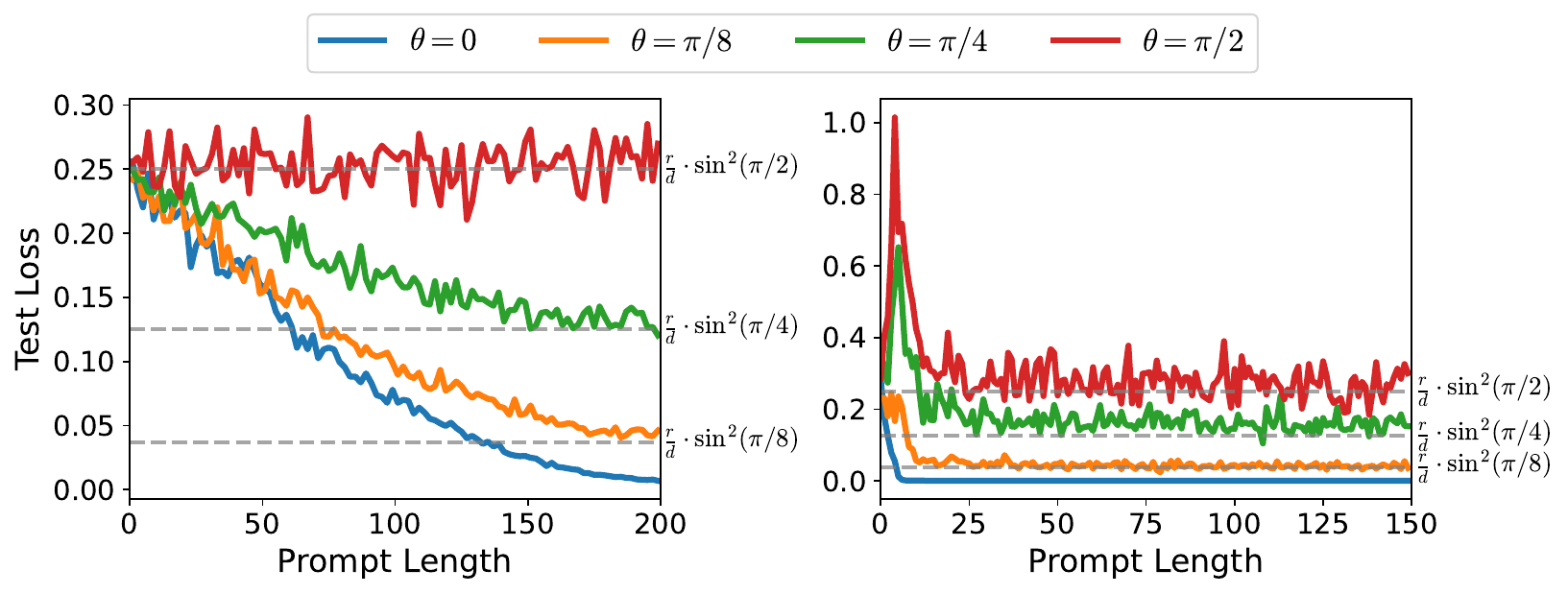}
    \caption{Plot of the normalized test risk as a function of the prompt length for a linear transformer (left) and a GPT-2 model (right) under covariance shifts. As the covariance at test time shifts away from the covariance used at training time as a function of $\theta$, the test risk exhibits a non-negligible dependence on $\theta$ for both the linear and nonlinear transformer. Moreover, for both models, the test risk exactly matches the predicted risk from~\Cref{prop:neg_result}. }
    \label{fig:neg_result}
\end{figure*}

\begin{tcolorbox}%[breakable]

\begin{proposition}[Test Risk Under a Single Subspace] \label{prop:neg_result}
    Let $g_{\texttt{ATT}}^\star$ denote the optimal linear attention model corresponding to the independent data setting in \Cref{eqn:vanilla_setup}, where the task vector is drawn from \Cref{eqn:single_subspace_setup}. For all $j \in [m+1]$, suppose that the prompts at test time are constructed with features $\mbf{x}_j \sim \mc{N}(\mbf{0}, \mbf{I}_d)$ and labels 
    whose task vectors are drawn from \Cref{eqn:covariance_t}.
    Then, we have
    \begin{align*}
        \lim\limits_{m \rightarrow \infty} \lim\limits_{n \rightarrow \infty} \lim\limits_{\epsilon \rightarrow 0} \, \mbb{E}&\left[ \left(\widetilde{y}_{m+1} - g^\star_{\texttt{ATT}}\left( \widetilde{\mbf{Z}} \right)  \right)^2 \right] = r \sin^2(\theta) + \sigma^2,
    \end{align*}
    where $\theta \in [0, \frac{\pi}{2}]$ are the $r$ principal angles between $\mbf{U}_{s} \in \mbb{R}^{d\times r}$ and $\mbf{U}_t\in \mbb{R}^{d\times r}$. 
\end{proposition}
\end{tcolorbox}
The proof is provided in Appendix~\ref{sec:proof_of_neg_result}. In the asymptotic regime, our result reveals the following: when $\theta = 0$, the $\sin(\cdot)$ term vanishes, allowing perfect recovery up to the label noise variance. However, as $\theta$ increases from $0$ to $\frac{\pi}{2}$, the test risk increases with respect to $\theta$. At $\theta = \frac{\pi}{2}$, %(or equivalently, when $t = 1$), 
the test risk becomes exactly the rank of the covariance matrix. Notably, this represents the largest possible error in this setting, as a low-rank covariance matrix induces an error dependent on the rank rather than the ambient dimension, as observed in related work~\cite{garg2022what, oko2025pretrained}. 

The proof technique is similar to that of \Cref{thm:mix_two_subspaces}: 
\begin{align*}
    \widehat{y}_{m+1} = g^\star_{\texttt{ATT}}\left( \widetilde{\mbf{Z}} \right)  &\to \mbf{x}_{m+1}^\top \mbf{U}_{s} \mbf{U}_{s}^\top \widetilde{\mbf{w}}
    \tag{$m, n \to \infty$ and $\epsilon \to 0$} \\
    &= \mbf{x}_{m+1}^\top \mbf{U}_{s} \mbf{U}_{s}^\top \mbf{U}_t \mbf{g}, \tag{$\widetilde{\mbf{w}} = \mbf{U}_t \mbf{g}$}
\end{align*}
By taking appropriate limits, it is easy to see that the dependence on $\theta$ arises from $\mbf{U}_s^\top \mbf{U}_t$, which reflects a rotation by an angle $\theta$ between the subspaces. Since $\mbf{A} \to \mbf{U}_s \mbf{U}_s^\top$ in the asymptotic regime, PGD projects the data onto an ``incorrect'' subspace, thereby inducing an error proportional to $\theta$ in the test risk.

In \Cref{fig:neg_result}, we present experiments corroborating \Cref{prop:neg_result} on linear attention and a GPT-2 model. %\qq{needs to be clear that this is under our linear regression setup}
Interestingly, our experiments show that both models incur the same test risk under the distribution shift when given enough in-context examples. This implies that the linear attention model can adequately capture the behavior in this setting, and that the observed error is not merely an artifact of using a linear model. Lastly, we assumed equal principal angles between the subspaces for simplicity, and defer the more general result to Proposition~\ref{prop:neg_result_diff_angles} in Appendix~\ref{sec:diff_angles}.
\section{Experimental Results}
\label{sec:experiments}

\paragraph{Experimental Setup.}
% \label{sec:exp_setup}

Unless otherwise stated, the experimental setup is as follows: for both the linear and nonlinear Transformer, we consider the noiseless case, and set $d = 20$, $r = 5$, and $\epsilon = 10^{-5}$. To construct the train and test subspaces, we sample an orthogonal matrix $\mbf{U} \in \mbb{R}^{d\times d}$ uniformly at random, set $\mbf{U}_s$ to be the first $r$ columns of $\mbf{U}$, and set $\mbf{U}_{s, \perp}$ to be the second $r$ columns. %and take the first $2r$ columns, which we further split into $\mbf{U}_s \in \mbb{R}^{d\times r}$ and $\mbf{U}_{s, \perp} \in \mbb{R}^{d\times r}$ as illustrated in Equation~(\ref{eqn:split_subspaces}).
Given this setup, we typically consider a mixture of $K = 2$ subspaces for the experiments.

For the experiments with the linear transformer, we plug in the optimal weights according to their respective settings (e.g., optimal weights using a single subspace or a mixture of subspaces) and set $m = n = 200$.
For the nonlinear transformer, following~\cite{garg2022what}, we use a small GPT-2 model with $6$ layers, $4$ heads, and a $128$-dimensional embedding space. We append a learnable linear transformation to map the vector
predicted by the model to a scalar. We use a learning rate of $\eta = 10^{-4}$, batch size $128$, prompt lengths $m = n = 150$, and train for $100$K iterations. We run all experiments on a single A100 GPU.

\begin{figure*}[t!]
    \centering
     \begin{subfigure}[t!]{0.49\textwidth}
         \centering
        \includegraphics[width=0.9\textwidth]{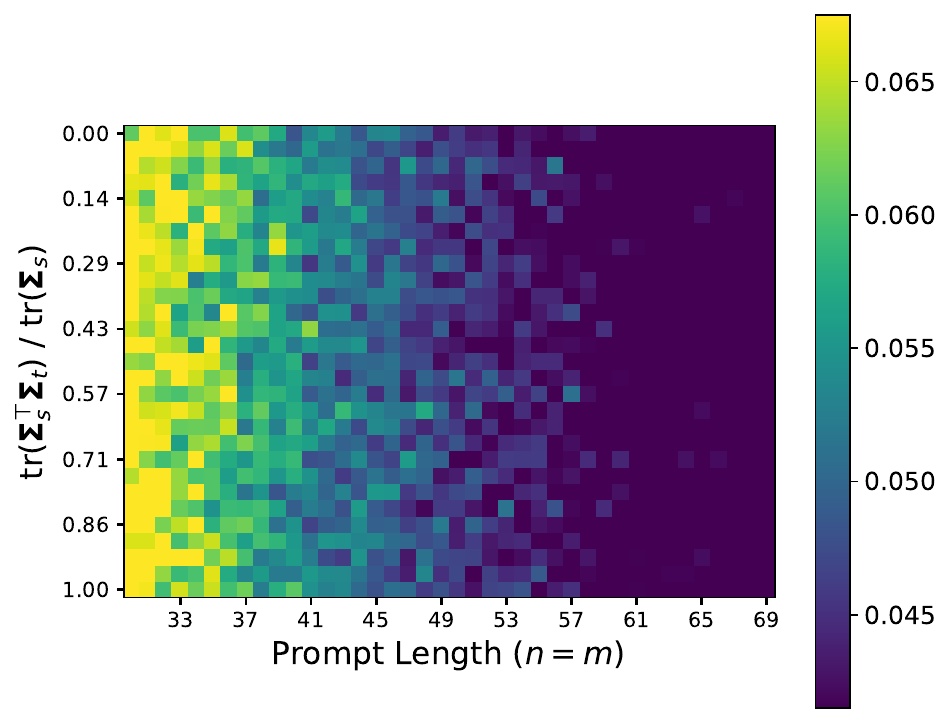}
     \end{subfigure}
     \hfill
     \begin{subfigure}[t!]{0.49\textwidth}
         \centering
         \includegraphics[width=0.8\textwidth]{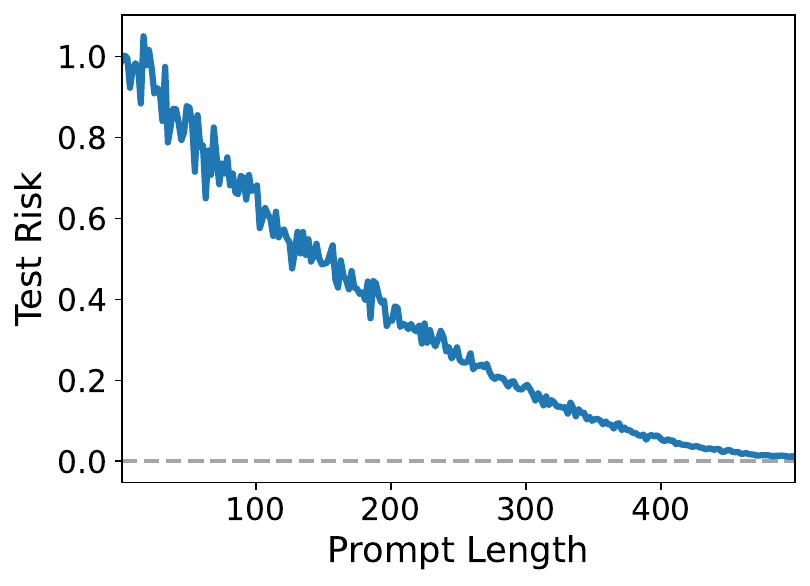}
     \end{subfigure}   
    \caption{Left: Phase plot of the test risk as we vary the angle between $\mbf{\Sigma}_s$ and $\mbf{\Sigma}_t$ and the prompt length with $m=n$ for a linear attention model trained with a mixture of Gaussians. The test risk is low across all angle shifts, and decreases further as the prompt length increases.  Right: Plot of the test risk as a function of the prompt length for a case in which $\mbf{\Sigma}_s \neq \mbf{\Sigma}_t$ but with $\theta = 0$, following the OOD example in~\cite{gatmirylooped2024}. This serves to explain why ICL can seemingly do OOD generalization as observed in the literature.}
    \label{fig:phase_and_unify}
\end{figure*}

\subsection{More Results on Linear Function Classes}

Previously, we presented results on the test risk as a function of the prompt length. In Figure~\ref{fig:phase_and_unify} (left), we present a phase plot of the test risk as a function of $ \tr(\mbf{\Sigma}_s^\top \mbf{\Sigma}_t) \, / \, \tr(\mbf{\Sigma}_s)$ (which measures the angle between two covariance matrices) and the prompt length on linear attention with task vectors drawn from a mixture of two Gaussians. Similar to Figure~\ref{fig:pos_result}, the test risk is low for all values of $m = n$, and it decreases further as the prompt length increases. Note that the largest possible normalized test risk in this setting is $r/d = 0.25$, so the test risk is still considered low even when the prompt length is small.

In Section~\ref{sec:main_results}, we discussed how apparent abilities of ICL to perform OOD generalization arises when the test task lies within the span of the training task vectors. Here, we present an extra experiment to support this claim, using the example from~\cite{gatmirylooped2024}, with $\mbf{\Sigma}_s = \mbf{I}_5$ and $\mbf{\Sigma}_t = \mbf{V}\mbf{\Lambda}_t \mbf{V}^\top$, where $\mbf{V} \in \mbb{R}^{5\times 5}$ is a random orthogonal matrix and $\mbf{\Lambda}_t = \mathrm{Diag}(1, 1, 1/2, 1/4, 1)$. In Figure~\ref{fig:phase_and_unify} (right), we observe that the test risk approaches zero given enough samples. This implies that our result may help explain many observations of OOD generalization in ICL and offers a unifying perspective on findings reported in the literature.

\subsection{More Results on Nonlinear Function Classes}

In this section, we present an additional result using a nonlinear function class to complement both our theory and \Cref{sec:case_study}.
We look at the function space $L^2(\mathbb{R}, e^{-x^2/2} / \sqrt{2\pi} \,dx)$, i.e., square-integrable functions under the Gaussian measure, and construct an orthonormal basis via Hermite polynomials: 
\begin{equation*}
    \psi_k^H(x) = \frac{(-1)^k}{\sqrt{k!}} e^{x^2/2} \frac{d^k(e^{-x^2/2})}{dx^k} \quad \mbox{for} \; k \in \mathbb{N}.
\end{equation*}
In \Cref{fig:hermite}, we observe the same pattern as in \Cref{fig:cosines}: training on a single polynomial does not allow OOD generalization beyond that basis, whereas training on a union of bases enables generalization to all points in the interior. This demonstrates that our theory and setup extend beyond linear settings and hold with greater generality.

\begin{figure}[t!]
    \centering
    \includegraphics[width=0.5\linewidth]{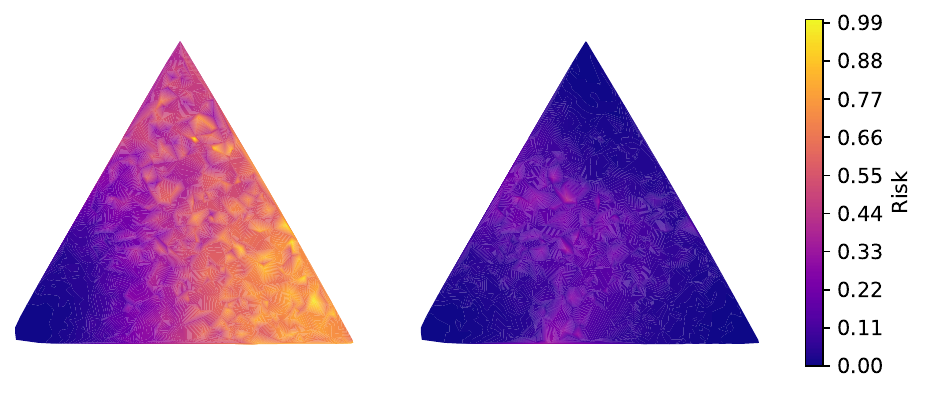}
    \caption{Visualization of the generalization behavior of transformers for learning Hermite polynomials functions in-context. Similar to \Cref{fig:cosines}, each corner of a triangle represents a one-dimensional subspace spanned by $\psi^H_k$. Left: Test risk when trained on $\mbox{span}(\{\psi^H_1\})$. Right: Test risk when trained on $\mbox{span}(\{\psi^H_1\}) \cup \mbox{span}(\{\psi^H_2\}) \cup \mbox{span}(\{\psi^H_3\})$.}
    \label{fig:hermite}
\end{figure}

\section{Conclusion and Related Work}

In this work, we proposed a mathematical framework to analyze the OOD capabilities of ICL. Unlike prior studies, our framework provides theoretical guarantees for when linear transformers can or cannot generalize OOD, based on their pre-training data, which we parameterize using low-dimensional subspaces. We show that pre-training task diversity, defined as a union of subspaces, enables generalization to regions with zero probability density in the training distribution, whereas tasks drawn from a single subspace do not have this property. We further demonstrate that our theoretical results extend empirically to transformer models such as GPT-2, and we present experiments showing that the findings also apply to nonlinear function classes. We conclude with a survey of
related works and subsequently discuss their connections to our work.

\paragraph{ICL on Transformers with Linear Attention.} There is abundant research on ICL that analyzes single-layer linear attention models. Below, we survey several works most relevant to our work; like ours, many of them focus on linear regression settings, where for all $i \in [n+1]$:
\begin{align*}
    y_i = f(\mbf{x}_i) = \mbf{w}^\top \mbf{x}_i + \eta, \quad \text{where} \quad \mbf{w} \sim \mathcal{N}\left(\mbf{0}, \mbf{\Sigma}_{\mbf{w}}\right), \quad \mbf{x}_i \sim \mathcal{N}\left(\mbf{0}, \mbf{\Sigma}_{\mbf{x}}\right),
\end{align*}
and $\eta$ is additive Gaussian noise.
As previously mentioned, Zhang et al. \cite{zhang2024trained} studied the training dynamics of a single-layer linear attention model on the population loss for a linear regression ICL task. Specifically, assuming $\mbf{\Sigma}_{\mbf{w}} = \mbf{I}_d$ and an arbitrary $\mbf{\Sigma}_{\mbf{x}}$, they showed the model weights converge to a globally optimal solution under gradient flow, despite the non-convex objective. They also provide closed-form expressions for the model weights at the global minima. A follow-up work \cite{zhang2024context} considered a linear regression task with $\mbf{w} \sim \mathcal{N}\left(\boldsymbol{\mu}_{\mbf{w}}, \mbf{\Sigma}_{\mbf{w}}\right)$ and a linear Transformer model (a linear attention layer followed by a two-layer linear network). They showed a single linear attention layer incurs a sub-optimal risk that depends on $\boldsymbol{\mu}_{\mbf{w}}$, but adding a linear network allows the model to achieve the Bayes optimal risk. 

Other works \cite{von2023transformers, ahn2023transformers, bai2023transformers,li2024fine, zhang2024context, mahdavi2024revisiting, liildiz_icl} study the underlying learning algorithms that linear attention models implement when learning linear functions in-context. Specifically, for a single linear attention layer, Von Oswald et al. \cite{von2023transformers} demonstrated the existence of model weights that implement a single step of GD on a mean-squared error loss. They further showed empirically that the weights of a trained linear attention layer closely align with those that implement a GD step. Follow-up works \cite{ahn2023transformers, li2024fine, mahdavi2024revisiting} rigorously proved the equivalence between a single step of preconditioned gradient descent (PGD) with zero initialization and the weights of a single-layer linear attention model under the population loss. Specifically, Ahn et al. \cite{ahn2023transformers} theoretically showed when $\mbf{\Sigma}_{\mbf{w}} = \mbf{I}_d$ and $\mbf{\Sigma}_{\mbf{x}}$ is arbitrary, the single-layer linear attention model learns a preconditioning matrix that is dependent on $\mbf{\Sigma}_{\mbf{x}}$. Li et al. \cite{li2024fine} generalized this result by considering an arbitrary $\mbf{\Sigma}_{\mbf{w}}$ in addition to $\mbf{\Sigma}_{\mbf{x}}$ --- they showed the learned preconditioning matrix depends on both $\mbf{\Sigma}_{\mbf{x}}$ and $\mbf{\Sigma}_{\mbf{w}}$. Finally, \cite{zhang2024context} showed when $\mbf{w} \sim \mathcal{N}\left(\boldsymbol{\mu}_{\mbf{w}}, \mbf{\Sigma}_{\mbf{w}}\right)$, a linear attention layer followed by a linear network implements a PGD step while \emph{learning} the initialization. While our work builds on the fact that a single-layer linear attention model implements PGD, our goal differs from these prior works: we study how ICL under this model can generalize out-of-distribution.

% globally optimal weights for a single linear attention layer under the population loss in \Cref{eqn:icl_objective} implements one step of PGD, again under Gaussian assumptions for $\mbf{x}_i, \mbf{x}_q$, and $\mbf{w}$. 

% \begin{itemize}
%     \item \cite{zhang2024trained} shows gradient flow converges to globally optimum single-layer linear attention model on population loss
%     \item \cite{von2023transformers, ahn2023transformers, li2024fine} study the ICL on linear models from an algorithmic perspective. Specifically, they show one a globally-optimal single-layer linear attention model implements one step of (projected) gradient descent on linear regression tasks
% \end{itemize}

%\ax{Maybe we can combine the following two paragraphs somehow? I don't know if we need to discuss training dynamics of nonlinear models in ICL since we don't have  theory in nonlinear settings.}

%\paragraph{ICL on Transformers Beyond Linear Attention.} 
\paragraph{Empirical Observations on OOD Generalization of ICL.} As part of their study, Garg et al. \cite{garg2022what} empirically observed Transformer-based ICL is robust to a number of distribution shifts, such as between the train and test distributions of the features $\mbf{x}_i$, as well as between the features $\mbf{x}_i$ and query $\mbf{x}_q$. These observations inspired an extensive line of empirical work studying ICL's ability to generalize to OOD tasks~\cite{yadlowsky2023pretraining, raventos2023pretraining, ahuja2023closer, kossen2024incontext, pan2023context, fan2024transformers, wang2025can}. To our knowledge, \cite{yadlowsky2023pretraining, wang2025can} are the most closely related with our setting. Specifically, these works consider sampling tasks from a mixture of \emph{function class} distributions, e.g., $f$ is sampled from the class of dense linear functions with probability $\gamma \in (0, 1)$, or from the class of sparse linear functions with probability $1 - \gamma$. Yadlowskey et al. \cite{yadlowsky2023pretraining} showed when Transformers are trained for ICL on a mixture of function classes, ICL cannot generalize well to function classes  not present in the training mixture. Wang et al. \cite{wang2025can} argue if the test task is not in the training mixture, Transformers select a task from the training mixture that minimizes the test error. In contrast, our work assumes that the target function is sampled from a mixture of low-dimensional subspaces in a fixed function space. In other words, the mixture distribution from which we sample is always within a \emph{single} function class. We emphasize this is different from sampling from a mixture of \emph{multiple} function class distributions. 

%In our work, we always sample the target function from a \emph{single} function class distribution $\mathcal{D}_\mathcal{F}$, e.g., the class of dense linear functions. The mixture distribution from which we sample functions is al \emph{within} $\mathcal{D}_\mathcal{F}$. 

% \ax{todo: discuss \cite{yadlowsky2023pretraining, wang2025can}}

\paragraph{Theoretical Studies on OOD Generalization of ICL.} The above empirical observations motivated theoretical  studies on ICL's OOD generalization ability. Under their setting, Zhang et al. \cite{zhang2024trained} studied how a trained single linear attention layer handles various distribution shifts. Assuming the model weights were at the   minima of \Cref{eqn:expected_lin_att_objective}, they derived a closed-form expression for the prediction $\widehat{y}_q$ for a given query $\mbf{x}_q$ and in-context examples $\left( \mbf{x}_1, \mbf{w}^\top \mbf{x}_1, \dots, \mbf{x}_m, \mbf{w}^\top \mbf{x}_m \right)$. Using this expression for $\widehat{y}_q$, they concluded a trained linear attention model is robust to task and query shifts, but cannot tolerate feature shifts well. 

 Other works have studied \emph{nonlinear} models and function classes. For instance, \cite{li2024nonlinear} considered a binary classification ICL task. They showed a sufficiently trained single-layer, single-head Transformer model (one softmax attention layer followed by a two-layer perceptron) can achieve arbitrarily small generalization error when the inference-time features are \emph{linear combinations} of the training features. Another work \cite{yang2024context} assumed the function to learn in-context was $f(\mbf{x}) = \mbf{w}^\top g(\mbf{x}) + \eta$, where $g(\mbf{x}) = \left(g_1(\mbf{x}), \dots, g_\ell(\mbf{x})\right)$ is an arbitrary feature mapping. They showed if $\mbf{w}$ has iid, zero mean, unit variance entries at train time, and $\|\mbf{w}\|_2$ is bounded at inference time, a trained single-layer, multi-head softmax attention model generalizes well under \emph{any} shift in $\mbf{w}$. More recent work studied OOD generalization of ICL through the approximation theory \cite{li2025transformers}, where the required prompt length depends on the complexity of the downstream tasks. Again, our paper differs from these works by studying when ICL can and cannot perform OOD generalization, particularly by using low-dimensional subspaces to parameterize the covariance matrices. 

\paragraph{Learning Functions with Low-Dimensional Structure In-Context.}

To the best of our knowledge, the work by \cite{oko2025pretrained} is the only most related work that also considers learning functions with low-dimensional structures. In their setting, the function to learn in-context is a single-index model $f(\mbf{x}) = \sigma\left(\mbf{w}^\top \mbf{x}\right) + \eta$, where $\sigma(\cdot)$ is a nonlinear link function, $\mbf{w}$ is drawn from a low-dimensional subspace, and $\eta$ is additive noise. We only consider linear functions $f(\mbf{x}) = \mbf{w}^\top \mbf{x} + \eta$ in our analysis, but also assume $\mbf{w}$ is sampled from a low-dimensional distribution. In our experiments, we sample nonlinear functions from subspaces of the \emph{function space}, which differs from sampling the function \emph{parameters} from a subspace of Euclidean space. Furthermore, our goal is to use such a parameterization to study OOD generalization, whereas the main focus of \cite{oko2025pretrained} is to examine whether ICL can solve such functions at all.

\section*{Acknowledgments}
 QQ, SK, AX, and CY acknowledge NSF CAREER CCF-2143904, NSF IIS 2312842, NSF IIS 2402950, and ONR N00014-22-1-2529. QQ also acknowledge Google Research Scholar Award.  LB, CY, and SK acknowledge NSF CAREER CCF-1845076 and NSF CCF-2331590. We would like to thank Emrullah Ildiz (University of Michigan), Samet Oymak (University of Michigan), and Daniel Hsu (Columbia University) for fruitful discussions.

% ------------------------------------------------------------
% Bibliography options
% ------------------------------------------------------------

\newpage

\printbibliography

\newpage

\appendix

\section{Additional Results}

In this section, we present additional theoretical and experimental results to supplement those in the main text. In \Cref{sec:diff_angles}, we generalize the result in \Cref{prop:neg_result} to the case of $r$ distinct principal angles. In \Cref{sec:feature_shifts}, we explore how our framework can be used to analyze the effect of feature shifts between training and testing prompts.

\subsection{Result with Different Principal Angles}
\label{sec:diff_angles}

In Proposition~\ref{prop:neg_result}, we assumed that all of the $r$ principal angles between the subspaces $\mbf{U}_s \in \mbb{R}^{d\times r}$ and $\mbf{U}_{t} \in \mbb{R}^{d\times r}$ were all the same, i.e., $\theta_i = \theta \in [0, \frac{\pi}{2}]$, for simplicity. In \Cref{prop:neg_result_diff_angles}, we relax this requirement and present a result where the angles are not necessarily the same.

\begin{tcolorbox}
\begin{proposition}[Task Distribution Shift with Different Angles]
\label{prop:neg_result_diff_angles}
Let $g_{\texttt{ATT}}^\star$ denote the optimal linear attention model corresponding to the independent data setting in \Cref{eqn:vanilla_setup} with where the task vector is drawn from a single subspace. For all $j \in [m+1]$, suppose that the prompts at test time are constructed with features $\mbf{x}_j \sim \mc{N}(\mbf{0}, \mbf{I}_d)$ and labels 
    \begin{align*}
        \widetilde{y}_j = \widetilde{\mbf{w}}^\top \mbf{x}_j + \eta_j, \quad  \text{where} \quad \widetilde{\mbf{w}} \sim \mc{N}(\mbf{0}, \mbf{\Sigma}_t) \quad \text{and} \quad \eta_j \sim \mc{N}(0, \sigma^2 ),
    \end{align*}
    with covariance matrix $\mbf{\Sigma}_t \in \mbb{R}^{d\times d}$ from Equation~(\ref{eqn:covariance_t}), but now with $\theta_1 \neq \theta_2 \neq \dots \neq \theta_r$.
    Then, we have
    \begin{equation}
        \lim\limits_{m \rightarrow \infty} \lim\limits_{n \rightarrow \infty} \lim\limits_{\epsilon \rightarrow 0}  \mbb{E}\left[ \left(\tilde{y}_{m+1} - g^\star_{\texttt{ATT}}\left( \widetilde{\mbf{Z}} \right)  \right)^2 \right]  = \sum_{i=1}^r \sin^2(\theta_i) + \sigma^2,
    \end{equation}
    where $\theta_i \in [0, \frac{\pi}{2}]$ is the $i$-th principal angle between the train subspace $\mbf{U}_s \in \mbb{R}^{d\times r}$ and the test subspace $\mbf{U}_t\in \mbb{R}^{d\times r}$.
    
\end{proposition}
\end{tcolorbox}
The proof is a direct extension of the proof for \Cref{prop:neg_result} and is left in \Cref{sec:proof_of_neg_result}. 
Recall that the test risk presented in \Cref{prop:neg_result} was $r\sin^2(\theta) + \sigma^2$. It is easy to see that if we set $\theta_i = \theta$, then the test risk in \Cref{prop:neg_result_diff_angles} recovers the risk in \Cref{prop:neg_result}, i.e., $\sum_{i=1}^r \sin^2(\theta_i) =  r\sin^2(\theta)$.

\subsection{Can Transformers Tolerate Feature Shifts?}
\label{sec:feature_shifts}

Thus far, we studied how changes in the task vector affect the test risk of ICL. In this section, we explore how distribution shifts in the features (in both the prompts and the query) affect the test risk. Recent work by ~\cite{zhang2024trained} demonstrated that Transformers are not robust to feature shifts, showing that ICL fails to remain robust when features are scaled by a constant. We share a similar observation that feature shifts can be much more detrimental to the test risk than task shifts, by considering the mathematical framework used thus far.
Consider training a Transformer using prompts whose labels are constructed by a task vector $\mbf{w} \sim \mc{N}(\mbf{0}, \mbf{I}_d)$:
\begin{align}
\label{eqn:feature_shift}
    y_i = \mbf{w}^\top \mbf{x}_i + \eta_i, \quad \text{where} \quad \mbf{x}_i \sim \mc{N}(\mbf{0}, \mbf{\Sigma}_s) \quad \text{and} \quad \eta_i \sim \mc{N}(0, \sigma^2 ).
\end{align}
During testing, we construct labels using the task vectors drawn from the same distribution but with shifted features: $\widetilde{y}_{j} = \mbf{w}^\top \widetilde{\mbf{x}}_j + \eta_j$, where $\widetilde{\mbf{x}}_j \sim \mc{N}(\mbf{0}, \mbf{\Sigma}_t)$. 
Before presenting our theoretical result, we illustrate in Figure~\ref{fig:feature_shift} how shifts in the features affect the test risk, in contrast to shifts in the task vector.
Interestingly, for both linear and nonlinear Transformer models, when the feature shift is small (i.e., $\theta > \frac{\pi}{4}$), the test risk changes only minimally, whereas it increases sharply for 
$\theta < \frac{\pi}{4}$. The following result captures this behavior, which arises as an artifact of using a degenerate distribution for the features.

\begin{figure}[t!]
    \centering
     \begin{subfigure}[t!]{0.49\textwidth}
         \centering
        \includegraphics[width=0.9\textwidth]{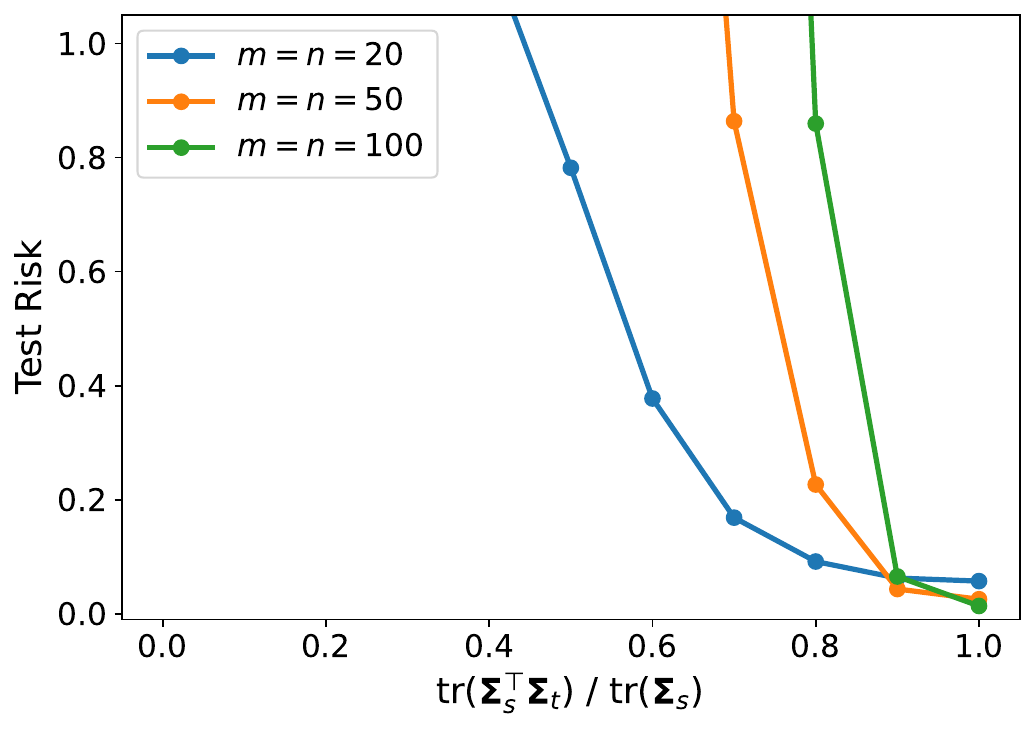}
     \end{subfigure}
     \hfill
     \begin{subfigure}[t!]{0.49\textwidth}
         \centering
         \includegraphics[width=0.9\textwidth]{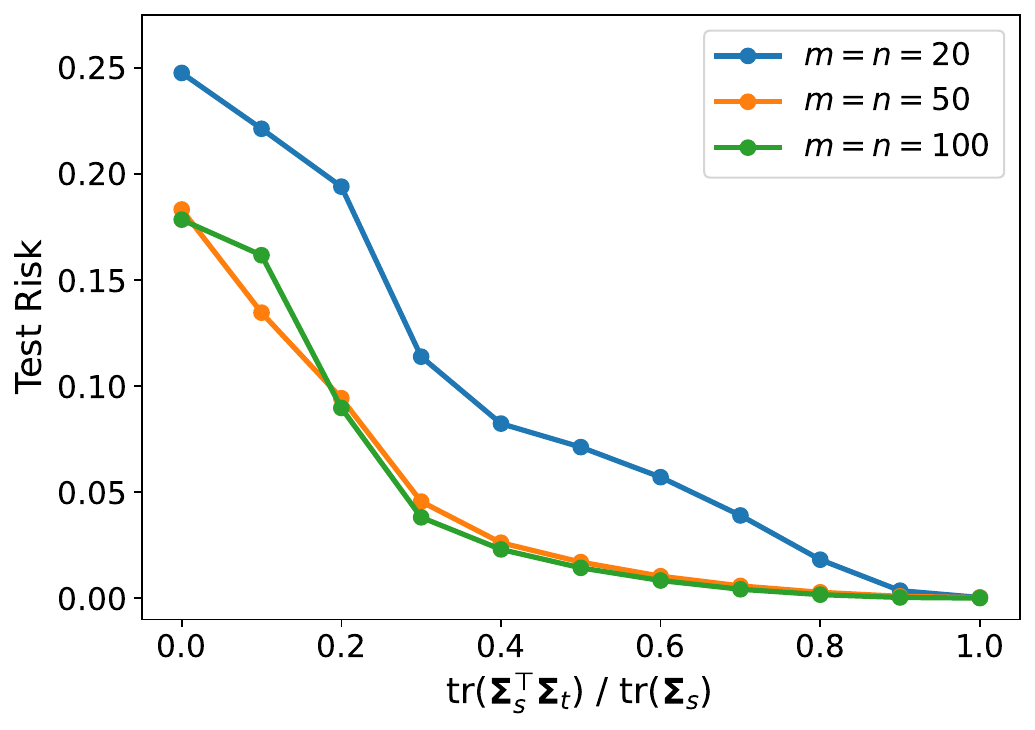}
     \end{subfigure} 
    \caption{Plot of the test risk as the covariance matrix of the features shift as a function of $t$ on a linear (left) and a nonlinear Transformer (right). When $t > 0.5$, the test risk does not increase as much as it does for $t < 0.5$, which contrasts with the behavior observed under covariance shifts in the task vector.}
    \label{fig:feature_shift}
\end{figure}

\begin{tcolorbox} 
\begin{proposition}[Feature Distribution Shift] \label{prop:feature_neg_result}
Let $g_{\texttt{ATT}}^\star$ denote the optimal linear attention model corresponding to the independent data setting in Equation~(\ref{eqn:feature_shift}). For all $j \in [m+1]$, suppose that the prompts at test time are constructed with task vectors $\mbf{w} \sim \mc{N}(\mbf{0}, \mbf{I}_d)$ and labels 
    \begin{align*}
        \widetilde{y}_j = \mbf{w}^\top \widetilde{\mbf{x}}_j + \eta_j, \quad  \text{where} \quad \widetilde{\mbf{x}}_j \sim \mc{N}(\mbf{0}, \mbf{\Sigma}_t), \quad \eta_j \sim \mc{N}(0, \sigma^2 ),
    \end{align*}
    and $\mbf{\Sigma}_t \in \mbb{R}^{d\times d}$ is from Equation~(\ref{eqn:covariance_t}).
    Then, we have
    \begin{align*}
        \lim\limits_{m \rightarrow \infty} \lim\limits_{n \rightarrow \infty}  \mbb{E}\left[ \left(\tilde{y}_{m+1} - g^\star_{\texttt{ATT}}\left( \widetilde{\mbf{Z}} \right)  \right)^2 \right] = \, \,
        &\left(\frac{c_1 - (1+\epsilon)c_2}{(1+\epsilon)^2}  \right)\cdot r\cos^2(\theta)  \\
        &+ \left(\frac{c_1 - \epsilon c_2}{\epsilon^2}\right)\cdot r\sin^2(\theta) + r +\sigma^2 + \mc{O}(\epsilon),
    \end{align*}
    where $c_1 = \left((1+2\epsilon)(1+\epsilon) + \epsilon^2 \right)$ and $c_2 = 2(1+\epsilon)$.
\end{proposition}
\end{tcolorbox}
The proof is available in \Cref{sec:feature_neg_result_proof}. 
Notice that we cannot immediately take $\epsilon \to 0$ in this setting, as that would result in an underdetermined system. The problematic term is the $\sin(\cdot)$, which reveals a few interesting insights: when $\theta = 0$, the $\sin(\cdot)$ term vanishes, allowing us to take $\epsilon \to 0$ and recover the same test risk as in Proposition~\ref{prop:neg_result}. However, as $\theta$ moves away from zero, the risk contribution from the $\sin(\cdot)$ term starts to dominate, scaling with a factor of $\mc{O}(1/\epsilon^2)$ and causing the test risk to diverge since $\epsilon$ is assumed to be small.
This result shows that testing with prompts whose features are shifted toward an almost independent covariance matrix (since $\epsilon$ is not exactly zero) can degrade test risk  more than shifting the distribution of the task vector, roughly echoing the results of~\cite{zhang2024trained}.

\subsection{Can LoRA Be Used to Adapt to Subspace Shifts?}

The overall paper discusses how the training task vectors affect the OOD generalization capabilities of ICL. In this section, we shift focus to instead ask a different question: can we use parameter-efficient fine-tuning to adapt a trained transformer to exhibit OOD generalization abilities? 
Specifically, consider the setting in \Cref{sec:setup}, where we train a model on data from \Cref{eqn:single_subspace_setup} to obtain $g_{\texttt{ATT}}^\star$, and test on data from \Cref{eq:vanilla_setup_test}. Since we deal with low-dimensional subspaces where $r \ll d$, we are interested in using low-rank adaptation (LoRA) \cite{hu2022lora} to fine-tune the model to mitigate the dependence on $\theta$ in the test risk. %Since we deal with low-dimensional subspaces where $r \ll d$, we consider if LoRA \qq{if we do not mention parameter efficiency at the beginning of this subsection, I feel the motivation is not strong enough for LoRA} can exhibit such a behavior. 
In the following, we prove there exists a pair of low-rank adapters that can equivalently generalize to the span of subspaces, showing that LoRA can enable OOD generalization.

\begin{tcolorbox}%[breakable]
\begin{corollary}
\label{coro:lora}
Let $\mbf{W}^\star_Q, \mbf{W}^\star_K, \mbf{W}^\star_V, \mbf{p}^\star$ denote the optimal linear attention weights via \Cref{eqn:expected_lin_att_objective} corresponding to the independent data setting in Equation~(\ref{eqn:single_subspace_setup}). For all $j \in [m+1]$, suppose that the prompts at test time are constructed with features $\mbf{x}_j \sim \mc{N}(\mbf{0}, \mbf{I}_d)$ and labels 
    \begin{align*}
        \widetilde{y}_{j} = \widetilde{\mbf{w}}^\top \mbf{x}_j + \eta_j, \quad  \text{where} \quad \widetilde{\mbf{w}} \sim \mc{N}(\mbf{0}, \mbf{\Sigma}_t), \quad \eta_j \sim \mc{N}(0, \sigma^2 ),
    \end{align*}
    and $\mbf{\Sigma}_t \in \mbb{R}^{d\times d}$ is from Equation~(\ref{eqn:covariance_t}). Define the adapted model $h_\texttt{ATT}^\star\left( \widetilde{\mbf{Z}} \right) $ as
    \begin{align*}
        h_\texttt{ATT}^\star\left( \widetilde{\mbf{Z}} \right)  \coloneqq \frac{1}{m}\left(\tilde{\mbf{z}}_{q}^\top \left(\mbf{W}_Q^\star \mbf{W}_K^{\star\top} + \mbf{B}_2\mbf{B}_1^\top \right) \widetilde{\mbf{Z}}_\mc{M}^\top \right) \widetilde{\mbf{Z}}_\mc{M} \mbf{W}_V^\star \mbf{p}^\star.
    \end{align*}
    There exist $\mbf{B}_1, \mbf{B}_2 \in \mbb{R}^{(d+1) \times r}$ such that for any  $\theta \in \left[0, \frac{\pi}{2}\right]$ and $\delta \in (0, r)$, if 
    \begin{equation}
        m \geq n > \frac{2(r + \sigma^2 + 1)(r - \delta) - r}{\delta} + (r + \sigma^2 + 1) \sqrt{\frac{r - \delta}{\delta}},
    \end{equation}
    then $\lim\limits_{\epsilon \to 0} \mbb{E}\left[ \left(\tilde{y}_{m+1} - h^\star_{\texttt{ATT}}\left( \widetilde{\mbf{Z}} \right)  \right)^2 \right] < \sigma^2 + \delta$.
    %has a test risk $\lim\limits_{m \rightarrow \infty} \lim\limits_{n \rightarrow \infty} \lim\limits_{\epsilon \rightarrow 0}  \mbb{E}\left[ \left(\tilde{y}_{m+1} - h^\star_{\texttt{ATT}}(\tilde{\mbf{z}}_q, \widetilde{\mbf{Z}}_{\mc{M}}) \right)^2 \right] = \sigma^2$.
\end{corollary}
\end{tcolorbox}
The analysis is in Appendix~\ref{sec:proof_of_coro_lora} and involves finding a pair of low-rank adapters such that the adapted model can capture a rank-$2r$ covariance matrix. Here, we emphasize the adapters are rank-$r$, even though the adapted model generalizes to all tasks in a $2r$-dimensional subspace. %are of rank-$r$ rather than rank-$2r$. 
Roughly speaking, since the pre-trained weights already span $\mbf{U}_s$, it suffices for the adapters to span the $r$-dimensional space $\mbf{U}_{s, \perp}$. %as long as they span  .
This provides an interesting insight into LoRA: the rank of the adapters should be at least the difference between the dimensions of the target and pre-trained weights to enable generalization. Here, we focused solely on analyzing the existence of optimal solutions for LoRA. A deeper understanding of the learning dynamics requires additional techniques -- such as those developed in \cite{yaras2024compressible,pmlr-v238-min-kwon24a,balzano2025overview,ghosh2025learning} -- which we leave for future work. 

Next, we present experiments to support Corollary~\ref{coro:lora}. On the left-hand side of Figure~\ref{fig:lora_theory}, we show the test risk as a function of prompt length, demonstrating that our choice of adapters indeed drives the test risk to zero as the prompt length becomes sufficiently large. Note that no training is performed in this setup; we directly plug in the optimal linear attention weights adapted by the LoRA parameters and compute the risk for varying prompt lengths. This naturally raises the question of whether the adapters derived in Corollary~\ref{coro:lora} are actually learned in practice.

To this end, we train rank-$r$ LoRA adapters, where the pre-trained model is initialized with the optimal linear attention weights corresponding to a task vector drawn from a single subspace  $\mbf{U}_s \in \mathbb{R}^{d \times r}$. %(i.e., as in Lemma~\ref{lem:vanilla_opt_weights}). 
We then fine-tune the model using prompts whose task vectors lie in $\mbf{U}_{2r} \coloneqq \begin{bmatrix} \mbf{U}_s & \mbf{U}_{s,\perp} \end{bmatrix}$. To evaluate whether the learned LoRA adapters align with the analytical ones derived in Corollary~\ref{coro:lora}, we compute the basis error between the two resulting orthonormal bases defined as follows:
\begin{align}
\label{eqn:subspace_err}
    \mathrm{Subspace \, Error }(\mbf{U}, \widehat{\mbf{U}}) = \frac{\|\mbf{U}\mbf{U}^\top - \widehat{\mbf{U}}\widehat{\mbf{U}}^\top\|^2_{\mathsf{F}}}{\|\mbf{U}\mbf{U}^\top\|^2_{\mathsf{F}}}.
\end{align}
For each iteration, we extract the top-$r$ left singular vectors of $\widehat{\mbf{B}}_1$ and $\widehat{\mbf{B}}_2$ and compute the subspace error relative to their analytical counterparts, starting from an orthonormal initialization scaled by a small constant $0.01$. On the right-hand side of Figure~\ref{fig:lora_theory}, we show that the subspace error converges approximately to zero, demonstrating that such adapters are indeed learned in practice.

\begin{figure}[t!]
    \centering
     \begin{subfigure}[t!]{0.49\textwidth}
         \centering
        \includegraphics[width=0.9\textwidth]{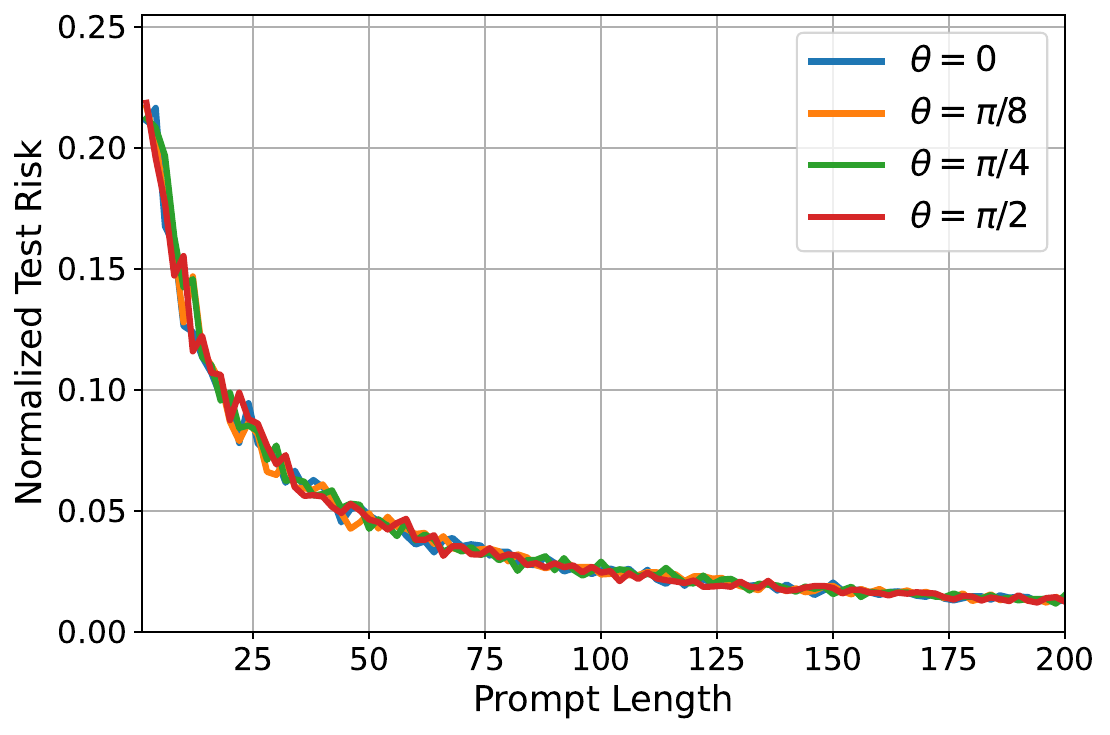}
     \end{subfigure}
     \hfill
     \begin{subfigure}[t!]{0.49\textwidth}
         \centering
         \includegraphics[width=0.9\textwidth]{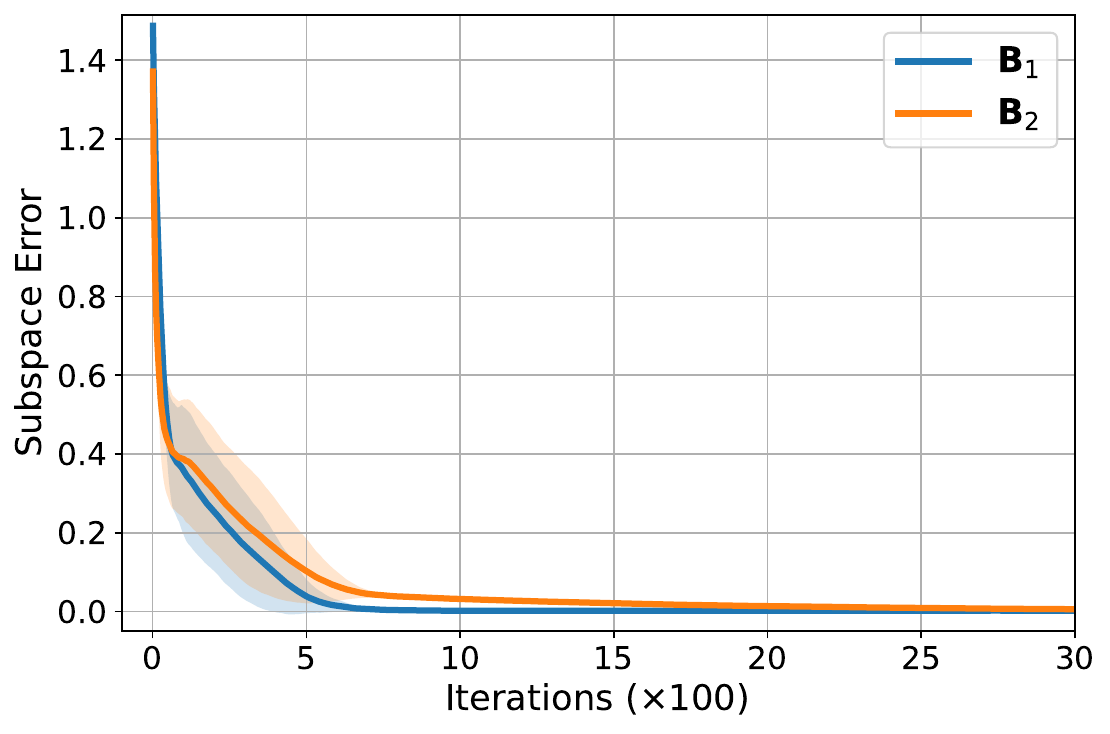}
     \end{subfigure} 
    \caption{Experiments demonstrating that LoRA can be used to adapt to distribution shifts. Left: Test risk using the optimal LoRA adapters, showing that the risk approaches zero as the prompt length increases. Right: Subspace error between the learned and analytical LoRA adapters, demonstrating that the optimal adapters can be recovered using gradient descent over $5$ random runs. The subspace error is defined in Equation~(\ref{eqn:subspace_err}).}
    \label{fig:lora_theory}
\end{figure}

\subsection{Experimental Results with Noisy Data}

In the main text, we presented results for the noiseless setting ($\sigma^2 = 0$). Recall that our theory in \Cref{thm:mix_two_subspaces} states that the test risk should converge to the variance of the noise given a sufficiently large prompt length. In \Cref{fig:linformer_noisy}, we present results for the single-layer linear attention model, showing that when the training and testing prompt lengths are set to $m=n=500$, the test risk converges roughly to the variance of the noise in two settings, $\sigma^2 = 0.01$ and $0.25$, corroborating our theory. Here, we set $d=20$ and $r=6$.

\begin{figure}[t!]
    \centering
     \begin{subfigure}[t!]{0.49\textwidth}
         \centering
        \includegraphics[width=\textwidth]{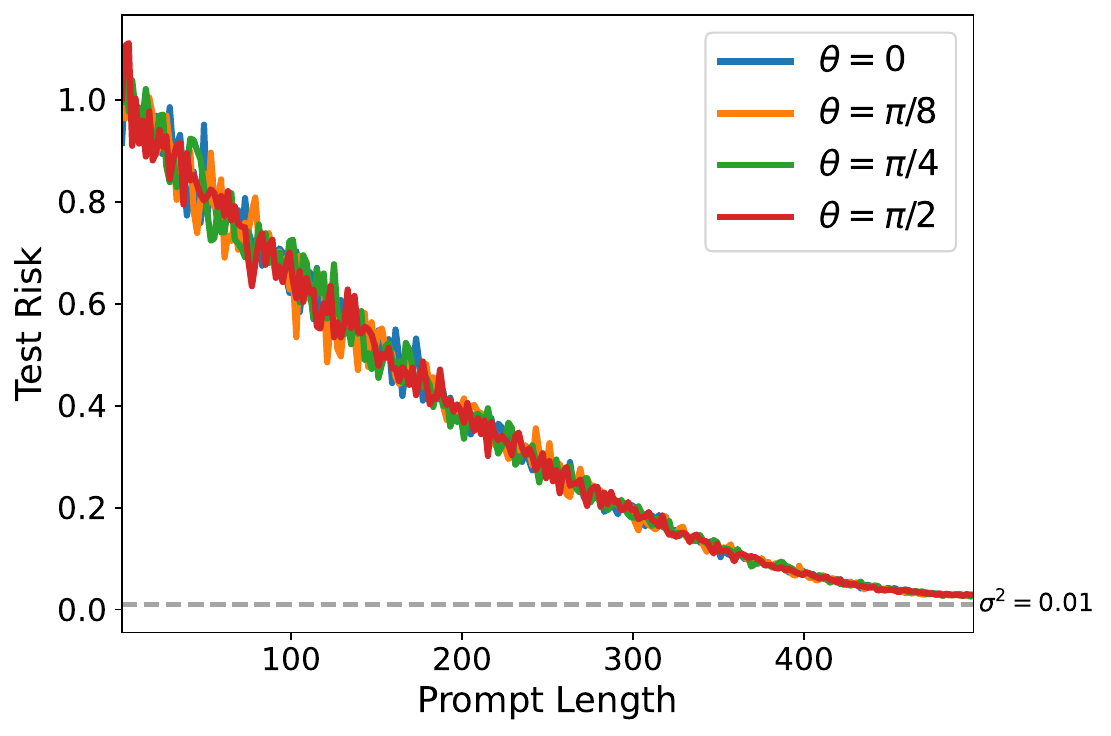}
     \end{subfigure}
     \hfill
     \begin{subfigure}[t!]{0.49\textwidth}
         \centering
         \includegraphics[width=\textwidth]{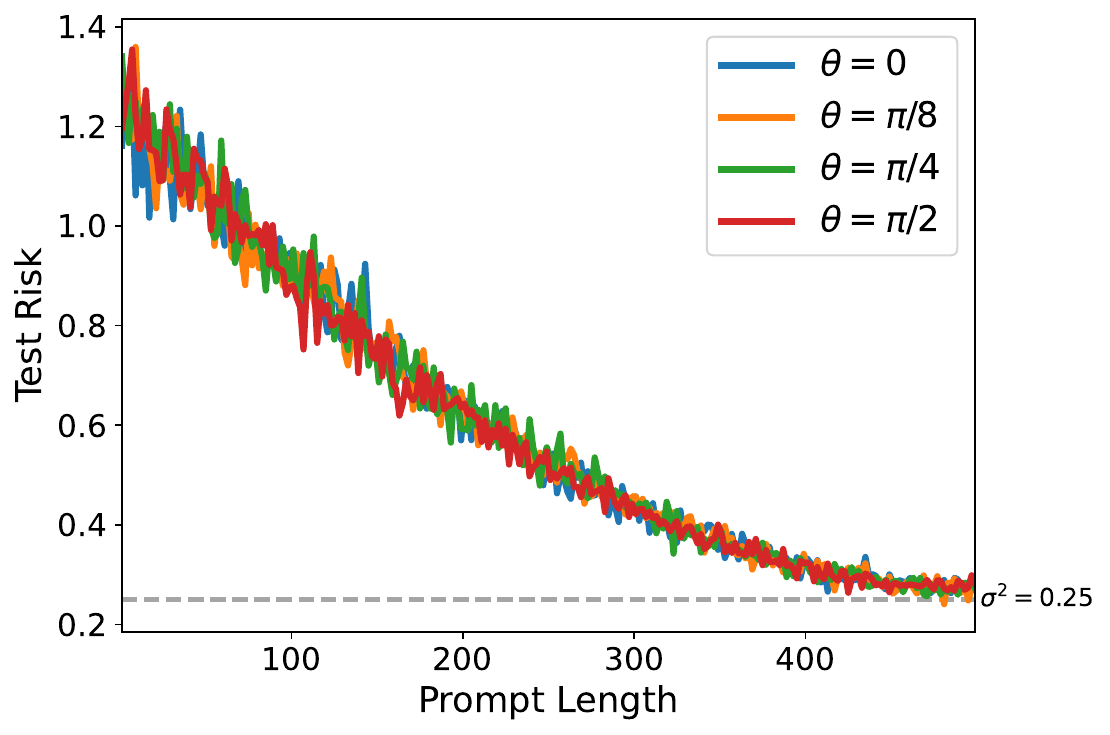}
     \end{subfigure} 
    \caption{Experimental results for the single-layer linear attention model in the noisy setting, where the model is trained on a mixture of two low-dimensional Gaussians. Left: Test risk as a function of the prompt length with $\sigma^2 = 0.01$. Right: Test risk as a function of the prompt length with $\sigma^2 = 0.25$. Both cases converge to the variance of the noise, corroborating our theory in \Cref{thm:mix_two_subspaces}.}
    \label{fig:linformer_noisy}
\end{figure}

%\clearpage

\section{Deferred Proofs}

This section presents all deferred proofs and is organized as follows: \Cref{sec:task_proofs} contains all proofs related to shifts in the task, \Cref{sec:feature_shifts_proofs} includes all proofs related to shifts in the features, and \Cref{sec:aux_proofs} provides auxiliary results used to support both the task and feature shift proofs.

\subsection{Proofs for Task Shifts}
\label{sec:task_proofs}
Here, we provide proofs of the results related to subspace shifts in the task vector.

\subsubsection{Supporting Results}
We first derive an expression for the test risk under a general distribution shift for the task vector. 
\begin{lemma}[Test Risk under General Task Distribution Shift] \label{lem:task_test_risk}
    Let $g_{\texttt{ATT}}^\star$ denote the optimal linear attention model corresponding to the independent data setting in \Cref{eqn:vanilla_setup}, where $\mbf{w}$ follows a (mixture of) Gaussian distribution(s) with zero mean and covariance $\mbf{\Sigma}_s$. For all $j \in [m+1]$, suppose that the prompts at test time are constructed with features $\mbf{x}_i \sim \mc{N}(\mbf{0}, \mbf{I}_d)$ and labels 
    \begin{align*}
        \widetilde{y}_{j} = \widetilde{\mbf{w}}^\top \mbf{x}_j + \eta_i, \quad  \text{where} \quad \widetilde{\mbf{w}} \sim \mc{N}(\mbf{0}, \mbf{\Sigma}_t) \quad \text{and} \quad \eta_j \sim \mc{N}(0, \sigma^2 ).
    \end{align*}
   Then,
    \begin{align*}
        &\mbb{E}\left[ \left(\widetilde{y}_{m+1} - g_{\texttt{ATT}}^\star\left(\widetilde{\mbf{Z}}\right)\right)^2 \right] = M_t - \operatorname{Tr}\left(\mbf{\Sigma}_t\mbf{A}\right)+ \frac{M_t}{m} \operatorname{Tr}\left( \mbf{A}^\top \mbf{A} \right) -  \operatorname{Tr}\left(\mbf{\Sigma}_t \mbf{A} \right) + \frac{m+1}{m} \operatorname{Tr}\left(\mbf{A} \mbf{\Sigma}_t \mbf{A}\right),
    \end{align*}
    where $M_t = \operatorname{Tr}\left(\mbf{\Sigma}_t\right) + \sigma^2$.
\end{lemma}

\begin{proof}
    Recall at inference time,
    \begin{align}
        \widetilde{\mbf{Z}}_{\mc{M}} = \begin{bmatrix}
            \tilde{\mbf{z}}_1 & \ldots & \tilde{\mbf{z}}_m & \mbf{0}
        \end{bmatrix}^\top  =  \begin{bmatrix}
            \mbf{x}_1 & \ldots & \mbf{x}_m & \mbf{0} \\
            \widetilde{y}_1 & \ldots & \widetilde{y}_m & 0
    \end{bmatrix}^\top
    \quad \text{and} \quad \tilde{\mbf{z}}_{q} = \begin{bmatrix}
        \mbf{x}_{m+1} \\ 0
    \end{bmatrix}\coloneqq \begin{bmatrix}
        \mbf{x}_{q} \\ 0
    \end{bmatrix}.
\end{align}
Then, let us define 
\begin{align*}
    \mbf{X}_{te} \coloneqq \begin{bmatrix}
        \mbf{x}_1 & \mbf{x}_2& \ldots & \mbf{x}_m
    \end{bmatrix}^\top,
    \quad \mbf{y}_{te} := \begin{bmatrix}
        \widetilde{y}_1 & \widetilde{y}_2 & \ldots & \widetilde{y}_m
    \end{bmatrix}^\top, \quad \boldsymbol{\eta}_{te} := \begin{bmatrix}
        \eta_1 & \eta_2 & \ldots \eta_m
    \end{bmatrix}^\top,
\end{align*}
and $\eta_q := \eta_{m+1}$. Note $\mbf{y}_{te} = \mbf{X}_{te} \widetilde{\mbf{w}} + \boldsymbol{\eta}_{te}$. By Lemma~\ref{lem:vanilla_opt_weights}, we have
\begin{align*}
g^\star_{\texttt{ATT}}\left(\widetilde{\mbf{Z}} \right) = \frac{1}{m} \mbf{x}_q^\top \mbf{A} \mbf{X}_{te}^\top \mbf{y}_{te} = \mbf{x}_q^\top \underbrace{\left( \frac{1}{m} \mbf{A} \mbf{X}_{te}^\top \mbf{y}_{te} \right)}_{\coloneqq \widehat{\mbf{w}}},
\end{align*}
where $\mbf{A} = \left( \frac{n+1}{n} \mbf{I}_d + \frac{M_s}{n} \mbf{\Sigma}_s \right)^{-1}$.
By plugging this into the risk and linearity of expectation, 
    \begin{equation} \label{eq:risk-expanded}
        \mbb{E} \left[ \left(\widetilde{\mbf{w}}^\top \mbf{x}_q + \eta_q - \widehat{\mbf{w}}^\top \mbf{x}_q \right)^2 \right] = \underbrace{\mbb{E} \left[ \left(\widetilde{\mbf{w}}^\top \mbf{x}_q + \eta_q \right)^2 \right]}_{(a)} - 2 \underbrace{\mbb{E} \left[ \left( \widetilde{\mbf{w}}^\top \mbf{x}_q + \eta_q \right) \left( \mbf{x}_q^\top \widehat{\mbf{w}} \right) \right]}_{(b)} + \underbrace{\mbb{E} \left[ \left( \widehat{\mbf{w}}^\top \mbf{x}_q \right)^2 \right]}_{(c)}.
    \end{equation}
    It suffices to analyze each individual term. 
    
    \paragraph{Analyzing $(a)$.} We first evaluate $\mbb{E}\left[ \left(\widetilde{\mbf{w}}^\top \mbf{x}_q + \eta_q \right)^2 \right]$. First, we note 
    \begin{equation*}
        \mbb{E} \left[ \left( \widetilde{\mbf{w}}^\top \mbf{x}_q + \eta_q \right)^2 \right] = \mbb{E} \left[ \left(\widetilde{\mbf{w}}^\top \mbf{x}_q \right)^2 \right] + 2 \mbb{E}\left[ \eta_q \widetilde{\mbf{w}}^\top \mbf{x}_q \right] + \mbb{E}\left[ \eta_q^2 \right] = \mbb{E} \left[ \left( \widetilde{\mbf{w}}^\top \mbf{x}_q \right)^2 \right] + \sigma^2,
    \end{equation*}
    so it suffices to analyze $\mbb{E} \left[ \left( \widetilde{\mbf{w}}^\top \mbf{x}_q \right)^2 \right]$. By law of total expectation and the fact that $\widetilde{\mbf{w}}, \mbf{x}_q$ are independent,
    \begin{equation*}
        \mbb{E} \left[ \left( \widetilde{\mbf{w}}^\top \mbf{x}_q \right)^2 \right] = \mbb{E}_{\widetilde{\mbf{w}}} \left[ \mbb{E}_{\mbf{x}_q} \big[ \left( \widetilde{\mbf{w}}^\top \mbf{x}_q \right)^2 \; \big| \: \widetilde{\mbf{w}} \big] \right]. 
    \end{equation*}
    Conditioned on $\widetilde{\mbf{w}}$, $\widetilde{\mbf{w}}^\top \mbf{x}_q \sim \mathcal{N}(0, \|\widetilde{\mbf{w}}\|^2)$, so $\mbb{E} \left[ \left( \widetilde{\mbf{w}}^\top \mbf{x}_q \right)^2 \: \big| \: \widetilde{\mbf{w}} \right] = \mathrm{Var}\left( \widetilde{\mbf{w}}^\top \mbf{x}_q \: | \: \widetilde{\mbf{w}} \right) = \|\widetilde{\mbf{w}}\|^2$. Therefore, 
    \begin{align*}
        &\mbb{E}_{\widetilde{\mbf{w}}} \left[ \mbb{E}_{\mbf{x}} \big[ \left( \widetilde{\mbf{w}}^\top \mbf{x}_q \right)^2 \; \big| \: \widetilde{\mbf{w}} \big] \right] = \mbb{E} \left[ \|\widetilde{\mbf{w}}\|^2 \right] = \operatorname{Tr}\left( \mbb{E} \left[ \widetilde{\mbf{w}} \widetilde{\mbf{w}}^\top \right] \right) = \operatorname{Tr}\left(\mbf{\Sigma}_t \right).
    \end{align*}
    Therefore,
    \begin{equation*}
        \mbb{E}\left[ \left(\widetilde{\mbf{w}}^\top \mbf{x}_q + \eta_q \right)^2 \right] = \operatorname{Tr}(\mbf{\Sigma}_t) + \sigma^2.
    \end{equation*}

    \bigskip
    
    \paragraph{Analyzing $(b)$.} Next, we analyze $\mbb{E} \left[ \left( \widetilde{\mbf{w}}^\top \mbf{x}_q + \eta_q \right) \left( \mbf{x}_q^\top \widehat{\mbf{w}} \right) \right]$. We first note
    \begin{equation*}
        \mbb{E} \left[ \left( \widetilde{\mbf{w}}^\top \mbf{x}_q + \eta_q \right) \left( \mbf{x}_q^\top \widehat{\mbf{w}} \right) \right] = \mbb{E} \left[ \left( \widetilde{\mbf{w}}^\top \mbf{x}_q \right) \left( \mbf{x}_q^\top \widehat{\mbf{w}} \right) \right] + \underbrace{\mbb{E}\left[ \eta_q \mbf{x}_q^\top \widehat{\mbf{w}} \right]}_{= 0} = \mbb{E} \left[ \left( \widetilde{\mbf{w}}^\top \mbf{x}_q \right) \left( \mbf{x}_q^\top \widehat{\mbf{w}} \right) \right], 
    \end{equation*}
    so it suffices to analyze $\mbb{E} \left[ \left( \widetilde{\mbf{w}}^\top \mbf{x}_q \right) \left( \mbf{x}_q^\top \widehat{\mbf{w}} \right) \right]$. Substituting $\widehat{\mbf{w}} := \frac{1}{m} \mbf{A} \mbf{X}_{te}^\top \mbf{y}_{te} = \frac{1}{m} \mbf{A} \mbf{X}_{te}^\top (\mbf{X}_{te}\widetilde{\mbf{w}} + \boldsymbol{\eta}_{te})$ yields
    \begin{align*}
        &\mbb{E} \left[ \left( \widetilde{\mbf{w}}^\top \mbf{x}_q \right) \left( \mbf{x}_q^\top \widehat{\mbf{w}} \right) \right] = \frac{1}{m} \mbb{E} \left[ \widetilde{\mbf{w}}^\top \mbf{x}_q \mbf{x}_q^\top \mbf{A} \mbf{X}_{te}^\top (\mbf{X}_{te} \widetilde{\mbf{w}} + \boldsymbol{\eta}_{te}) \right] \\
        &= \frac{1}{m} \Big( \mbb{E}\left[ \widetilde{\mbf{w}}^\top \mbf{x}_q \mbf{x}_q^\top \mbf{A} \mbf{X}_{te}^\top \mbf{X}_{te} \widetilde{\mbf{w}} \right] + \mbb{E} \left[ \widetilde{\mbf{w}}^\top \mbf{x}_q \mbf{x}_q^\top \mbf{A} \mbf{X}_{te}^\top \boldsymbol{\eta}_{te} \right] \Big) \\
        &= \frac{1}{m} \Big( \mbb{E}\left[ \widetilde{\mbf{w}}^\top \mbf{x}_q \mbf{x}_q^\top \mbf{A} \mbf{X}_{te}^\top \mbf{X}_{te} \widetilde{\mbf{w}} \right] + \underbrace{\mbb{E} \left[ \widetilde{\mbf{w}}^\top \mbf{x}_q \mbf{x}_q^\top \mbf{A} \mbf{X}_{te}^\top\right] \mbb{E}\left[ \boldsymbol{\eta}_{te} \right]}_{= \mbf{0}} \Big) \\
        &= \frac{1}{m} \mbb{E}\left[ \widetilde{\mbf{w}}^\top \mbf{x}_q \mbf{x}_q^\top \mbf{A} \mbf{X}_{te}^\top \mbf{X}_{te} \widetilde{\mbf{w}} \right] = \frac{1}{m} \mbb{E} \Big[ \operatorname{Tr}\left( \widetilde{\mbf{w}} \widetilde{\mbf{w}}^\top \mbf{x}_q \mbf{x}_q^\top \mbf{A} \mbf{X}_{te}^\top \mbf{X}_{te}  \right) \Big] \\
        &= \frac{1}{m} \operatorname{Tr} \Big( \mbb{E} \left[ \widetilde{\mbf{w}} \widetilde{\mbf{w}}^\top \mbf{x}_q \mbf{x}_q^\top \mbf{A} \mbf{X}_{te}^\top \mbf{X}_{te}  \right] \Big) \\
        &= \frac{1}{m} \operatorname{Tr} \Big( \underbrace{\mbb{E} \left[ \widetilde{\mbf{w}} \widetilde{\mbf{w}}^\top \right]}_{ \mbf{\Sigma}_t} \underbrace{\mbb{E} \left[ \mbf{x}_q \mbf{x}_q^\top \right]}_{\mbf{I}_d} \mbf{A} \underbrace{\mbb{E} \left[ \mbf{X}_{te}^\top \mbf{X}_{te} \right]}_{m \cdot \mbf{I}_d} \Big) = \operatorname{Tr} \Big( \mbf{\Sigma}_t \mbf{A} \Big).
    \end{align*} 
    
    \bigskip
    
    \paragraph{Analyzing $(c)$.} Finally, we analyze $\mbb{E}\left[ (\mbf{x}_q^\top \widehat{\mbf{w}} )^2 \right]$:
    \begin{align*}
        &\mbb{E}\left[ (\mbf{x}_q^\top \widehat{\mbf{w}})^2 \right] = \frac{1}{m^2} \mbb{E}\left[ \left( \mbf{x}_q^\top \mbf{A}\mbf{X}_{te}^\top (\mbf{X}_{te} \widetilde{\mbf{w}} + \boldsymbol{\eta}_{te}) \right)^2 \right] = \frac{1}{m^2} \mbb{E} \left[ \big( \mbf{x}_q^\top \mbf{A} \mbf{X}_{te}^\top \mbf{X}_{te} \widetilde{\mbf{w}} + \mbf{x}_q^\top \mbf{A} \mbf{X}_{te}^\top \boldsymbol{\eta}_{te} \big)^2 \right] \\
        &= \frac{1}{m^2} \Big( \mbb{E} \left[ (\mbf{x}_q^\top \mbf{A} \mbf{X}_{te}^\top \mbf{X}_{te} \widetilde{\mbf{w}})^2 \right] + 2 \underbrace{\mbb{E}\left[ (\mbf{x}_q^\top \mbf{A} \mbf{X}_{te}^\top \mbf{X}_{te} \widetilde{\mbf{w}})(\mbf{x}_q^\top \mbf{A} \mbf{X}_{te}^\top \boldsymbol{\eta}_{te}) \right]}_{=0} + \mbb{E}\left[ (\mbf{x}_q^\top \mbf{A} \mbf{X}_{te}^\top \boldsymbol{\eta}_{te})^2 \right] \Big) \\
        &= \frac{1}{m^2} \Big( \underbrace{\mbb{E}\left[ \mbf{x}_q^\top \mbf{A} \mbf{X}_{te}^\top \mbf{X}_{te} \widetilde{\mbf{w}} \widetilde{\mbf{w}}^\top \mbf{X}_{te}^\top \mbf{X}_{te} \mbf{A}^\top \mbf{x}_q \right]}_{(d)} + \underbrace{\mbb{E} \left[ \mbf{x}_q^\top \mbf{A} \mbf{X}_{te}^\top \boldsymbol{\eta}_{te} \boldsymbol{\eta}_{te}^\top \mbf{X}_{te} \mbf{A}^\top \mbf{x}_q \right]}_{(e)} \Big).
    \end{align*}
    We first focus on $(d)$:
    \begin{align*}
        &\mbb{E}\left[ \mbf{x}_q^\top \mbf{A} \mbf{X}_{te}^\top \mbf{X}_{te} \widetilde{\mbf{w}} \widetilde{\mbf{w}}^\top \mbf{X}_{te}^\top \mbf{X}_{te} \mbf{A}^\top \mbf{x}_q \right] = \mbb{E} \Big[ \operatorname{Tr}\left( \mbf{x}_q \mbf{x}_q^\top \mbf{A} \mbf{X}_{te}^\top \mbf{X}_{te} \widetilde{\mbf{w}} \widetilde{\mbf{w}}^\top \mbf{X}_{te}^\top \mbf{X}_{te} \mbf{A}^\top  \right) \Big] \\
        &= \operatorname{Tr}\Big( \underbrace{\mbb{E}\left[ \mbf{x}_q \mbf{x}_q^\top\right] }_{\mbf{I}_d} \mbf{A} \mbb{E}\left[ \mbf{X}_{te}^\top \mbf{X}_{te} \widetilde{\mbf{w}} \widetilde{\mbf{w}}^\top \mbf{X}_{te}^\top \mbf{X}_{te} \right] \mbf{A}^\top \Big) \\
        &= \mbb{E}\Big[ \operatorname{Tr}\left( \mbf{A} \mbf{X}_{te}^\top \mbf{X}_{te} \widetilde{\mbf{w}} \widetilde{\mbf{w}}^\top \mbf{X}_{te}^\top \mbf{X}_{te} \mbf{A}^\top \right) \Big] = \mbb{E}\Big[ \operatorname{Tr}\left( \widetilde{\mbf{w}}\widetilde{\mbf{w}}^\top \mbf{X}_{te}^\top \mbf{X}_{te} \mbf{A}^\top \mbf{A} \mbf{X}_{te}^\top \mbf{X}_{te} \right) \Big] \\
        &= \operatorname{Tr}\Big( \underbrace{\mbb{E}\left[ \widetilde{\mbf{w}} \widetilde{\mbf{w}}^\top \right]}_{\mbf{\Sigma}_t} \mbb{E}\left[  \mbf{X}_{te}^\top \mbf{X}_{te} \mbf{A}^\top \mbf{A} \mbf{X}_{te}^\top \mbf{X}_{te} \right] \Big) = \mbb{E}\Big[ \operatorname{Tr}\left( \mbf{\Sigma}_t \mbf{X}_{te}^\top \mbf{X}_{te} \mbf{A}^\top \mbf{A} \mbf{X}_{te}^\top \mbf{X}_{te} \right) \Big] \\
        &= \mbb{E}\Big[ \operatorname{Tr}\left( \mbf{A} \mbf{X}_{te}^\top \mbf{X}_{te} \mbf{\Sigma}_t^{1/2} \mbf{\Sigma}_t^{1/2}  \mbf{X}_{te}^\top \mbf{X}_{te} \mbf{A}^\top \right) \Big] := \mbb{E}\Big[ \operatorname{Tr}\left( \widetilde{\mbf{X}}_{te}^\top \overline{\mbf{X}}_{te} \overline{\mbf{X}}_{te}^\top \widetilde{\mbf{X}}_{te} \right) \Big],
    \end{align*}
    where $\widetilde{\mbf{X}}_{te}^\top := \mbf{A} \mbf{X}_{te}^\top$ and $\overline{\mbf{X}}_{te}^\top := \mbf{\Sigma}_t^{1/2} \mbf{X}_{te}^\top$. Note $\widetilde{\mbf{X}}_{te}^\top = \begin{bmatrix}
        \mbf{A} \mbf{x}_1 & \dots & \mbf{A} \mbf{x}_m
    \end{bmatrix} := \begin{bmatrix}
        \widetilde{\mbf{x}}_1 & \dots & \widetilde{\mbf{x}}_m
    \end{bmatrix}$ where $\widetilde{\mbf{x}}_i := \mbf{A} \mbf{x}_i \overset{iid}{\sim} \mathcal{N}(\mbf{0}_d, \mbf{A} \mbf{A}^\top)$, and $\overline{\mbf{X}}_{te}^\top = \begin{bmatrix}
        \mbf{\Sigma}_t^{1/2} \mbf{x}_1 & \dots & \mbf{\Sigma}_t^{1/2} \mbf{x}_m
    \end{bmatrix} := \begin{bmatrix}
        \overline{\mbf{x}}_1 & \dots & \overline{\mbf{x}}_m
    \end{bmatrix}$ where $\overline{\mbf{x}}_i := \mbf{\Sigma}_t^{1/2} \mbf{x}_i \overset{iid}{\sim} \mathcal{N}(\mbf{0}_d, \mbf{\Sigma}_t)$. We can express $\widetilde{\mbf{X}}_{te}^\top \overline{\mbf{X}}_{te}$ and $\overline{\mbf{X}}_{te}^\top \widetilde{\mbf{X}}_{te}$ as such:
    \begin{equation*}
        \widetilde{\mbf{X}}_{te}^\top \overline{\mbf{X}}_{te} = \sum\limits_{i=1}^m \widetilde{\mbf{x}}_i \overline{\mbf{x}}_i^\top \; \; \text{and} \; \;  \overline{\mbf{X}}_{te}^\top \widetilde{\mbf{X}}_{te} = \sum\limits_{j=1}^m \overline{\mbf{x}}_j \widetilde{\mbf{x}}_j^\top.
    \end{equation*}
    Therefore,
    \begin{align*}
        &\mbb{E}\Big[ \operatorname{Tr}\left( \widetilde{\mbf{X}}_{te}^\top \overline{\mbf{X}}_{te} \overline{\mbf{X}}_{te}^\top \widetilde{\mbf{X}}_{te} \right) \Big] = \operatorname{Tr}\Big( \mbb{E}\left[ \widetilde{\mbf{X}}_{te}^\top \overline{\mbf{X}}_{te} \overline{\mbf{X}}_{te}^\top \widetilde{\mbf{X}}_{te} \right] \Big) \\
        &= \operatorname{Tr}\bigg( \sum\limits_{i=1}^m \sum\limits_{j=1}^m \mbb{E}\left[ \widetilde{\mbf{x}}_i \overline{\mbf{x}}_i^\top \overline{\mbf{x}}_j \widetilde{\mbf{x}}_j^\top \right] \bigg) \\
        &= \operatorname{Tr}\bigg( \sum\limits_{i=1}^m \sum\limits_{j\neq i} \mbb{E}\left[ \widetilde{\mbf{x}}_i \overline{\mbf{x}}_i^\top \overline{\mbf{x}}_j \widetilde{\mbf{x}}_j^\top \right] \bigg) + \operatorname{Tr}\bigg( \sum\limits_{i=1}^m \mbb{E}\left[ \widetilde{\mbf{x}}_i \overline{\mbf{x}}_i^\top \overline{\mbf{x}}_i \widetilde{\mbf{x}}_i^\top \right] \bigg)
    \end{align*}
    We first consider the case when $i \neq j$. In this setting, $\mbf{x}_i$ and $\mbf{x}_j$ are independent, so
    \begin{equation*}
        \mbb{E}\left[ \widetilde{\mbf{x}}_i \overline{\mbf{x}}_i^\top \overline{\mbf{x}}_j \widetilde{\mbf{x}}_j^\top \right] = \mbb{E}\left[ \widetilde{\mbf{x}}_i \overline{\mbf{x}}_i^\top \right] \mbb{E}\left[ \overline{\mbf{x}}_j \widetilde{\mbf{x}}_j^\top \right] = \mbf{A} \underbrace{\mbb{E}\left[\mbf{x}_i \mbf{x}_i^\top \right]}_{\mbf{I}_d} \mbf{\Sigma}_t \underbrace{\mbb{E} \left[\mbf{x}_j \mbf{x}_j^\top \right]}_{\mbf{I}_d} \mbf{A}^\top = \mbf{A} \mbf{\Sigma}_t \mbf{A}^\top.
    \end{equation*}
    Therefore,
    \begin{equation*}
        \operatorname{Tr}\bigg( \sum\limits_{i=1}^m \sum\limits_{j\neq i} \mbb{E}\left[ \widetilde{\mbf{x}}_i \overline{\mbf{x}}_i^\top \overline{\mbf{x}}_j \widetilde{\mbf{x}}_j^\top \right] \bigg) = m \cdot (m-1) \cdot \operatorname{Tr}\left( \mbf{A} \mbf{\Sigma}_t \mbf{A}^\top \right).
    \end{equation*}

    \noindent We now consider the case where $i = j$:
    \begin{align*}
        &\operatorname{Tr}\bigg( \sum\limits_{i=1}^m \mbb{E}\left[ \widetilde{\mbf{x}}_i \overline{\mbf{x}}_i^\top \overline{\mbf{x}}_i \widetilde{\mbf{x}}_i^\top \right] \bigg) = \sum\limits_{i=1}^m \mbb{E}\Big[  \operatorname{Tr}\left( \widetilde{\mbf{x}}_i \overline{\mbf{x}}_i^\top \overline{\mbf{x}}_i \widetilde{\mbf{x}}_i^\top \right) \Big] \\
        &= \sum\limits_{i=1}^m \mbb{E} \left[ \widetilde{\mbf{x}}_i^\top \widetilde{\mbf{x}}_i \overline{\mbf{x}}_i^\top \overline{\mbf{x}}_i \right] = \sum\limits_{i=1}^m \mbb{E} \left[ (\mbf{x}_i^\top \mbf{A}^\top \mbf{A} \mbf{x}_i)(\mbf{x}_i^\top \mbf{\Sigma}_t \mbf{x}_i) \right] \\
        &\overset{(i)}{=} m \cdot \bigg( 2\operatorname{Tr}\Big( \mbf{A} \mbf{\Sigma}_t \mbf{A}^\top \Big) + \operatorname{Tr}\big(\mbf{A}^\top\mbf{A}\big) \operatorname{Tr}\Big(\mbf{\Sigma}_t\Big) \bigg),
    \end{align*}
    where $(i)$ is because for 
    $\mbf{a} \sim \mathcal{N}(\mbf{0}_d, \mbf{I}_d)$ and fixed $\mbf{Q}, \mbf{R} \in \mbb{R}^{d \times d}$, $\mbb{E}\left[ (\mbf{a}^\top \mbf{Q} \mbf{a})(\mbf{a}^\top \mbf{R} \mbf{a}) \right] = \operatorname{Tr}\left( \mbf{Q} (\mbf{R} + \mbf{R}^\top) \right) + \operatorname{Tr}\left(\mbf{Q}\right) \operatorname{Tr}\left(\mbf{R}\right)$ (see Section 8.2.4 in \cite{petersen2008matrix}).
    
    \medskip
    
    \noindent We now focus on $(e)$:
    \begin{align*}
        &\mbb{E} \left[ \mbf{x}_q^\top \mbf{A} \mbf{X}_{te}^\top \boldsymbol{\eta}_{te} \boldsymbol{\eta}_{te}^\top \mbf{X}_{te} \mbf{A}^\top \mbf{x}_q \right] = \mbb{E}\Big[ \operatorname{Tr}\left( \mbf{x}_q \mbf{x}_q^\top \mbf{A} \mbf{X}_{te}^\top \boldsymbol{\eta}_{te} \boldsymbol{\eta}_{te}^\top \mbf{X}_{te} \mbf{A}^\top  \right) \Big] \\
        &= \operatorname{Tr}\Big( \underbrace{\mbb{E}\left[ \mbf{x}_q \mbf{x}_q^\top \right]}_{\mbf{I}_d} \mbf{A} \mbb{E}\left[ \mbf{X}_{te}^\top \boldsymbol{\eta}_{te} \boldsymbol{\eta}_{te}^\top \mbf{X}_{te} \right] \mbf{A}^\top  \Big) = \operatorname{Tr} \Big( \mbb{E}\left[ \mbf{A} \mbf{X}_{te}^\top \boldsymbol{\eta}_{te} \boldsymbol{\eta}_{te}^\top \mbf{X}_{te} \mbf{A}^\top  \right] \Big) \\
        & := \operatorname{Tr}\Big( \mbb{E}\left[  \widetilde{\boldsymbol{\eta}}_{te} \widetilde{\boldsymbol{\eta}}_{te}^\top \right] \Big),
    \end{align*}
    where $\widetilde{\boldsymbol{\eta}}_{te} := \mbf{A} \mbf{X}_{te}^\top \boldsymbol{\eta}_{te} = \widetilde{\mbf{X}}_{te}^\top \boldsymbol{\eta}_{te}$.  Note the columns of $\widetilde{\mbf{X}}_{te}^\top$ are iid Gaussian with covariance $\mbf{A}\mbf{A}^\top$. By Corollary 6 in \cite{mattei2017multiplying}, $\widetilde{\boldsymbol{\eta}}_{te} \sim \text{GAL}_d(2\sigma^2 \mbf{A}\mbf{A}^\top, \mbf{0}_d, m/2)$, where $\text{GAL}_p(\mbf{\Sigma}, \boldsymbol{\mu}, s)$ denotes a $p$-dimensional \emph{multivariate generalized asymmetric Laplace distribution} with mean $s \boldsymbol{\mu}$ and covariance $s(\mbf{\Sigma} + \boldsymbol{\mu} \boldsymbol{\mu}^\top)$ (Definition 1 and Proposition 2 in \cite{mattei2017multiplying}). Therefore,
    \begin{align*}
        \operatorname{Tr}\Big( \mbb{E}\left[  \widetilde{\boldsymbol{\eta}}_{te} \widetilde{\boldsymbol{\eta}}_{te}^\top \right] \Big) = \operatorname{Tr}\Big( \operatorname{Cov}(\widetilde{\boldsymbol{\eta}}_{te}) \Big) = m \sigma^2 \operatorname{Tr}\Big( \mbf{A} \mbf{A}^\top \Big).
    \end{align*}
    
    \paragraph{Adding $(a)$, $(b)$, and $(c)$.} Adding the expressions for $(a)$, $(b)$, and $(c)$, where $(c) = (d) + (e)$, yields and combining like terms yields the following expression:
    \begin{align*}
        \mbb{E} &\left[ \left(\widetilde{\mbf{w}}^\top \mbf{x}_q + \eta_q - \widehat{\mbf{w}}^\top \mbf{x}_q \right)^2 \right] = \underbrace{\operatorname{Tr}(\mbf{\Sigma}_t) + \sigma^2}_{= (a)} - 2 \underbrace{\operatorname{Tr}\left( \mbf{\Sigma}_t \mbf{A} \right)}_{= (b)} \\
        &+ \underbrace{\frac{1}{m^2} \Big( \underbrace{m(m-1) \operatorname{Tr}\left(\mbf{A} \mbf{\Sigma}_t\mbf{A}^\top\right) + 2m \operatorname{Tr}\left(\mbf{A} \mbf{\Sigma}_t \mbf{A}^\top  \right) + m \operatorname{Tr}(\mbf{\Sigma}_t) \operatorname{Tr}\left( \mbf{A}^\top \mbf{A} \right)}_{=(d)} + \underbrace{m \sigma^2 \operatorname{Tr}\left( \mbf{A}^\top \mbf{A} \right)}_{=(e)} \Big)}_{=(c)}.
    \end{align*}
    Combining like terms yields
    \begin{align*}
        \mbb{E} \left[ \left(\widetilde{\mbf{w}}^\top \mbf{x}_q + \eta_q - \widehat{\mbf{w}}^\top \mbf{x}_q \right)^2 \right] &= \left( \frac{1}{m} \operatorname{Tr}\left( \mbf{A}^\top\mbf{A} \right) + 1 \right) \Big(\operatorname{Tr}\left(\mbf{\Sigma}_t\right) + \sigma^2 \Big) - 2 \operatorname{Tr}\left(\mbf{\Sigma}_t \mbf{A} \right) + \frac{m+1}{m} \operatorname{Tr}\left(\mbf{A} \mbf{\Sigma}_t \mbf{A}^\top\right) \\
        &= M_t - \operatorname{Tr}\left(\mbf{\Sigma}_t\mbf{A}\right) + \frac{M_t}{m} \operatorname{Tr}\left( \mbf{A}^\top \mbf{A} \right) -  \operatorname{Tr}\left(\mbf{\Sigma}_t \mbf{A} \right) + \frac{m+1}{m} \operatorname{Tr}\left(\mbf{A} \mbf{\Sigma}_t \mbf{A}\right),
    \end{align*}
    which is exactly Equation~(\ref{eq:risk-expanded}). This completes the proof.
\end{proof}

\subsubsection{Proof of \Cref{thm:mix_two_subspaces,thm:mixture_k_subspaces}}
\label{sec:proof_of_k_subspaces}

\begin{proof}
We only provide a proof for \Cref{thm:mixture_k_subspaces}, as \Cref{thm:mix_two_subspaces} is a special case of \Cref{thm:mixture_k_subspaces} when $K = 2$, $\alpha_1 = \sin(\theta)$, and $\alpha_2 = \cos(\theta)$ for any $\theta \in [0, \pi/2]$. 

\bigskip

\noindent For simplicity, let $\widetilde{y} := \widetilde{y}_{m+1}$. By Lemma~\ref{lem:task_test_risk} and Lemma~\ref{lem:opt_mix_K_weights}, we have
    \begin{align}
        \mbb{E} &\left[ \left(\widetilde{y} - g^\star_{\texttt{ATT}}(\widetilde{\mbf{Z}}) \right)^2 \right] = \Big( \frac{1}{m} \operatorname{Tr}\left( \mbf{A}^\top\mbf{A} \right) + 1 \Big) \Big(\operatorname{Tr}\left(\overline{\mbf{\Sigma}}_t\right) + \sigma^2 \Big) - 2 \operatorname{Tr}\left(\overline{\mbf{\Sigma}}_t \mbf{A} \right) + \frac{m+1}{m} \operatorname{Tr}\left(\mbf{A} \overline{\mbf{\Sigma}}_t \mbf{A}^\top\right), \label{eq:task_test_risk_thm2}
    \end{align}
    where $\mbf{A} = \left( \frac{n+1}{n} \mbf{I}_d + \frac{M_s}{n} \mbf{\Sigma}^{-1} \right) ^{-1}$, $M_s = \tr(\mbf{\Sigma}) + \sigma^2$, and $\mbf{\Sigma} = \sum\limits_{k=1}^K \gamma_k \cdot \mbf{\Sigma}_{s, k}$. 
    \bigskip
    
    \noindent Let $\mbf{U} := \begin{bmatrix}
        \mbf{U}_{s, 1} & \mbf{U}_{s, 2} & \dots & \mbf{U}_{s, K} & \mbf{U}_\perp
    \end{bmatrix}$, where $\mbf{U}_\perp \in \mbb{R}^{d \times (d - Kr)}$ completes the orthonormal basis for $\mbb{R}^d$. By Lemma~\ref{lem:I-psd-inv}, 
    \begin{align*}
        \mbf{A} = \left(\frac{n+1}{n} \mbf{I}_d + \frac{M_s}{n} \mbf{\Sigma}^{-1} \right)^{-1} = \mbf{U \Lambda U}^\top,
    \end{align*}
    where 
    \begin{align*}
    \mbf{\Lambda} = \begin{bmatrix}
            \nu_1 \mbf{I}_r & & & \\
            & \ddots & & \\
            & & \nu_K \mbf{I}_r & \\
            & & & \nu_{K+1} \mbf{I}_{d-Kr} 
        \end{bmatrix}
    \end{align*}
    with $\nu_k = \frac{n(\gamma_k + \epsilon)}{(n+1)(\gamma_k + \epsilon) + M_s}$ for all $k \in [K]$, and $\nu_{K+1} = \frac{n \epsilon}{(n+1)\epsilon + r + \epsilon d + \sigma^2}$.
    
    \paragraph{Simplifying $\operatorname{Tr}(\overline{\mbf{\Sigma}}_t)$.} We can write $\operatorname{Tr}(\overline{\mbf{\Sigma}}_t)$ as such:
    \begin{equation*}
        \operatorname{Tr}(\overline{\mbf{\Sigma}}_t) = \operatorname{Tr}(\overline{\mbf{U}}_t \overline{\mbf{U}}_t^\top) + \epsilon \operatorname{Tr}(\mbf{I}_d) = r + \epsilon d.
    \end{equation*}
    
    \paragraph{Simplifying $\operatorname{Tr}(\mbf{A})$ and $\operatorname{Tr}(\mbf{A}^\top \mbf{A})$.} We can write $\operatorname{Tr}(\mbf{A})$ and $\operatorname{Tr}(\mbf{A}^\top \mbf{A})$ as such:
    \begin{align*}
        \tr(\mbf{A}) &= r \sum\limits_{k=1}^K \nu_k + (d-Kr) \nu_{K+1} \quad \text{and} \quad
        \tr(\mbf{A}^\top \mbf{A}) = \tr(\mbf{A}^2) = r\sum\limits_{k=1}^K \nu_k^2 + (d-Kr) \nu_{K+1}^2.
    \end{align*}
    
    \paragraph{Simplifying $\operatorname{Tr}(\overline{\mbf{\Sigma}}_t \mbf{A})$ and $\operatorname{Tr}(\mbf{A} \overline{\mbf{\Sigma}}_t \mbf{A}^\top)$.} Note $\operatorname{Tr}(\mbf{A} \overline{\mbf{\Sigma}}_t \mbf{A}^\top) = \operatorname{Tr}(\overline{\mbf{\Sigma}}_t \mbf{A}^2)$. We first focus on $\operatorname{Tr}(\overline{\mbf{\Sigma}}_t \mbf{A})$:
    \begin{align*}
        &\overline{\mbf{\Sigma}}_t \mbf{A} = \left( \overline{\mbf{U}}_t \overline{\mbf{U}}_t^\top + \epsilon \mbf{I}_d \right) \mbf{U} \mbf{\Lambda} \mbf{U}^\top = \overline{\mbf{U}}_t \overline{\mbf{U}}_t^\top \mbf{U} \mbf{\Lambda} \mbf{U}^\top + \epsilon \mbf{U} \mbf{\Lambda} \mbf{U}^\top \\
        &\implies \operatorname{Tr}\left( \overline{\mbf{\Sigma}}_t \mbf{A} \right) = \operatorname{Tr}\left( \overline{\mbf{U}}_t^\top \mbf{U} \mbf{\Lambda} \mbf{U}^\top \overline{\mbf{U}}_t \right) + \epsilon \operatorname{Tr}(\mbf{A}).
    \end{align*}
    Recall $\overline{\mbf{U}}_t = \sum\limits_{k=1}^K \alpha_k \mbf{U}_{s, k}$ where $\sum\limits_{k=1}^K \alpha_k^2 = 1$, and so we have
    \begin{align*}
        \overline{\mbf{U}}_t^\top \mbf{U} = \left( \sum\limits_{k=1}^K \alpha_k \mbf{U}_k \right)^\top \begin{bmatrix}
            \mbf{U}_{s, 1} & \dots & \mbf{U}_{s,K } & \mbf{U}_{\perp}
        \end{bmatrix} = \begin{bmatrix}
           \alpha_1 \mbf{I}_r & & & \\
           & \ddots & & \\
           & & \alpha_K \mbf{I}_r & \\
           & & & \mbf{0}_{(d - Kr) \times (d - Kr)}
        \end{bmatrix}
    \end{align*}
    Thus,
    \begin{align*}
        &\operatorname{Tr}\left( \overline{\mbf{U}}_t^\top \mbf{U} \mbf{\Lambda} \mbf{U}^\top \overline{\mbf{U}}_t \right) %= \operatorname{Tr} \left(  \begin{bmatrix}
        %    \alpha_1 \mbf{I}_r & & & \\
        %    & \ddots & & \\
        %    & & \alpha_k \mbf{I}_r & \\
        %    & & & \mbf{0}_{(d - Kr) \times (d - Kr)}
        % \end{bmatrix}
        % \begin{bmatrix}
        %     \nu_1 \mbf{I}_r & & & \\
        %     & \ddots & & \\
        %     & & \nu_K \mbf{I}_r & \\
        %     & & & \nu_{K+1} \mbf{I}_{d-Kr} 
        % \end{bmatrix} 
        % \begin{bmatrix}
        %    \alpha_1 \mbf{I}_r & & & \\
        %    & \ddots & & \\
        %    & & \alpha_k \mbf{I}_r & \\
        %    & & & \mbf{0}_{(d - Kr) \times (d - Kr)}
        % \end{bmatrix} \right) \\
        = \operatorname{Tr}\left( \begin{bmatrix}
            \alpha_1^2 \nu_1  \mbf{I}_r & & & \\
            & \ddots & & \\
            & & \alpha_K^2 \nu_K  \mbf{I}_r & \\
            & & & \mbf{0}_{(d - Kr) \times (d - Kr)}
        \end{bmatrix} \right) = r \sum\limits_{k=1}^K \alpha_k^2 \nu_k
    \end{align*}
    Using a similar argument, 
    \begin{equation*}
        \operatorname{Tr}\left( \overline{\mbf{\Sigma}}_t^\top \mbf{A}^2 \right) = r \sum\limits_{k=1}^K \alpha_k^2 \nu_k^2 + \epsilon \operatorname{Tr}\left(\mbf{A}^2\right). 
    \end{equation*}

    \paragraph{Simplifying the test risk.} Substituting the expressions for the $\operatorname{Tr}(\cdot)$ terms into \Cref{eq:task_test_risk_thm2} yields
    \begin{align*}
        \mbb{E}\left[ \left(\widetilde{y} - g^\star_{\texttt{ATT}}(\widetilde{\mbf{Z}})\right)^2 \right] = &\left(\frac{1}{m} \left( r \sum\limits_{k=1}^K \nu_k^2 + (d - Kr) \nu_{K+1}^2 \right) + 1 \right) \big( r + \epsilon d + \sigma^2 \big) \\
        &- 2 \left( r \sum\limits_{k=1}^K \alpha_k^2 \nu_k + \left( r \sum\limits_{k=1}^K \nu_k + (d - Kr) \nu_{K+1} \right)\epsilon  \right) \\
        &+ \frac{m+1}{m} \left( r \sum\limits_{k=1}^K \alpha_k^2 \nu_k^2 + \left( r \sum\limits_{k=1}^K \nu_k^2 + (d - Kr) \nu_{K+1}^2 \right)\epsilon \right).
    \end{align*}
    Taking $\epsilon \to 0$ results in the following expression for the test risk:
    \begin{align*}
        \lim\limits_{\epsilon \to 0} \mbb{E}\left[ \left(\widetilde{y} - g^\star_{\texttt{ATT}}(\widetilde{\mbf{Z}})\right)^2 \right] &= r + \sigma^2 + \frac{(r + \sigma^2) r}{m} \sum\limits_{k=1}^K \left( \frac{\gamma_k n}{\gamma_k(n+1) + r + \sigma^2} \right)^2 \\
        &-2r \sum\limits_{k=1}^K \frac{\alpha_k^2 \gamma_k n}{\gamma_k(n+1) + r + \sigma^2} + \frac{(m+1)r}{m}\sum\limits_{k=1}^K \left(\frac{\alpha_k \gamma_k n}{\gamma_k(n+1) + r + \sigma^2}\right)^2
    \end{align*}
    Substituting $\gamma_k = \frac{1}{K}$ for all $k \in [K]$ and combining like terms yields
    \begin{align*}
        \lim\limits_{\epsilon \to 0} \mbb{E}\left[ \left(\widetilde{y} - g^\star_{\texttt{ATT}}\left(\widetilde{\mbf{Z}} \right)\right)^2 \right] = r + \sigma^2 + \frac{m + 1 + K(r + \sigma^2)}{m} \cdot \frac{rn^2}{(n + 1 + K(r + \sigma^2))^2} - \frac{2rn}{n + 1 + K(r + \sigma^2)}.
    \end{align*}
    Now suppose $n \leq m$. Then, we have
    \begin{align*}
        \lim\limits_{\epsilon \to 0} \mbb{E}\left[ \left(\widetilde{y} - g^\star_{\texttt{ATT}}\left(\widetilde{\mbf{Z}} \right)\right)^2 \right] \leq r + \sigma^2 - \frac{rn^2}{n + 1 + K(r + \sigma^2)}.
    \end{align*}
    Upper bounding this by $\sigma^2 + \delta$ for some $\delta \in (0, r)$, then solving for $n$, yields the following result. For any $\delta \in (0, r)$, if
    \begin{equation*}
        m \geq n > \frac{(K(r + \sigma^2) + 1)r}{\delta} - (K(r + \sigma^2) + 1),
    \end{equation*}
    then $\lim\limits_{\epsilon \to 0} \mbb{E}\left[ \left(\widetilde{y} - g^\star_{\texttt{ATT}}\left(\widetilde{\mbf{Z}} \right)\right)^2 \right] < \sigma^2 + \delta,$ which completes the proof.  
\end{proof}

\subsubsection{Proof of \Cref{prop:neg_result,prop:neg_result_diff_angles}}
\label{sec:proof_of_neg_result}
We now provide the proof of \Cref{prop:neg_result}, which discusses how the proof can be generalized to the setting of \Cref{prop:neg_result_diff_angles}.

\begin{proof}
For simplicity, we denote $\widetilde{y} \coloneqq \widetilde{y}_{m+1}$.
     Recall $\mbf{U} := \begin{bmatrix}
        \mbf{U}_s & \mbf{U}_{s, \perp} & \mbf{U}_{2r, \perp}
    \end{bmatrix} \in \mbb{R}^{d \times d}$, where $\mbf{U}_s, \mbf{U}_{s, \perp} \in \mbb{R}^{d \times r}$ and $\mbf{U}_{2r, \perp} \in \mbb{R}^{d \times (d - 2r)}$ all have orthonormal columns, while $\mbf{U}_s^\top \mbf{U}_{s, \perp} = \mbf{0}_{r \times r}$ and $ \mbf{U}_s^\top \mbf{U}_{2r, \perp} = \mbf{U}_{s, \perp}^\top \mbf{U}_{2r, \perp} = \mbf{0}_{r \times (d - 2r)}$. We re-write $\mbf{\Sigma}_s$ as such:
    \begin{equation*}
        \mbf{\Sigma}_s = \mbf{U}_s\mbf{U}_s^\top + \epsilon \mbf{I}_d = \mbf{U} \begin{bmatrix}
            \mbf{I}_r & \ \\
            \ & \mbf{0}_{(d - r) \times (d - r)}
        \end{bmatrix} \mbf{U}^\top + \epsilon \mbf{I} = \mbf{U} \begin{bmatrix}
            (1 + \epsilon) \mbf{I}_r & \ \\
            \ & \epsilon \mbf{I}_{d - r}
        \end{bmatrix} \mbf{U}^\top.
    \end{equation*}
    Note this is a valid eigendecomposition of $\mbf{\Sigma}_s$. Thus, by Lemma~\ref{lem:I-psd-inv}, we have
    \begin{equation} \label{eq:A-eigen}
        \mbf{A} = \left( \frac{n+1}{n} \mbf{I}_d + \frac{M_s}{n} \mbf{\Sigma}_s^{-1}  \right)^{-1} = \mbf{U} \mbf{\Lambda} \mbf{U}^\top,
    \end{equation}
    where 
    \begin{align*}
        \mbf{\Lambda} = \begin{bmatrix}
            \frac{n(1 + \epsilon)}{(n + 1)(1+\epsilon) + M_s} \cdot \mbf{I}_r & \ \\
            \ & \frac{n \epsilon}{(n+1)\epsilon + M_s} \cdot \mbf{I}_{d - r}
        \end{bmatrix} := \begin{bmatrix}
        \nu_1 \mbf{I}_r & \ \\
        \ & \nu_2 \mbf{I}_{d - r}
        \end{bmatrix}.
    \end{align*}
     and $M_s = \operatorname{Tr}(\mbf{\Sigma}_s) + \sigma^2$. 
    
    \bigskip
    \bigskip
    
    \noindent By Lemma~\ref{lem:task_test_risk} (and omitting the subscripts in the expectation), 
    \begin{equation} \label{eq:task_test_risk_prop1}
        \mbb{E}\left[ \left(\widetilde{y} - g_{\texttt{ATT}}^\star\left(\widetilde{\mbf{Z}}\right) \right)^2 \right] = \left( \frac{1}{m} \operatorname{Tr}\left( \mbf{A}^\top\mbf{A} \right) + 1 \right) \Big(\operatorname{Tr}\left(\mbf{\Sigma}_t\right) + \sigma^2 \Big) - 2 \operatorname{Tr}\left(\mbf{\Sigma}_t \mbf{A} \right) + \frac{m+1}{m} \operatorname{Tr}\left(\mbf{A} \mbf{\Sigma}_t \mbf{A}^\top\right) .
    \end{equation} 
    We simplify the remaining $\operatorname{Tr}(\cdot)$ terms using Equation~(\ref{eq:A-eigen}).

    \paragraph{Simplifying $\operatorname{Tr}(\mbf{A})$ and $\operatorname{Tr}(\mbf{A}^\top \mbf{A})$.} Directly from Equation~(\ref{eq:A-eigen}):
    \begin{align*}
       \operatorname{Tr}(\mbf{A}) = r \cdot \nu_1 + (d - r) \cdot \nu_2 \; \; \text{and} \; \; \operatorname{Tr}(\mbf{A}^\top \mbf{A}) = \operatorname{Tr}(\mbf{A}^2) = r \cdot \nu_1^2 + (d - r) \cdot \nu_2^2,
    \end{align*}
    where $\mbf{A}^2 = \mbf{U} \mbf{\Lambda}^2 \mbf{U}^\top$.

    \paragraph{Simplifying $\operatorname{Tr}(\mbf{\Sigma}_t \mbf{A})$ and $\operatorname{Tr}(\mbf{A} \mbf{\Sigma}_t \mbf{A}^\top)$.} First note $\operatorname{Tr}(\mbf{A} \mbf{\Sigma}_t \mbf{A}^\top) = \operatorname{Tr}(\mbf{\Sigma}_t \mbf{A}^2)$. We first focus on $\operatorname{Tr}(\mbf{\Sigma}_t \mbf{A})$:
    \begin{align*}
        &\mbf{\Sigma}_t \mbf{A} = \left( \mbf{U}_t \mbf{U}_t^\top + \epsilon \mbf{I}_d \right) \mbf{U} \mbf{\Lambda} \mbf{U}^\top = \mbf{U}_t \mbf{U}_t^\top \mbf{U} \mbf{\Lambda} \mbf{U}^\top + \epsilon \mbf{U} \mbf{\Lambda} \mbf{U}^\top \\
        &\implies \operatorname{Tr}\left( \mbf{\Sigma}_t \mbf{A} \right) = \operatorname{Tr}\left( \mbf{U}_t^\top \mbf{U} \mbf{\Lambda} \mbf{U}^\top \mbf{U}_t \right) + \epsilon \operatorname{Tr}(\mbf{A}).
    \end{align*}
    Recall we defined $\mbf{U}_t$ in Equation~(\ref{eq:task_subspace}) as follows:
    \begin{equation*}
        \mbf{U}_t = \mbf{U}_s\cos(\mbf{\Theta})   + \mbf{U}_{s,\perp}\sin(\mbf{\Theta}).
    \end{equation*}
    Therefore:
    \begin{align*}
        \mbf{U}_t^\top \mbf{U} = \left(\mbf{U}_s\cos(\mbf{\Theta})   + \mbf{U}_{s,\perp}\sin(\mbf{\Theta}) \right)^\top \begin{bmatrix}
            \mbf{U}_s & \mbf{U}_{s,\perp} & \mbf{U}_{\perp}
        \end{bmatrix} = \begin{bmatrix}
            \cos(\mbf{\Theta})  & \sin(\mbf{\Theta})  & \mbf{0}_{r \times d-2r}
        \end{bmatrix},
    \end{align*}
    and thus,
    \begin{align*}
        &\operatorname{Tr}\left( \mbf{U}_t^\top \mbf{U} \mbf{\Lambda} \mbf{U}^\top \mbf{U}_t \right) = \operatorname{Tr} \left( \begin{bmatrix}
            \cos(\mbf{\Theta})  & \sin(\mbf{\Theta})  & \mbf{0}_{r \times (d-2r)}
        \end{bmatrix} \begin{bmatrix}
                \nu_1 \mbf{I}_r & \ & \ \\
                \ & \nu_2 \mbf{I}_r & \ \\
                \ & \ & \nu_2 \mbf{I}_{d - 2r}
            \end{bmatrix} \begin{bmatrix}
            \cos(\mbf{\Theta})  \\ \sin(\mbf{\Theta})  \\ \mbf{0}_{(d-2r) \times r}
        \end{bmatrix} \right) \\
        &= \operatorname{Tr}\left( \begin{bmatrix}
            \nu_1 \cos^2(\mbf{\Theta}) & \ & \ \\
            \ & \nu_2 \sin^2(\mbf{\Theta}) & \ \\
            \ & \ & \mbf{0}_{(d - 2r) \times (d - 2r)}
        \end{bmatrix} \right) = r \cdot \nu_1 \cdot \cos^2(\theta) + r \cdot \nu_2 \cdot \sin^2(\theta),
    \end{align*}
    where we used the fact that the principal angles are all equal to $\theta$.
    Using a similar argument, 
    \begin{equation*}
        \operatorname{Tr}\left( \mbf{\Sigma}_t^\top \mbf{A}^2 \right) = r \cdot \nu_1^2 \cdot \cos^2(\theta) + r \cdot \nu_2^2 \cdot \sin^2(\theta) + \epsilon \operatorname{Tr}(\mbf{A}^2)
    \end{equation*}

    \paragraph{Simplifying the Test Risk.} Substituting the expressions for the $\operatorname{Tr}(\cdot)$ terms into Equation~(\ref{eq:task_test_risk_prop1}) yields
    \begin{align*}
        \mbb{E}\left[ \left(\widetilde{y} - g_{\texttt{ATT}}^\star\left(\widetilde{\mbf{Z}}\right) \right)^2 \right] = &\Big( \frac{1}{m} \big( r \nu_1^2 + (d - r) \nu_2^2 \big) + 1 \Big) \big( r + \epsilon d + \sigma^2 \big) \\
        &- 2 \left( r \nu_1 \cos^2(\theta) + r \nu_2 \sin^2(\theta) + \big( r \nu_1 + (d - r) \nu_2 \big)\epsilon  \right) \\
        &+ \frac{m+1}{m} \left( r \nu_1^2 \cos^2(\theta) + r \nu_2^2 \sin^2(\theta) + \big( r \nu_1^2 + (d - r) \nu_2^2 \big) \epsilon \right)
    \end{align*}
    Substituting the expressions for $\nu_1$ and $\nu_2$ and taking $\epsilon \rightarrow 0$ results in the following:
   \begin{align*}
        \lim\limits_{\epsilon \rightarrow 0} \mbb{E}\left[ \left(\widetilde{y} - g_{\texttt{ATT}}^\star\left(\widetilde{\mbf{Z}}\right) \right)^2 \right]
        &= \left( \frac{r n^2}{m (n + 1 + r + \sigma^2)^2} + 1 \right)(r + \sigma^2) \\
        &\quad - \frac{2 r n \cos^2(\theta)}{n + 1 + r + \sigma^2} 
        + \frac{(m + 1) r n^2 \cos^2(\theta)}{m (n + 1 + r + \sigma^2)^2}
    \end{align*}  
    Subsequently taking $m, n \rightarrow \infty$ yields
    \begin{align*}
        \lim\limits_{m \rightarrow \infty} \lim\limits_{n \rightarrow \infty} \lim\limits_{\epsilon \rightarrow 0} \mbb{E}&\left[ \left( \widetilde{y} - g_{\texttt{ATT}}^\star\left(\widetilde{\mbf{Z}}\right) \right)^2 \right]  = r + \sigma^2 - r \cos^2(\theta) = r \sin^2(\theta) + \sigma^2,
    \end{align*}
    which completes the proof. To generalize this result to the case in which all of the principal angles are not the same, i.e., $\theta_i \neq \theta$ for all $i \in [r]$, we can replace all terms with $r\cos^2(\theta)$ and $r\sin^2(\theta)$ with $\sum_{i=1}^r \cos^2(\theta_i)$ and $\sum_{i=1}^r \sin^2(\theta_i)$, respectively, due to the trace terms. This gives us an overall test risk of 
    \begin{align*}
        \lim\limits_{m \rightarrow \infty} \lim\limits_{n \rightarrow \infty} \lim\limits_{\epsilon \rightarrow 0} \mbb{E}&\left[ \left( \widetilde{y} - g_{\texttt{ATT}}^\star\left(\widetilde{\mbf{Z}}\right) \right)^2 \right]  = r + \sigma^2 - \sum_{i=1}^r \cos^2(\theta_i) = \sum_{i=1}^r \sin^2(\theta_i) + \sigma^2,
    \end{align*}
    which gives the result in \Cref{prop:neg_result_diff_angles}.
\end{proof}

\subsection{Proofs for Feature Shifts}
\label{sec:feature_shifts_proofs}

In this section, we provide proofs of our theoretical results on the generalization abilities of a linear attention model when the features undergo a distribution shift. 
\subsubsection{Supporting Results}
We first derive an expression for the test risk under a general distribution shift for the features.
\begin{lemma}[Test Risk Under General Feature Distribution Shift] \label{lem:feature_test_risk}

Let $g_{\texttt{ATT}}^\star$ denote the optimal linear attention model corresponding to the independent data setting in Equation~(\ref{eqn:feature_shift}). For all $j \in [m+1]$, suppose that the prompts at test time are constructed with task vectors $\mbf{w} \sim \mc{N}(\mbf{0}, \mbf{I}_d)$ and labels 
    \begin{align*}
        \widetilde{y}_j = \mbf{w}^\top \widetilde{\mbf{x}}_j + \eta_j, \quad  \text{where} \quad \widetilde{\mbf{x}}_j \sim \mc{N}(\mbf{0}, \mbf{\Sigma}_t) \quad \text{and} \quad \eta_j \sim \mc{N}(0, \sigma^2 ),
    \end{align*}
   Then, we have
    \begin{align}
         \mbb{E}\left[ \left(\widetilde{y}_{m+1} - g_{\texttt{ATT}}^\star\left(\widetilde{\mbf{Z}}\right)\right)^2 \right] &= M_t - 2 \operatorname{Tr}\left( \mbf{\Sigma}_t^2 \mbf{A} \right) + \frac{m+1}{m} \operatorname{Tr}\left(\mbf{A} \mbf{\Sigma}_t^3 \mbf{A} \right) \nonumber \\
        &+ \frac{1}{m} \operatorname{Tr}\left( \mbf{\Sigma}_t^2 \right) \operatorname{Tr}\left( \mbf{A} \mbf{\Sigma}_t \mbf{A} \right) + \frac{\sigma^2}{m} \operatorname{Tr}\left(\mbf{A} \mbf{\Sigma}_t^2 \mbf{A}\right), \label{eq:feature_shift_risk}
    \end{align}
    where $M_t = \operatorname{Tr}\left(\mbf{\Sigma}_t\right) + \sigma^2$.
\end{lemma}
\begin{proof}
    The proof is similar to that of Lemma~\ref{lem:task_test_risk} with $\widetilde{y}_{m+1} = \mbf{w}^\top \widetilde{\mbf{x}}_{m+1} + \eta_q$, where $\widetilde{\mbf{x}}_{m+1} \sim \mc{N}(\mbf{0}, \mbf{\Sigma}_t)$. Recall at inference time,
    \begin{align*}
        \widetilde{\mbf{Z}}_{\mc{M}} = \begin{bmatrix}
            \mbf{z}_1 & \ldots & \mbf{z}_m & \mbf{0}
        \end{bmatrix}^\top = \begin{bmatrix}
            \widetilde{\mbf{x}}_1 & \ldots & \widetilde{\mbf{x}}_m & \mbf{0} \\
            \widetilde{y}_1 & \ldots & \widetilde{y}_m & 0
        \end{bmatrix}^\top \quad \text{and} \quad \tilde{\mbf{z}}_q = \begin{bmatrix}
            \widetilde{\mbf{x}}_{m+1} \\
            0
        \end{bmatrix} := \begin{bmatrix}
            \widetilde{\mbf{x}}_q \\ 
            0 
        \end{bmatrix}.
    \end{align*}
    Again we define
    \begin{align*}
        \mbf{X}_{te} = \begin{bmatrix}
            \widetilde{\mbf{x}}_1 & \widetilde{\mbf{x}}_2 & \ldots & \widetilde{\mbf{x}}_m
        \end{bmatrix}^\top, \quad \mbf{y}_{te} = \begin{bmatrix}
            \widetilde{y}_1 & \widetilde{y}_2 & \ldots & \widetilde{y}_m
        \end{bmatrix}^\top, \quad \boldsymbol{\eta}_{te} = \begin{bmatrix}
            \eta_1 & \eta_2 & \ldots & \eta_m
        \end{bmatrix}^\top,
    \end{align*}
    and $\eta_q := \eta_{m+1}$.
     By \Cref{lem:opt_weights_noniso_features}, we have
    \begin{align*}
        g_{\texttt{ATT}}^\star\left(\widetilde{\mbf{Z}} \right) = \frac{1}{m} \widetilde{\mbf{x}}_q^\top \mbf{A} \mbf{X}_{te}^\top \mbf{y}_{te} = \widetilde{\mbf{x}}_q^\top \underbrace{\left(\frac{1}{m} \mbf{A} \mbf{X}_{te}^\top \mbf{y}_{te}\right)}_{\coloneqq \widehat{\mbf{w}}},
    \end{align*}
    where $\mbf{A} = \mbf{\Sigma}_s^{-1/2} \bar{\mbf{A}} \mbf{\Sigma}_s^{-1/2}$ and $\bar{\mbf{A}}  = \left( \frac{n+1}{n} \mbf{I}_d + \frac{M_s}{n} \mbf{\Sigma}_s^{-1/2} \right)^{-1}$.
    By plugging this into the risk and linearity of expectation, 
    \begin{equation} %\label{eq:risk-expanded}
        \mbb{E} \left[ \left(\mbf{w}^\top \widetilde{\mbf{x}}_q + \eta_q - \widehat{\mbf{w}}^\top \widetilde{\mbf{x}}_q \right)^2 \right] = \underbrace{\mbb{E} \left[ \left(\mbf{w}^\top \widetilde{\mbf{x}}_q + \eta_q \right)^2 \right]}_{(a)} - 2 \underbrace{\mbb{E} \left[ \left( \mbf{w}^\top \widetilde{\mbf{x}}_q + \eta_q \right) \left( \widetilde{\mbf{x}}_q^\top \widehat{\mbf{w}} \right) \right]}_{(b)} + \underbrace{\mbb{E} \left[ \left( \widehat{\mbf{w}}^\top \widetilde{\mbf{x}}_q \right)^2 \right]}_{(c)},
    \end{equation}
    so it suffices to analyze each individual term. 
    
    \paragraph{Analyzing $(a)$.} We first evaluate $\mbb{E}\left[ \left(\mbf{w}^\top \widetilde{\mbf{x}}_q + \eta_q \right)^2 \right]$. First, we note 
    \begin{equation*}
        \mbb{E} \left[ \left( \mbf{w}^\top \widetilde{\mbf{x}}_q + \eta_q \right)^2 \right] = \mbb{E} \left[ \left(\mbf{w}^\top \widetilde{\mbf{x}}_q \right)^2 \right] + 2 \mbb{E}\left[ \eta_q \mbf{w}^\top \widetilde{\mbf{x}}_q \right] + \mbb{E}\left[ \eta_q^2 \right] = \mbb{E} \left[ \left( \mbf{w}^\top \widetilde{\mbf{x}}_q \right)^2 \right] + \sigma^2,
    \end{equation*}
    so it suffices to analyze $\mbb{E} \left[ \left( \mbf{w}^\top \widetilde{\mbf{x}}_q \right)^2 \right]$. By law of total expectation and the fact that $\mbf{w}, \widetilde{\mbf{x}}_q$ are independent,
    \begin{equation*}
        \mbb{E} \left[ \left( \mbf{w}^\top \widetilde{\mbf{x}}_q \right)^2 \right] = \mbb{E}_{\widetilde{\mbf{x}}_q} \left[ \mbb{E}_{\mbf{w}} \left[ \left( \mbf{w}^\top \widetilde{\mbf{x}}_q \right)^2 \; \big| \: \widetilde{\mbf{x}}_q \right] \right]. 
    \end{equation*}
    Conditioned on $\widetilde{\mbf{x}}_q$, $\mbf{w}^\top \widetilde{\mbf{x}}_q \sim \mathcal{N}(0, \|\widetilde{\mbf{x}}_q\|^2)$, so $\mbb{E} \left[ \left( \mbf{w}^\top \widetilde{\mbf{x}}_q \right)^2 \: \big| \: \widetilde{\mbf{x}}_q \right] = \mathrm{Var}\left( \mbf{w}^\top \widetilde{\mbf{x}}_q \: | \: \mbf{w} \right) = \|\widetilde{\mbf{x}}_q\|^2$. Therefore, 
    \begin{align*}
        &\mbb{E}_{\mbf{w}} \left[ \mbb{E}_{\mbf{x}} \big[ \left( \mbf{w}^\top \widetilde{\mbf{x}}_q \right)^2 \; \big| \: \mbf{w} \big] \right] = \mbb{E} \left[ \|\widetilde{\mbf{x}}_q\|^2 \right] = \operatorname{Tr}\left( \mbb{E} \left[ \widetilde{\mbf{x}}_q \widetilde{\mbf{x}}_q^\top \right] \right) = \operatorname{Tr}\left(\mbf{\Sigma}_t \right).
    \end{align*}
    Therefore,
    \begin{equation*}
        \mbb{E}\left[ \left(\mbf{w}^\top \widetilde{\mbf{x}}_q + \eta_q \right)^2 \right] = \operatorname{Tr}\left(\mbf{\Sigma}_t\right) + \sigma^2.
    \end{equation*}

    \bigskip
    
    \paragraph{Analyzing $(b)$.} Next, we analyze $\mbb{E} \left[ \left( \mbf{w}^\top \widetilde{\mbf{x}}_q + \eta_q \right) \left( \widetilde{\mbf{x}}_q^\top \widehat{\mbf{w}} \right) \right]$. We first note
    \begin{equation*}
        \mbb{E} \left[ \left( \mbf{w}^\top \widetilde{\mbf{x}}_q + \eta_q \right) \left( \widetilde{\mbf{x}}_q^\top \widehat{\mbf{w}} \right) \right] = \mbb{E} \left[ \left( \mbf{w}^\top \widetilde{\mbf{x}}_q \right) \left( \widetilde{\mbf{x}}_q^\top \widehat{\mbf{w}} \right) \right] + \underbrace{\mbb{E}\left[ \eta_q \widetilde{\mbf{x}}_q^\top \widehat{\mbf{w}} \right]}_{= 0} = \mbb{E} \left[ \left( \mbf{w}^\top \widetilde{\mbf{x}}_q \right) \left( \widetilde{\mbf{x}}_q^\top \widehat{\mbf{w}} \right) \right], 
    \end{equation*}
    so it suffices to analyze $\mbb{E} \left[ \left( \mbf{w}^\top \widetilde{\mbf{x}}_q \right) \left( \widetilde{\mbf{x}}_q^\top \widehat{\mbf{w}} \right) \right]$. Note $\widehat{\mbf{w}} := \frac{1}{m} \mbf{A} \mbf{X}_{te}^\top \mbf{y}_{te} = \frac{1}{m} \mbf{A} \mbf{X}_{te}^\top (\mbf{X}_{te}\mbf{w} + \boldsymbol{\eta}_{te})$, so
    \begin{align*}
        &\mbb{E} \left[ \left( \mbf{w}^\top \widetilde{\mbf{x}}_q \right) \left( \widetilde{\mbf{x}}_q^\top \widehat{\mbf{w}} \right) \right] = \frac{1}{m} \mbb{E} \left[ \mbf{w}^\top \widetilde{\mbf{x}}_q \widetilde{\mbf{x}}_q^\top \mbf{A} \mbf{X}_{te}^\top (\mbf{X}_{te} \mbf{w} + \boldsymbol{\eta}_{te}) \right] \\
        &= \frac{1}{m} \Big( \mbb{E}\left[ \mbf{w}^\top \widetilde{\mbf{x}}_q \widetilde{\mbf{x}}_q^\top \mbf{A} \mbf{X}_{te}^\top \mbf{X}_{te} \mbf{w} \right] + \mbb{E} \left[ \mbf{w}^\top \widetilde{\mbf{x}}_q \widetilde{\mbf{x}}_q^\top \mbf{A} \mbf{X}_{te}^\top \boldsymbol{\eta}_{te} \right] \Big) \\
        &= \frac{1}{m} \Big( \mbb{E}\left[ \mbf{w}^\top \widetilde{\mbf{x}}_q \widetilde{\mbf{x}}_q^\top \mbf{A} \mbf{X}_{te}^\top \mbf{X}_{te} \mbf{w} \right] + \underbrace{\mbb{E} \left[ \mbf{w}^\top \widetilde{\mbf{x}}_q \widetilde{\mbf{x}}_q^\top \mbf{A} \mbf{X}_{te}^\top\right] \mbb{E}\left[ \boldsymbol{\eta}_{te} \right]}_{=0} \Big) \\
        &= \frac{1}{m} \mbb{E}\left[ \mbf{w}^\top \widetilde{\mbf{x}}_q \widetilde{\mbf{x}}_q^\top \mbf{A} \mbf{X}_{te}^\top \mbf{X}_{te} \mbf{w} \right] = \frac{1}{m} \mbb{E} \Big[ \operatorname{Tr}\left( \mbf{w} \mbf{w}^\top \widetilde{\mbf{x}}_q \widetilde{\mbf{x}}_q^\top \mbf{A} \mbf{X}_{te}^\top \mbf{X}_{te}  \right) \Big] \\
        &= \frac{1}{m} \operatorname{Tr} \Big( \mbb{E} \left[ \mbf{w} \mbf{w}^\top \widetilde{\mbf{x}}_q \widetilde{\mbf{x}}_q^\top \mbf{A} \mbf{X}_{te}^\top \mbf{X}_{te}  \right] \Big) \\
        &= \frac{1}{m} \operatorname{Tr} \Big( \underbrace{\mbb{E} \left[ \mbf{w} \mbf{w}^\top \right]}_{ \mbf{I}_d} \underbrace{\mbb{E} \left[ \widetilde{\mbf{x}}_q \widetilde{\mbf{x}}_q^\top \right]}_{\mbf{\Sigma}_t} \mbf{A} \underbrace{\mbb{E} \left[ \mbf{X}_{te}^\top \mbf{X}_{te} \right]}_{m \cdot \mbf{\Sigma}_t} \Big) = \operatorname{Tr} \Big( \mbf{\Sigma}_t^2 \mbf{A} \Big),
    \end{align*}
    where again $\mbf{A} = \mbf{\Sigma}_s^{-1/2} \mbf{\bar{A}} \mbf{\Sigma}_s^{-1/2}$ and $\mbf{\bar{A}} = \left( \frac{n+1}{n} \mbf{I}_d + \frac{M_s}{n} \mbf{\Sigma}_s^{-1} \right)^{-1}$. 
    
    \bigskip
    
    \paragraph{Analyzing $(c)$.} Finally, we analyze $\mbb{E}\left[ (\widetilde{\mbf{x}}_q^\top \widehat{\mbf{w}} )^2 \right]$:
    \begin{align*}
        &\mbb{E}\left[ (\widetilde{\mbf{x}}_q^\top \widehat{\mbf{w}})^2 \right] = \frac{1}{m^2} \mbb{E}\left[ \left( \widetilde{\mbf{x}}_q^\top \mbf{A}\mbf{X}_{te}^\top (\mbf{X}_{te} \mbf{w} + \boldsymbol{\eta}_{te}) \right)^2 \right] = \mbb{E} \left[ \big( \widetilde{\mbf{x}}_q^\top \mbf{A} \mbf{X}_{te}^\top \mbf{X}_{te} \mbf{w} + \widetilde{\mbf{x}}_q^\top \mbf{A} \mbf{X}_{te}^\top \boldsymbol{\eta}_{te} \big)^2 \right] \\
        &= \frac{1}{m^2} \Big( \mbb{E} \left[ (\widetilde{\mbf{x}}_q^\top \mbf{A} \mbf{X}_{te}^\top \mbf{X}_{te} \mbf{w})^2 \right] + 2 \underbrace{\mbb{E}\left[ (\widetilde{\mbf{x}}_q^\top \mbf{A} \mbf{X}_{te}^\top \mbf{X}_{te} \mbf{w})(\widetilde{\mbf{x}}_q^\top \mbf{A} \mbf{X}_{te}^\top \boldsymbol{\eta}_{te}) \right]}_{=0} + \mbb{E}\left[ (\widetilde{\mbf{x}}_q^\top \mbf{A} \mbf{X}_{te}^\top \boldsymbol{\eta}_{te})^2 \right] \Big) \\
        &= \frac{1}{m^2} \Big( \underbrace{\mbb{E}\left[ \widetilde{\mbf{x}}_q^\top \mbf{A} \mbf{X}_{te}^\top \mbf{X}_{te} \mbf{w} \mbf{w}^\top \mbf{X}_{te}^\top \mbf{X}_{te} \mbf{A}^\top \widetilde{\mbf{x}}_q \right]}_{(d)} + \underbrace{\mbb{E} \left[ \widetilde{\mbf{x}}_q^\top \mbf{A} \mbf{X}_{te}^\top \boldsymbol{\eta}_{te} \boldsymbol{\eta}_{te}^\top \mbf{X}_{te} \mbf{A}^\top \widetilde{\mbf{x}}_q \right]}_{(e)} \Big).
    \end{align*}
    We first focus on $(d)$:
    \begin{align*}
        &\mbb{E}\left[ \widetilde{\mbf{x}}_q^\top \mbf{A} \mbf{X}_{te}^\top \mbf{X}_{te} \mbf{w} \mbf{w}^\top \mbf{X}_{te}^\top \mbf{X}_{te} \mbf{A}^\top \widetilde{\mbf{x}}_q \right] = \frac{1}{m^2} \mbb{E} \Big[ \operatorname{Tr}\left( \mbf{w} \mbf{w}^\top \mbf{X}_{te}^\top \mbf{X}_{te} \mbf{A}^\top \widetilde{\mbf{x}}_q \widetilde{\mbf{x}}_q^\top \mbf{A} \mbf{X}_{te}^\top \mbf{X}_{te} \right) \Big] \\
        &= \operatorname{Tr}\Big( \underbrace{\mbb{E}\left[ \mbf{w} \mbf{w}^\top\right] }_{\mbf{I}_d} \mbb{E}\left[ \mbf{A} \mbf{X}_{te}^\top \mbf{X}_{te} \widetilde{\mbf{x}}_q \widetilde{\mbf{x}}_q^\top \mbf{X}_{te}^\top \mbf{X}_{te} \mbf{A}^\top \right] \Big) \\
        &= \mbb{E}\Big[ \operatorname{Tr}\left( \mbf{A} \mbf{X}_{te}^\top \mbf{X}_{te} \widetilde{\mbf{x}}_q \widetilde{\mbf{x}}_q^\top \mbf{X}_{te}^\top \mbf{X}_{te} \mbf{A}^\top \right) \Big] = \mbb{E}\Big[ \operatorname{Tr}\left( \widetilde{\mbf{x}}_q\widetilde{\mbf{x}}_q^\top \mbf{X}_{te}^\top \mbf{X}_{te} \mbf{A}^\top \mbf{A} \mbf{X}_{te}^\top \mbf{X}_{te} \right) \Big] \\
        &= \operatorname{Tr}\Big( \underbrace{\mbb{E}\left[ \widetilde{\mbf{x}}_q \widetilde{\mbf{x}}_q^\top \right]}_{\mbf{\Sigma}_t} \mbb{E}\left[  \mbf{X}_{te}^\top \mbf{X}_{te} \mbf{A}^\top \mbf{A} \mbf{X}_{te}^\top \mbf{X}_{te} \right] \Big) = \mbb{E}\Big[ \operatorname{Tr}\left( \mbf{\Sigma}_t \mbf{X}_{te}^\top \mbf{X}_{te} \mbf{A}^\top \mbf{A} \mbf{X}_{te}^\top \mbf{X}_{te} \right) \Big] \\
        &= \mbb{E}\Big[ \operatorname{Tr}\left( \mbf{A} \mbf{X}_{te}^\top \mbf{X}_{te} \mbf{\Sigma}_t^{1/2} \mbf{\Sigma}_t^{1/2}  \mbf{X}_{te}^\top \mbf{X}_{te} \mbf{A}^\top \right) \Big] := \mbb{E}\Big[ \operatorname{Tr}\left( \widehat{\mbf{X}}_{te}^\top \overline{\mbf{X}}_{te} \overline{\mbf{X}}_{te}^\top \widehat{\mbf{X}}_{te} \right) \Big],
    \end{align*}
    where $\widehat{\mbf{X}}_{te}^\top := \mbf{A} \mbf{X}_{te}^\top$ and $\overline{\mbf{X}}_{te}^\top := \mbf{\Sigma}_t^{1/2} \mbf{X}_{te}^\top$. Note $\widehat{\mbf{X}}_{te}^\top = \begin{bmatrix}
        \mbf{A} \widetilde{\mbf{x}}_1 & \dots & \mbf{A} \widetilde{\mbf{x}}_m
    \end{bmatrix} := \begin{bmatrix}
        \widehat{\mbf{x}}_1 & \dots & \widehat{\mbf{x}}_m
    \end{bmatrix}$ where $\widehat{\mbf{x}}_i := \mbf{A} \widetilde{\mbf{x}}_i \overset{iid}{\sim} \mathcal{N}(\mbf{0}_d, \mbf{A} \mbf{\Sigma}_t \mbf{A}^\top)$, and $\overline{\mbf{X}}_{te}^\top = \begin{bmatrix}
        \mbf{\Sigma}_t^{1/2} \widetilde{\mbf{x}}_1 & \dots & \mbf{\Sigma}_t^{1/2} \widetilde{\mbf{x}}_m
    \end{bmatrix} := \begin{bmatrix}
        \overline{\mbf{x}}_1 & \dots & \overline{\mbf{x}}_m
    \end{bmatrix}$ where $\overline{\mbf{x}}_i := \mbf{\Sigma}_t^{1/2} \widetilde{\mbf{x}}_i \overset{iid}{\sim} \mathcal{N}(\mbf{0}_d, \mbf{\Sigma}_t^2)$. We can express $\widehat{\mbf{X}}_{te}^\top \overline{\mbf{X}}_{te}$ and $\overline{\mbf{X}}_{te}^\top \widehat{\mbf{X}}_{te}$ as such:
    \begin{equation*}
        \widehat{\mbf{X}}_{te}^\top \overline{\mbf{X}}_{te} = \sum\limits_{i=1}^m \widehat{\mbf{x}}_i \overline{\mbf{x}}_i^\top \; \; \text{and} \; \;  \overline{\mbf{X}}_{te}^\top \widehat{\mbf{X}}_{te} = \sum\limits_{j=1}^m \overline{\mbf{x}}_j \widehat{\mbf{x}}_j^\top.
    \end{equation*}
    Therefore,
    \begin{align*}
        &\mbb{E}\Big[ \operatorname{Tr}\left( \widehat{\mbf{X}}_{te}^\top \overline{\mbf{X}}_{te} \overline{\mbf{X}}_{te}^\top \widehat{\mbf{X}}_{te} \right) \Big] = \operatorname{Tr}\Big( \mbb{E}\left[ \widehat{\mbf{X}}_{te}^\top \overline{\mbf{X}}_{te} \overline{\mbf{X}}_{te}^\top \widehat{\mbf{X}}_{te} \right] \Big) \\
        &= \operatorname{Tr}\bigg( \sum\limits_{i=1}^m \sum\limits_{j=1}^m \mbb{E}\left[ \widehat{\mbf{x}}_i \overline{\mbf{x}}_i^\top \overline{\mbf{x}}_j \widehat{\mbf{x}}_j^\top \right] \bigg) \\
        &= \operatorname{Tr}\bigg( \sum\limits_{i=1}^m \sum\limits_{j\neq i} \mbb{E}\left[ \widehat{\mbf{x}}_i \overline{\mbf{x}}_i^\top \overline{\mbf{x}}_j \widehat{\mbf{x}}_j^\top \right] \bigg) + \operatorname{Tr}\bigg( \sum\limits_{i=1}^m \mbb{E}\left[ \widehat{\mbf{x}}_i \overline{\mbf{x}}_i^\top \overline{\mbf{x}}_i \widehat{\mbf{x}}_i^\top \right] \bigg)
    \end{align*}
    We first consider the case when $i \neq j$. In this setting, $\mbf{x}_i$ and $\mbf{x}_j$ are independent, so
    \begin{equation*}
        \mbb{E}\left[ \widehat{\mbf{x}}_i \overline{\mbf{x}}_i^\top \overline{\mbf{x}}_j \widehat{\mbf{x}}_j^\top \right] = \mbb{E}\left[ \widehat{\mbf{x}}_i \overline{\mbf{x}}_i^\top \right] \mbb{E}\left[ \overline{\mbf{x}}_j \widehat{\mbf{x}}_j^\top \right] = \mbf{A} \underbrace{\mbb{E}\left[\widetilde{\mbf{x}}_i \widetilde{\mbf{x}}_i^\top \right]}_{\mbf{\Sigma}_t} \mbf{\Sigma}_t \underbrace{\mbb{E} \left[\widetilde{\mbf{x}}_j \widetilde{\mbf{x}}_j^\top \right]}_{\mbf{\Sigma}_t} \mbf{A}^\top = \mbf{A} \mbf{\Sigma}_t^3 \mbf{A}^\top.
    \end{equation*}
    Therefore,
    \begin{equation*}
        \operatorname{Tr}\bigg( \sum\limits_{i=1}^m \sum\limits_{j\neq i} \mbb{E}\left[ \widehat{\mbf{x}}_i \overline{\mbf{x}}_i^\top \overline{\mbf{x}}_j \widehat{\mbf{x}}_j^\top \right] \bigg) = m \cdot (m-1) \cdot \operatorname{Tr}\left( \mbf{A} \mbf{\Sigma}_t^3 \mbf{A}^\top \right).
    \end{equation*}

    \noindent We now consider the case where $i = j$:
    \begin{align*}
        &\operatorname{Tr}\bigg( \sum\limits_{i=1}^m \mbb{E}\left[ \widehat{\mbf{x}}_i \overline{\mbf{x}}_i^\top \overline{\mbf{x}}_i \widehat{\mbf{x}}_i^\top \right] \bigg) = \sum\limits_{i=1}^m \mbb{E}\Big[  \operatorname{Tr}\left( \widehat{\mbf{x}}_i \overline{\mbf{x}}_i^\top \overline{\mbf{x}}_i \widehat{\mbf{x}}_i^\top \right) \Big] \\
        &= \sum\limits_{i=1}^m \mbb{E} \left[ \widehat{\mbf{x}}_i^\top \widehat{\mbf{x}}_i \overline{\mbf{x}}_i^\top \overline{\mbf{x}}_i \right] = \sum\limits_{i=1}^m \mbb{E} \left[ (\widetilde{\mbf{x}}_i^\top \mbf{A}^\top \mbf{A} \widetilde{\mbf{x}}_i)(\widetilde{\mbf{x}}_i^\top \mbf{\Sigma}_t \widetilde{\mbf{x}}_i) \right] \\
        &\overset{(i)}{=} m \cdot \bigg( 2\operatorname{Tr}\Big( \mbf{A} \mbf{\Sigma}_t^3 \mbf{A}^\top \Big) + \operatorname{Tr}\big(\mbf{A} \mbf{\Sigma}_t \mbf{A}^\top \big) \operatorname{Tr}\big( \mbf{\Sigma}_t^2 \big) \bigg),
    \end{align*}
    where $(i)$ is because for 
    $\mbf{a} \sim \mathcal{N}(\mbf{0}_d, \mbf{\Sigma}_t)$ and fixed $\mbf{Q}, \mbf{R} \in \mbb{R}^{d \times d}$, $\mbb{E}\left[ (\mbf{a}^\top \mbf{Q} \mbf{a})(\mbf{a}^\top \mbf{R} \mbf{a}) \right] = \operatorname{Tr}\left( \mbf{Q} \mbf{\Sigma}_t (\mbf{R} + \mbf{R}^\top) \mbf{\Sigma}_t \right) + \operatorname{Tr}\left(\mbf{Q}\mbf{\Sigma}_t\right) \operatorname{Tr}\left(\mbf{R}\mbf{\Sigma}_t\right)$ (see Section 8.2.4 in \cite{petersen2008matrix}).
    
    \medskip
    
    \noindent We now focus on $(e)$:
    \begin{align*}
        &\mbb{E} \left[ \widetilde{\mbf{x}}_q^\top \mbf{A} \mbf{X}_{te}^\top \boldsymbol{\eta}_{te} \boldsymbol{\eta}_{te}^\top \mbf{X}_{te} \mbf{A}^\top \widetilde{\mbf{x}}_q \right] = \mbb{E}\Big[ \operatorname{Tr}\left( \widetilde{\mbf{x}}_q \widetilde{\mbf{x}}_q^\top \mbf{A} \mbf{X}_{te}^\top \boldsymbol{\eta}_{te} \boldsymbol{\eta}_{te}^\top \mbf{X}_{te} \mbf{A}^\top  \right) \Big] \\
        &= \operatorname{Tr}\Big( \underbrace{\mbb{E}\left[ \widetilde{\mbf{x}}_q \widetilde{\mbf{x}}_q^\top \right]}_{\mbf{\Sigma}_t} \mbb{E}\left[\mbf{A}  \mbf{X}_{te}^\top \boldsymbol{\eta}_{te} \boldsymbol{\eta}_{te}^\top \mbf{X}_{te} \mbf{A}^\top\right] \Big) = \operatorname{Tr} \Big( \mbb{E}\left[ \mbf{\Sigma}_t^{1/2} \mbf{A} \mbf{X}_{te}^\top \boldsymbol{\eta}_{te} \boldsymbol{\eta}_{te}^\top \mbf{X}_{te} \mbf{A}^\top \mbf{\Sigma}_t^{1/2}  \right] \Big) \\
        & := \operatorname{Tr}\Big( \mbb{E}\left[  \widetilde{\boldsymbol{\eta}}_{te} \widetilde{\boldsymbol{\eta}}_{te}^\top \right] \Big),
    \end{align*}
    where $\widetilde{\boldsymbol{\eta}}_{te} := \mbf{\Sigma}_t^{1/2} \mbf{A} \mbf{X}_{te}^\top \boldsymbol{\eta}_{te} := \widetilde{\mbf{X}}_{te}^\top \boldsymbol{\eta}_{te}$.  Note the columns of $\widetilde{\mbf{X}}_{te}^\top$ are iid Gaussian with covariance $\mbf{A} \mbf{\Sigma}_t^2 \mbf{A}^\top$. By Corollary 1 in \cite{mattei2017multiplying}, $\widetilde{\boldsymbol{\eta}}_{te} \sim \text{GAL}_d(2\sigma^2 \mbf{A} \mbf{\Sigma}_t \mbf{A}^\top, \mbf{0}_d, m/2)$, where $\text{GAL}_p(\mbf{\Sigma}, \boldsymbol{\mu}, s)$ denotes a $p$-dimensional \emph{multivariate generalized asymmetric Laplace distribution} with mean $s \boldsymbol{\mu}$ and covariance $s(\mbf{\Sigma} + \boldsymbol{\mu} \boldsymbol{\mu}^\top)$ (Definition 1 and Proposition 2 in \cite{mattei2017multiplying}). Therefore,
    \begin{align*}
        \operatorname{Tr}\Big( \mbb{E}\left[  \widetilde{\boldsymbol{\eta}}_{te} \widetilde{\boldsymbol{\eta}}_{te}^\top \right] \Big) = m\sigma^2 \operatorname{Tr}\Big( \mbf{A} \mbf{\Sigma}_t^2 \mbf{A}^\top \Big).
    \end{align*}

    \paragraph{Adding $(a)$, $(b)$, and $(c)$.} Adding the expressions for $(a)$, $(b)$, and $(c)$, where $(c) = (d) + (e)$, yields and combining like terms yields the following expression:
    \begin{align*}
        \mbb{E} &\left[ \left(\mbf{w}^\top \widetilde{\mbf{x}}_q + \eta_q - \widehat{\mbf{w}}^\top \widetilde{\mbf{x}}_q \right)^2 \right] = \underbrace{\operatorname{Tr}\left(\mbf{\Sigma}_t\right) + \sigma^2}_{= (a)} - 2 \underbrace{\operatorname{Tr}\left( \mbf{\Sigma}_t^2 \mbf{A} \right)}_{= (b)} \\
        &+ \underbrace{\underbrace{\frac{m(m-1)}{m^2} \operatorname{Tr}\left(\mbf{A} \mbf{\Sigma}_t^3 \mbf{A}^\top\right) + \frac{2}{m} \operatorname{Tr}\left(\mbf{A} \mbf{\Sigma}_t^3 \mbf{A}^\top  \right) + \frac{1}{m} \operatorname{Tr}(\mbf{\Sigma}_t^2) \operatorname{Tr}\left( \mbf{A} \mbf{\Sigma}_t \mbf{A}^\top \right)}_{=(d)} + \underbrace{\frac{\sigma^2}{m}\operatorname{Tr}\left( \mbf{A} \mbf{\Sigma}_t^2 \mbf{A}^\top \right)}_{=(e)}}_{=(c)}.
    \end{align*}
    Combining like terms and using the fact $M_t = \operatorname{Tr}\left(\mbf{\Sigma}_t \right) + \sigma^2$ yields
    \begin{align*}
         \mbb{E} \left[ \left(\mbf{w}^\top \widetilde{\mbf{x}}_q + \eta_q - \widehat{\mbf{w}}^\top \widetilde{\mbf{x}}_q \right)^2 \right] &= M_t - 2 \operatorname{Tr}\left( \mbf{\Sigma}_t^2 \mbf{A} \right) + \frac{m+1}{m} \operatorname{Tr}\left(\mbf{A} \mbf{\Sigma}_t^3 \mbf{A} \right) \\
         &+ \frac{1}{m} \operatorname{Tr}\left( \mbf{\Sigma}_t^2 \right) \operatorname{Tr}\left( \mbf{A} \mbf{\Sigma}_t \mbf{A} \right) +  \frac{\sigma^2}{m} \operatorname{Tr}\left(\mbf{A} \mbf{\Sigma}_t^2 \mbf{A}\right),
    \end{align*}
    which is exactly \Cref{eq:feature_shift_risk}. This completes the proof.
\end{proof}

\subsubsection{Proof of \Cref{prop:feature_neg_result}}
\label{sec:feature_neg_result_proof}

\begin{proof}
    From \Cref{lem:feature_test_risk}, recall that the test risk is given by
    \begin{align*}
        \mbb{E} \left[ \left(\mbf{w}^\top \widetilde{\mbf{x}}_q + \eta_q - \widehat{\mbf{w}}^\top \widetilde{\mbf{x}}_q \right)^2 \right] &= \operatorname{Tr}\left(\mbf{\Sigma}_t\right) + \sigma^2 - 2\operatorname{Tr}\left( \mbf{\Sigma}_t^2 \mbf{A} \right) + \frac{m + 1}{m} \operatorname{Tr}\left(\mbf{A} \mbf{\Sigma}_t^3 \mbf{A}\right) \\
        &+ \frac{1}{m} \operatorname{Tr}(\mbf{\Sigma}_t^2) \operatorname{Tr}\left( \mbf{A} \mbf{\Sigma}_t \mbf{A} \right) + \frac{\sigma^2}{m}\operatorname{Tr}\left( \mbf{A} \mbf{\Sigma}_t^2 \mbf{A} \right)
    \end{align*}
    where $M_t = \operatorname{Tr}\left(\mbf{\Sigma}_t\right) + \sigma^2$, $\mbf{\bar{A}} = \left( \frac{n+1}{n} \mbf{I}_d + \frac{M_s}{n} \mbf{\Sigma}_s^{-1} \right)^{-1}$, and $\mbf{A} = \mbf{\Sigma}_s^{-1/2} \bar{\mbf{A}} \mbf{\Sigma}_s^{-1/2}$. Since $\mbf{\Sigma}_s = \mbf{U}_s \mbf{U}_s^\top + \epsilon \mbf{I}_d$, we can write its eigendecomposition as such:
    \begin{equation*}
        \mbf{\Sigma}_s = \mbf{U}_s\mbf{U}_s^\top + \epsilon \mbf{I}_d = \mbf{U} \begin{bmatrix}
            \mbf{I}_r & \ \\
            \ & \mbf{0}_{(d - r) \times (d - r)}
        \end{bmatrix} \mbf{U}^\top + \epsilon \mbf{I} = \mbf{U} \begin{bmatrix}
            (1 + \epsilon) \mbf{I}_r & \ \\
            \ & \epsilon \mbf{I}_{d - r}
        \end{bmatrix} \mbf{U}^\top,
    \end{equation*}
    where $\mbf{U} = \begin{bmatrix}
        \mbf{U}_s & \mbf{U}_{s, \perp} & \mbf{U}_{2r, \perp}
    \end{bmatrix}$.
    Then, we can write $\mbf{A}$ as follows:
    \begin{equation} \label{eq:A_eigen_prop2}
         \mbf{A} = \mbf{\Sigma}_s^{-1/2} \left( \frac{n+1}{n} \mbf{I}_d + \frac{M_s}{n} \mbf{\Sigma}_s^{-1} \right)^{-1} \mbf{\Sigma}_s^{-1/2}  = \left( \frac{n+1}{n} \mbf{\Sigma}_s + \frac{M_s}{n} \mbf{I}_d \right)^{-1} = \mbf{U\Lambda U}^\top 
    \end{equation}
    where
    \begin{align*}
        \mbf{\Lambda} = \begin{bmatrix}
             \frac{n}{(n+1)(1+\epsilon) + M_s} \cdot \mbf{I}_r & \\
         & \frac{n}{(n+1)\epsilon + M_s} \cdot \mbf{I}_{d-r}
         \end{bmatrix} := \begin{bmatrix}
             \nu_1 \mbf{I}_r & \ \\
             \ & \nu_2 \mbf{I}_{d - r}
         \end{bmatrix}.
    \end{align*}
    We simplify the $\operatorname{Tr}(\cdot)$ terms using \Cref{eq:A_eigen_prop2}. Before proceeding, note
    \begin{equation} \label{eq:nu_n_limits}
        \lim\limits_{n \to \infty} \nu_1 = \frac{1}{1 + \epsilon} \quad \text{and} \quad \lim\limits_{n \to \infty} \nu_2 = \frac{1}{\epsilon}, 
    \end{equation}
    as well as
    \begin{equation} \label{eq:test_risk_m_limit}
        \lim\limits_{m \to \infty} \mbb{E} \left[ \left(\mbf{w}^\top \widetilde{\mbf{x}}_q + \eta_q - \widehat{\mbf{w}}^\top \widetilde{\mbf{x}}_q \right)^2 \right] = \operatorname{Tr}\left(\mbf{\Sigma}_t\right) + \sigma^2 - 2\operatorname{Tr}\left( \mbf{\Sigma}_t^2 \mbf{A} \right) + \operatorname{Tr}\left(\mbf{A} \mbf{\Sigma}_t^3 \mbf{A}\right),
    \end{equation}
    so we only focus on the $\operatorname{Tr}(\mbf{\Sigma}_t), \operatorname{Tr}(\mbf{\Sigma}_t^2 \mbf{A})$, and $\operatorname{Tr}(\mbf{A} \mbf{\Sigma}_t^3 \mbf{A})$ terms. Also, recall $\mbf{\Sigma}_t = \mbf{U}_t\mbf{U}_t^\top  + \epsilon \mbf{I}_d$, so
    \begin{align}
        &\mbf{\Sigma}_t^2 = \left( \mbf{U}_t \mbf{U}_t^\top + \epsilon \mbf{I}_d \right) \left( \mbf{U}_t \mbf{U}_t^\top + \epsilon \mbf{I}_d \right) = (1 + 2\epsilon) \mbf{U}_t \mbf{U}_t^\top + \epsilon^2 \mbf{I}_d, \quad \text{and} \label{eq:sigma_t_squared} \\
        &\mbf{\Sigma}_t^3 = \left( \mbf{U}_t \mbf{U}_t^\top + \epsilon \mbf{I}_d \right) \left( (1 + 2\epsilon) \mbf{U}_t \mbf{U}_t^\top + \epsilon^2 \mbf{I}_d \right) = \left( (1 + 2\epsilon)(1 + \epsilon) + \epsilon^2 \right) \mbf{U}_t \mbf{U}_t^\top + \epsilon^3 \mbf{I}_d. \label{eq:sigma_t_cubed}
    \end{align}

    \paragraph{Simplifying $\operatorname{Tr}(\mbf{\Sigma}_t^2 \mbf{A})$.} Using \Cref{eq:A_eigen_prop2} and \Cref{eq:sigma_t_squared} yields
    \begin{align*}
        \operatorname{Tr} \left( \mbf{\Sigma}_t^2 \mbf{A} \right) &= (1 + 2\epsilon) \operatorname{Tr}\left( \mbf{U}_t^\top \mbf{U \Lambda U}^\top \mbf{U}_t \right) + \epsilon^2 \operatorname{Tr}(\mbf{A}) \\
        &= (1 + 2\epsilon)\left( r \nu_1 \cos^2(\theta) + r \nu_2 \sin^2(\theta) \right) + \epsilon^2 \left( r \nu_1 + (d - r) \nu_2 \right).
    \end{align*}

    \paragraph{Simplifying $\operatorname{Tr}\left( \mbf{A} \mbf{\Sigma}_t^3 \mbf{A} \right)$.} Using \Cref{eq:A_eigen_prop2} and \Cref{eq:sigma_t_cubed} yields
    \begin{align*}
        \operatorname{Tr}\left( \mbf{A} \mbf{\Sigma}_t^3 \mbf{A} \right) = \operatorname{Tr}\left( \mbf{\Sigma}_t^3 \mbf{A}^2 \right) &= \left( (1 + 2\epsilon)(1 + \epsilon) + \epsilon^2 \right) \operatorname{Tr}\left(\mbf{U}_t \mbf{U} \mbf{\Lambda}^2 \mbf{U}^\top \mbf{U}_t\right) + \epsilon^3 \operatorname{Tr}\left(\mbf{A}^2\right) \\
        &= \left( (1 + 2\epsilon)(1 + \epsilon) + \epsilon^2 \right) \left(r \nu_1^2 \cos^2(\theta) + r \nu_2^2 \sin^2(\theta) \right) + \epsilon^2 \left(r \nu_1^2 + (d - r) \nu_2^2 \right).
    \end{align*}

    \paragraph{Simplifying the Test Risk.} Combining \Cref{eq:nu_n_limits} and \Cref{eq:test_risk_m_limit}, and then substituting the expressions for the $\operatorname{Tr}(\cdot)$ terms yields
    \begin{align*}
        \lim\limits_{m \to \infty} \lim\limits_{n \to \infty} \mbb{E} \left[ \left(\mbf{w}^\top \widetilde{\mbf{x}}_q + \eta_q - \widehat{\mbf{w}}^\top \widetilde{\mbf{x}}_q \right)^2 \right] &= r + \sigma^2 - 2(1 + 2\epsilon) \left(\frac{r \cos^2(\theta)}{1 + \epsilon} + \frac{r \sin^2(\theta)}{\epsilon} \right) \\
        &+ \left((1 + 2\epsilon)(1 + \epsilon) + \epsilon^2 \right) \left(\frac{r \cos^2(\theta)}{(1 + \epsilon)^2} + \frac{r \sin^2(\theta)}{\epsilon^2} \right) + \mathcal{O}(\epsilon).
    \end{align*}
    Letting $c_1 := \left((1 + 2\epsilon)(1 + \epsilon) + \epsilon^2 \right)$ and $c_2 := 2(1 + \epsilon)$ yields
    \begin{equation*}
         \lim\limits_{m \to \infty} \lim\limits_{n \to \infty} \mbb{E} \left[ \left(\mbf{w}^\top \widetilde{\mbf{x}}_q + \eta_q - \widehat{\mbf{w}}^\top \widetilde{\mbf{x}}_q \right)^2 \right] = r + \sigma^2 + \left( \frac{c_1 - (1 + \epsilon)c_2}{(1 + \epsilon)^2} \right) r \cos^2(\theta) + \left( \frac{c_1 - \epsilon c_2}{\epsilon^2} \right) r \sin^2(\theta) + \mathcal{O}(\epsilon),
    \end{equation*}
    which completes the proof.
    
    % \begin{align*}
    %     &\mbf{U}_t = \mbf{U}_s\cos(\mbf{\Theta})   + \mbf{U}_{s,\perp}\sin(\mbf{\Theta}) = \begin{bmatrix}
    %         \mbf{U}_s &
    %         \mbf{U}_{s,\perp}
    %     \end{bmatrix}
    %     \begin{bmatrix}
    %         \cos(\mbf{\Theta}) \\ \sin(\mbf{\Theta})
    %     \end{bmatrix} \\
    %     &\implies \mbf{\Sigma}_t = \mbf{U} \begin{bmatrix}
    %         \cos^2(\mbf{\Theta}) + \epsilon \mbf{I}_r& \cos(\mbf{\Theta})\sin(\mbf{\Theta})  & \\
    %         \cos(\mbf{\Theta})\sin(\mbf{\Theta})& \sin^2(\mbf{\Theta}) + \epsilon \mbf{I}_r  & \\
    %         & & \epsilon \mbf{I}_{d-2r}
    %     \end{bmatrix}
    %     \mbf{U}^\top.
    % \end{align*}
\end{proof}

\subsubsection{Proof of Corollary~\ref{coro:lora}}
\label{sec:proof_of_coro_lora}

\begin{proof}

By Lemma~\ref{lem:vanilla_opt_weights}, we know that optimal solution to the linear attention model corresponding to the independent data setting in Equation~(\ref{eqn:single_subspace_setup}) (i.e., the ``pre-trained model'') satisfies
\begin{equation}
        \mbf{W}_K^\star = \mbf{W}_V^\star = \mbf{I}_{d+1}, \; \; \mbf{W}_Q^\star = \begin{bmatrix}
        \mbf{A} & \mbf{0}_d \\
        \mbf{0}_d^\top & 0
    \end{bmatrix}, \; \; \text{and} \; \; \mbf{p}^\star = \begin{bmatrix}
        \mbf{0}_d \\
        1
    \end{bmatrix},
    \end{equation}
    where $\mbf{A} = \left( \frac{n+1}{n} \mbf{I}_d + \frac{M_s}{n} \mbf{\Sigma}_s^{-1} \right)^{-1}$ and $M_s = \operatorname{Tr}(\mbf{\Sigma}_s) + \sigma^2 = (1+\epsilon)r + (d-r)\epsilon + \sigma^2$. We prove that
    \begin{align*}
        \mbf{B}_2 = \mbf{B}_1 = \begin{bmatrix}
            \mbf{U}_{s, \perp} \mbf{\Lambda}_r^{1/2} \\
            \mbf{0}_{r}^\top
        \end{bmatrix} \in \mbb{R}
^{(d+1) \times r}    \end{align*}
    suffices as our choice of adapters, where $\mbf{0}_{r} \in \mbb{R}^r$ (i.e., a vector of zeros) and
    $\mbf{\Lambda}_r = \left(\frac{n(1+\epsilon)}{(n+1)\epsilon + M_s} \right) \mbf{I}_r$. Note that
    \begin{align*}
        \mbf{W}_Q^\star \mbf{W}_K^{\star\top} + \mbf{B}_2\mbf{B}_1^\top &= \mbf{W}_Q^\star + \mbf{B}_2\mbf{B}_1^\top \\
        &= \begin{bmatrix}
        \mbf{A} & \mbf{0}_d \\
        \mbf{0}_d^\top & 0
    \end{bmatrix} + \begin{bmatrix}
        \mbf{U}_{s, \perp} \mbf{\Lambda}_r \mbf{U}_{s, \perp}^\top  & \mbf{0}_d \\
        \mbf{0}_d^\top & 0
    \end{bmatrix} \\
    &= \begin{bmatrix}
       \widehat{\mbf{A}}  & \mbf{0}_d \\
        \mbf{0}_d^\top & 0
    \end{bmatrix}, 
    \end{align*}
    where $\widehat{\mbf{A}} = \mbf{A} + \mbf{U}_{s, \perp} \mbf{\Lambda}_r \mbf{U}_{s, \perp}^\top$. Using Lemma~\ref{lem:I-psd-inv}, we can simplify $\widehat{\mbf{A}}$ as such:
    \begin{align*}
        \widehat{\mbf{A}} &= \mbf{U} \left(\begin{bmatrix}
            \left(\frac{n(1+\epsilon)}{(n+1)(1 + \epsilon)  + M_s} \right) \mbf{I}_r & \\
            & \left(\frac{n\epsilon}{(n+1)(1 + \epsilon) + M_s} \right) \mbf{I}_{d-r}
        \end{bmatrix} + \begin{bmatrix}
            \mbf{0}_{r\times r}  & & \\
            & \left(\frac{n(1+\epsilon)}{(n+1)(1 + \epsilon)  + M_s} \right) \mbf{I}_r & \\
            & & \mbf{0}_{(d-2r) \times (d-2r)}
        \end{bmatrix} \right) \mbf{U}^\top \\
        &= \mbf{U} \begin{bmatrix}
            \left(\frac{n(1+\epsilon)}{(n+1)(1 + \epsilon)  + M_s} \right) \mbf{I}_r & & \\
            & \left(\frac{n(1+\epsilon) + n\epsilon}{(n+1)(1 + \epsilon) + M_s} \right) \mbf{I}_r &\\
        & & \left(\frac{n\epsilon}{(n+1)(1 + \epsilon) + M_s} \right) \mbf{I}_{d-2r}\end{bmatrix}\mbf{U}^\top
    \end{align*}
    Then, by invoking Lemma~\ref{lem:task_test_risk} with $\widehat{\mbf{A}}$, following the proof strategies of \Cref{thm:mix_two_subspaces,thm:mixture_k_subspaces}, and assuming $n \leq m$, we have
    \begin{align*}
        \lim\limits_{\epsilon \to 0} \mbb{E}\left[ \left(\tilde{y}_{m+1} - h^\star_{\texttt{ATT}}( \widetilde{\mbf{Z}}) \right)^2 \right] &= r + \sigma^2 + \frac{m + 1 + 2(r + \sigma^2)}{m} \cdot \frac{rn^2}{(n + 1 + r + \sigma^2)^2} - \frac{2rn}{n + 1 + r + \sigma^2} \\
        &\leq r + \sigma^2 + \frac{n + 1 + 2(r + \sigma^2)}{n} \cdot \frac{rn^2}{(n + 1 + r + \sigma^2)^2} - \frac{2rn}{n + 1 + r + \sigma^2} \\
        &= r + \sigma^2 - \frac{rn(n+1)}{(n + 1 + r + \sigma^2)^2}.
    \end{align*}
    Upper bounding this by $\sigma^2 + \delta$ for some $\delta \in (0, r)$, then solving for $n$, leads to the following result. For any $\delta \in (0, r)$, if 
    \begin{equation*}
        m \geq n > \frac{2(r + \sigma^2 + 1)(r - \delta) - r}{\delta} + (r + \sigma^2 + 1) \sqrt{\frac{r - \delta}{\delta}},
    \end{equation*}
    then $\lim\limits_{\epsilon \to 0} \mbb{E}\left[ \left(\tilde{y}_{m+1} - h^\star_{\texttt{ATT}}( \widetilde{\mbf{Z}}) \right)^2 \right] < \sigma^2 + \delta$.
    % it is easy to show that 
    % \begin{align*}
    %     \lim\limits_{m \rightarrow \infty} \lim\limits_{n \rightarrow \infty} \lim\limits_{\epsilon \rightarrow 0}  \mbb{E}\left[ \left(\tilde{y}_{m+1} - h^\star_{\texttt{ATT}}(\tilde{\mbf{z}}_q, \widetilde{\mbf{Z}}_{\mc{M}}) \right)^2 \right]   = \sigma^2.
    % \end{align*}
    % This completes the proof.
\end{proof}

\subsection{Auxiliary Results}
\label{sec:aux_proofs}
Here, we provide auxiliary results to support the proofs in \Cref{sec:task_proofs,sec:feature_shifts_proofs}.

\subsubsection{Optimal Linear Attention Weights}
We first provide results on the form of the weights matrices after training a single-layer linear attention model on the objective \Cref{eqn:expected_lin_att_objective}. The following results are largely inspired by Theorem 1 in \cite{li2024fine}, but are slightly different  since we consider a normalization factor of $1/n$ in our linear attention model.

\begin{lemma}[Optimal Attention Weights~\cite{li2024fine}]
\label{lem:vanilla_opt_weights}
Consider the independent data model in Equation~(\ref{eqn:vanilla_setup}) with $\mbf{w} \sim \mc{N}(\mbf{0}, \mbf{\Sigma}_s)$, and let $n \in \mbb{N}$ denote the in-context prompt length used at training.
Then, the optimal linear attention weights obtained by minimizing the loss in Equation~(\ref{eqn:expected_lin_att_objective}) are given by 
\begin{equation}
        \mbf{W}_K^\star = \mbf{W}_V^\star = \mbf{I}_{d+1}, \; \; \text{and} \; \;\mbf{W}_Q^\star = \begin{bmatrix}
        \mbf{A} & \mbf{0}_d \\
        \mbf{0}_d^\top & 0
    \end{bmatrix}
    \end{equation}
    where $\mbf{A} = \left( \frac{n+1}{n} \mbf{I}_d + \frac{M_s}{n} \mbf{\Sigma}_s^{-1} \right)^{-1}$ and $M_s = \operatorname{Tr}(\mbf{\Sigma}_s) + \sigma^2$, with empirical risk $\mathcal{L}_s^\star = M_s  -\operatorname{Tr}\left(\mbf{\Sigma}_s \mbf{A}\right)$.
\end{lemma}
\begin{proof}
    The proof is the same as that of Theorem 1 in \cite{li2024fine} by absorbing the $1/n$ factor into $\mbf{W}_Q$.
\end{proof}

\begin{lemma}[Optimal Attention Weights for Mixture of $K$ Gaussians]
\label{lem:opt_mix_K_weights}
Consider the independent data model in \Cref{eqn:vanilla_setup} with $\mbf{w} \sim \sum\limits_{k=1}^K \gamma_k \cdot \mathcal{N}\left(\mbf{0}, \mbf{\Sigma}_{s, k}\right)$ for $\gamma_k \in (0, 1)$ for all $k \in [K]$ and $\sum\limits_{k=1}^K \gamma_k = 1$. Let $n \in \mbb{N}$ denote the in-context prompt length used at training.
Define $\mbf{\Sigma} = \sum\limits_{k=1}^K \gamma_k \cdot \mbf{\Sigma}_{s,k}$.
Then, the optimal linear attention weights obtained by minimizing the loss in Equation~(\ref{eqn:expected_lin_att_objective}) are given by 
\begin{equation}
        \mbf{W}_K^\star = \mbf{W}_V^\star = \mbf{I}_{d+1}, \; \; \text{and} \; \; \mbf{W}_Q^\star = \begin{bmatrix}
        \mbf{A} & \mbf{0}_d \\
        \mbf{0}_d^\top & 0
    \end{bmatrix},
    \end{equation}
    where $\mbf{A} = \left( \frac{n+1}{n} \mbf{I}_d + \frac{M_s}{n} \mbf{\Sigma}^{-1} \right)^{-1}$ and $M_s = \operatorname{Tr}(\mbf{\Sigma}) + \sigma^2$, with empirical risk $\mathcal{L}_s^\star = M_s  -\operatorname{Tr}\left(\mbf{\Sigma} \mbf{A}\right)$.
\end{lemma}
\begin{proof}
     It is straightforward to see that if $\mbf{w} \sim \sum\limits_{k=1}^K \gamma_k \gamma\cdot\mc{N}(\mbf{0}, \mbf{\Sigma}_{s, k})$, 
    then
    \begin{align*}
        \mbf{\Sigma}_s \coloneqq \mathrm{Cov}(\mbf{w}) = \sum\limits_{k=1}^K \gamma_k \cdot \mbf{\Sigma}_{s, k}.
    \end{align*}
    Then, the proof is equivalent to that of \Cref{lem:vanilla_opt_weights} under the new form of $\mbf{\Sigma}_s$. 
\end{proof}

\begin{lemma}[Optimal Attention Weights for Non-Isotropic Features \cite{li2024fine}] \label{lem:opt_weights_noniso_features}
    Consider the independent data model in \Cref{eqn:feature_shift} where the features are drawn from $\mbf{x}_i \sim \mathcal{N}\left(\mbf{0}, \mbf{\Sigma}_s\right)$ for all $i \in [n]$, and $n \in \mbb{N}$ denotes the in-context prompt length used at training. The optimal linear attention weights obtained by minimizing the loss in \Cref{eqn:expected_lin_att_objective} are given by
    \begin{equation}
        \mbf{W}_K^\star = \mbf{W}_V^\star = \mbf{I}_{d+1}, \; \; \text{and} \; \; \mbf{W}_Q^\star = \begin{bmatrix}
        \mbf{A} & \mbf{0}_d \\
        \mbf{0}_d^\top & 0
    \end{bmatrix},
    \end{equation}
    where $\mbf{A} = \mbf{\Sigma}_s^{-1/2} \mbf{\bar{A}} \mbf{\Sigma}_s^{-1/2}$, $\mbf{\bar{A}} = \left( \frac{n+1}{n} \mbf{I}_d + \frac{M_s}{n} \mbf{\Sigma}_s^{-1} \right)^{-1}  $ and $M_s = \operatorname{Tr}(\mbf{\Sigma}_s) + \sigma^2$, with empirical risk $\mathcal{L}_s^\star = M_s -  \tr(\mbf{\Sigma}_s^\top \mbf{\bar{A}})$.
\end{lemma}
\begin{proof}
    The proof is the same as that of \Cref{lem:vanilla_opt_weights} and Theorem 1 in \cite{li2024fine}.
\end{proof}

\subsubsection{Miscellaneous Results}

\begin{lemma} \label{lem:I-psd-inv}
    Let $0 \prec \mbf{\Sigma} \in \mbb{R}^{d \times d}$ and $c, k > 0$ be constants. Then, 
    \begin{equation}
        \left( c \cdot \mbf{I}_d + k \cdot \mbf{\Sigma}^{-1} \right)^{-1} = \mbf{V} \begin{bmatrix}
                \frac{\lambda_1}{c \cdot \lambda_1 + k} & 0 & \dots & 0  \\
                0 & \frac{\lambda_2}{c \cdot \lambda_2 + k} & \dots & 0 \\
                \vdots & \vdots & \ddots & \vdots \\
                0 & 0 & \dots & \frac{\lambda_d}{c \cdot \lambda_d + k}
            \end{bmatrix} \mbf{V}^\top,
    \end{equation}
    where $\mbf{V} \in \mbb{R}^{d \times d}$ is an orthonormal matrix whose columns are eigenvectors of $\mbf{\Sigma}$, and $\lambda_i$ is the $i^{th}$ largest eigenvalue of $\mbf{\Sigma}$.

    \begin{proof}
        Since $\mbf{\Sigma} \succ 0$, there exists an eigendecomposition $\mbf{\Sigma} = \mbf{V} \mbf{\Lambda} \mbf{V}^\top$ such that $\mbf{V}$ is an orthonormal matrix and $\mbf{\Lambda}$ is a diagonal matrix consisting of the real, positive eigenvalues of $\mbf{\Sigma}$, denoted as $\lambda_1, \lambda_2, \dots, \lambda_d$. Thus,
        \begin{align*}
            &\mbf{\Sigma}^{-1} = \mbf{V} \mbf{\Lambda}^{-1} \mbf{V}^\top \implies c \cdot \mbf{I}_d + k \cdot \mbf{\Sigma}^{-1} = \mbf{V}\underbrace{\begin{bmatrix}
                c + \frac{k}{\lambda_1} & 0 & \dots & 0  \\
                0 & c + \frac{k}{\lambda_2} & \dots & 0 \\
                \vdots & \vdots & \ddots & \vdots \\
                0 & 0 & \dots & c + \frac{k}{\lambda_d}
            \end{bmatrix}}_{\widetilde{\mbf{\Lambda}}} \mbf{V}^\top \\
            &\implies \left( c \cdot \mbf{I}_d + k \cdot \mbf{\Sigma}^{-1} \right)^{-1} = \mbf{V} \widetilde{\mbf{\Lambda}}^{-1} \mbf{V}^\top,
        \end{align*}
        which completes the proof.
    \end{proof}
\end{lemma}

\end{document}